\documentclass[twoside,10pt]{article}

\usepackage{thmtools} 
\usepackage{thm-restate}
\usepackage{jmlr2e}
\usepackage{mathtools}
\usepackage{enumitem}
\usepackage{tabularx}
\newtheorem{myexample}[theorem]{Example}

\def\alphab{\boldsymbol\alpha}
\def\betab{\boldsymbol\beta}

\def\thetab{\boldsymbol\theta}

\def\cte{M}
\def\moment{\operatorname{M}}

\def\a{{\bf a}}
\def\b{{\bf b}}

\def\h{{\bf h}}

\def\Mcal{{\mathcal{M}}}

\def\w{{\boldsymbol\omega}}

\def\Mbf{{\bf M}}

\def\ubf{{\bf u}}

\def\X{{\bf X}}

\def\R{{\mathbb{R}}}

\newcommand{\xbf}{\mathbf{x}}
\newcommand{\ebf}{\mathbf{e}}
\newcommand{\ybf}{\mathbf{y}}
\newcommand{\zbf}{\mathbf{z}}
\newcommand{\sbf}{\mathbf{s}}

\newcommand{\expo}{r}

\newcommand{\ie}{\textit{i.e.}}

\newcommand{\E}{\mathbb{E}}
\newcommand{\C}{\mathbb{C}}
\newcommand{\Acal}{\mathcal{A}}

\newcommand{\D}{\Delta}

\newcommand{\supp}{\operatorname{supp}}

\newcommand{\Lip}{\operatorname{Lip}}
\newcommand{\KL}{\operatorname{KL}}
\newcommand{\MLP}{\operatorname{MLP}}
\newcommand{\W}{\operatorname{W}}
\newcommand{\MMD}{\operatorname{MMD}}

\newcommand{\integ}[1]{{[\![#1]\!]}}
\def\P{{\mathcal{P}}}
\def\Xcal{{\mathcal{X}}}
\def\Ycal{{\mathcal{Y}}}

\def\Hcal{{\mathcal{H}}}
\def\Acal{{\mathcal{A}}}
\def\Rcal{{\mathcal{R}}}
\def\bbf{{\mathbf{b}}}
\def\wbf{{\mathbf{w}}}

\newcommand{\dr}{\mathrm{d}}

\newcommand{\Sfrak}{\mathfrak{S}}
\newcommand{\Gcal}{\mathcal{G}}

\usepackage{lastpage}

\firstpageno{1}
\jmlrheading{24}{2023}{1-\pageref{LastPage}}{12/21; Revised
3/23}{4/23}{21-1516}{Titouan Vayer and Rémi Gribonval}
\ShortHeadings{Controlling $\W_p$ by $\MMD$ with application to Compressive Statistical Learning}{Vayer and Gribonval}

\begin{document}

\title{Controlling Wasserstein Distances by Kernel Norms with Application to Compressive Statistical Learning}

\author{\name Titouan Vayer \email titouan.vayer@inria.fr \\
       \addr Univ Lyon, Inria, CNRS, ENS de Lyon, UCB Lyon 1,\\
       LIP UMR 5668, F-69342, Lyon, France
       \AND
       \name Rémi Gribonval \email remi.gribonval@inria.fr \\
       \addr Univ Lyon, Inria, CNRS, ENS de Lyon, UCB Lyon 1,\\
       LIP UMR 5668, F-69342, Lyon, France}

\editor{Marco Cuturi}

\maketitle

\begin{abstract}Comparing probability distributions is at the crux of many machine learning algorithms. Maximum Mean Discrepancies (MMD) and Wasserstein distances are two classes of distances between probability distributions that have attracted abundant attention in past years. This paper establishes some conditions under which the Wasserstein distance can be controlled by MMD norms. Our work is motivated by the \emph{compressive statistical learning} (CSL) theory, a general framework for resource-efficient large scale learning in which the training data is summarized in a single vector (called \emph{sketch}) that captures the information relevant to the considered learning task. Inspired by existing results in CSL, we introduce the \emph{Hölder Lower Restricted Isometric Property} and show that this property comes with interesting guarantees for compressive statistical learning. Based on the relations between the MMD and the Wasserstein distances, we provide guarantees for compressive statistical learning by introducing and studying the concept of \emph{Wasserstein regularity} of the learning task, that is when some task-specific metric between probability distributions can be bounded by a Wasserstein distance.
\end{abstract}


\section{Introduction}
\label{sec:intro}

Countless methods in machine learning (ML) and data science rely on comparing probability distributions. Whether it is to measure errors between parametric models and empirical datasets or to produce statistical tests, a recurring problem is to define loss functions that could faithfully quantify the
discrepancy between two probability distributions $\pi$ and $\pi'$. Divergences and metrics are frequently used to address this problem and are at the core of numerous works, ranging from signal processing \citep{kolouri_2017}, generative modeling \citep{arjovsky17a,pmlr-v84-genevay18a}, supervised and semi-supervised learning \citep{Frogner_2015,solomon14}, fairness \citep{pmlr-v97-gordaliza19a}, two-sample testing \citep{JMLR:v13:gretton12a} or in information theory \citep{Liese}. The choice of such a metric is an important issue, as finding a suitable one is delicate and often depends on many criteria such as its associated topology, its computational cost, the type of the problem being considered, the task at hand  \dots\ Consequently it is often of great interest to understand the links/relationships between them. \emph{Integral Probability Metrics} (IPMs) introduced by \citet{Mueller1997IntegralPM} (see also \citealp{sriperumbudur2009integral,Bharath}) offer an important class of distances that take the  form
\begin{equation}
\label{eq:ipmdefinition}
d_{\Gcal}(\pi,\pi'):=\sup_{g \in \Gcal} |\int g \dr \pi -\int g \dr \pi'| \,,
\end{equation}
where $\pi,\pi'$ are appropriately integrable distributions and $\Gcal$ is a class of real-valued functions parameterizing the distance. The choice of an adequate function class $\Gcal$ whose generated IPM faithfully describes the ``right notion'' of discrepancy is not straightforward. One possibility is to choose $\Gcal$ based on the learning task, for example by considering functions $g \in \Gcal$ that depend on the loss and the hypothesis space. This produces \emph{task-specific} pseudo-metrics\footnote{A pseudo-metric $D$ satisfies all the axioms of a metric except (possibly) for separation. In other words, $D$ is symmetric $D(x,y)=D(y,x)$, non-negative $D(x,y) \geq 0$, satisfies the triangular inequality $D(x,y) \leq D(x,z)+D(z,y)$ and is such that $D(x,x)=0$ (but possibly $D(x,y)=0$ for some $x\neq y$). } between probability distributions, abreviated as $\operatorname{TaskMetric}$, that can be used, \textit{inter alia}, to obtain bounds on the generalization error of a learning task \citep{bendavid,JMLR:Reid}. Another possibility is to rely on \emph{task-agnostic} IPM and to choose $\Gcal$ based on the prior knowledge that this class is appropriate for the task at hand. Notable examples of task-agnostic IPMs include the popular Maximum Mean Discrepancies (MMD) (when $\Gcal$ is the unit ball in a \emph{Reproducible Kernel Hilbert Space} (RKHS), see \citealp{Berlinet}) and the $1$-Wasserstein distance $\W_1$ (when $\Gcal$ is the class of $1$-Lipschitz functions, see \citealp{Villani}). Both are gaining interest from the machine learning community due to their ability to handle the metric structure of the feature space (see \citealp{cot_peyre_cutu,Muandet_2017} and references therein).  

Our first contribution is to exhibit some relationships between task-specific metrics between probability distributions, MMD and optimal transport (OT) distances. We first give necessary and sufficient conditions, on the kernel that defines the RKHS, under which the MMD can be bounded by a Wasserstein distance. We study in a second step the other direction, more difficult to obtain, which corresponds to finding the conditions under which the Wasserstein distance $\W_p$ can be upper-bounded by an MMD with a ‘‘Hölder'' exponent, that is when 
\begin{equation}
\label{eq:eqintro}
\W_p(\pi,\pi') \lesssim \operatorname{MMD}^{\delta}(\pi,\pi') \text{ for some } \delta \in (0,1]\,.
\end{equation}
Especially, we are interested in MMDs associated to RKHSs generated by translation-invariant positive semi-definite kernels that are widely used in many machine learning applications and are at the core of many large-scale learning algorithms \citep{rhaimi2,Rahimi}. Despite some connections between MMDs and \emph{regularized} OT distances, such as the Sinkhorn divergences \citep{pmlr-v89-feydy19a} or Gaussian smoothed OT \citep{nietert2021smooth,zhang2021convergence}, little is known regarding the relationships between non-regularized $\W_p$ and such MMDs. We show that the bound \eqref{eq:eqintro} can not hold in full generality and that one needs to find additional constraints on the distributions $\pi,\pi'$. This will be formalized by the means of a \emph{model set} of distributions $\Sfrak$, so that \eqref{eq:eqintro} applies for every $\pi,\pi' \in \Sfrak$. We shed light on several controls of the type \eqref{eq:eqintro} depending on the properties of this model set $\Sfrak$ and the TI kernel (see Section \ref{sec:wass_mmd}).

\begin{figure}[t!]
\centering
\includegraphics[width=\linewidth]{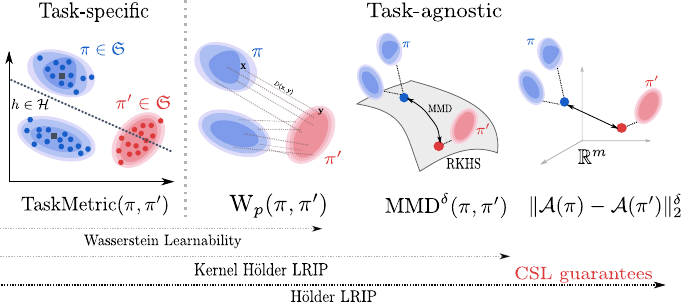}
\caption{\label{fig:fantasticfig} The reasoning used in the paper to obtain compressive statistical learning guarantees. \emph{(left)} Given two distributions $\pi,\pi'$ on a model set $\Sfrak$, our goal is to control some task-specific metric $\operatorname{Task Metric}(\pi,\pi')$ that depends on the learning task. \emph{(middle left)} In Section \ref{sec:wasserstein_learnability}, we use an upper-bound $\operatorname{Task Metric}(\pi,\pi')\lesssim \W_p(\pi,\pi')$ by introducing the notion of Wasserstein regularity of the task. \emph{(middle right)} In Section \ref{sec:wass_mmd}, we first show how to control the MMD by the Wasserstein distance, then we study the other direction that is controling $\W_p$ by an MMD with a Hölder exponent $\delta \in (0,1]$: $\W_p(\pi,\pi')\lesssim \operatorname{MMD}^{\delta}(\pi,\pi')$. \emph{(right)} In Section \ref{sec:compress_section} we discuss how to control the MMD by the distance between the finite dimensional sketches of the distributions $\Acal(\pi),\Acal(\pi')$ in $\R^{m}$. The whole pipeline gives the Hölder LRIP property which allows us to derive CSL guarantees.}
\end{figure}

This study is motivated by the compressive statistical learning (CSL) framework whose aim is to provide resource-efficient large-scale learning algorithms \citep{gribonval2020compressive,gribonval2020statistical,keriven:hal-01329195} and which heavily relies on MMDs with TI kernels. Large-scale ML faces nowadays a number of computational challenges, due to the high dimensionality of data and, often, very large training collections. Compressive statistical learning is one remedy to this situation. Its objective is: 1) to summarize a large dataset $\X \in \R^{d \times n}$, where $d$ is the dimension and $n$ the number of samples, into a single vector $\sbf \in \R^{m} \text{ or } \mathbb{C}^{m}$ with $m \ll nd$; and 2) to rely \emph{solely} on $\sbf$ to solve the learning task, such as finding centroids in K-means or learning mixture models \citep{compressivekmeans,keriven:hal-01329195,gribonval2020statistical}. The generic idea behind compressive learning is that, for many tasks, we only need to have access to informations from a ‘‘low-dimensional'' subspace, captured by a well-designed sketch vector $\sbf$. 

This framework requires specific statistical tools for establishing learning guarantees compared to standard machine learning approaches. One of the main notion in this context is found in the \emph{Lower Restricted Isometric Property} (LRIP) which is a condition on the sketching operator that maps a dataset to a sketch. However, this property is far from trivial to prove and is usually obtained by: 1) carefully designing a model set of distributions $\Sfrak$; 2) finding a kernel whose MMD dominates $\operatorname{Task Metric}$, a property being known as the \emph{Kernel LRIP}; and 3) approximating this MMD using random features \citep{gribonval2020compressive}. 

Based on the relationships between the MMD and the Wasserstein distance discussed above we will show that a slightly different property, namely the \emph{Kernel Hölder LRIP}, can be proved for a wide range of tasks where it is natural to control $\operatorname{Task Metric}$ by a Wasserstein distance (\emph{Wasserstein regularity}). In particular we prove that many unsupervised learning tasks such as \emph{compression-type tasks} (K-means/medians, PCA, see \citealp{gribonval2020compressive}) or supervised learning tasks, such as regression and binary classification with Lipschitz regressors/classifiers, fall into this category. From this study we will propose a property which generalizes the LRIP, namely the \emph{Hölder LRIP}, and we will show that this property also comes with interesting compressive statistical learning guarantees. Figure \ref{fig:fantasticfig} summarizes the whole reasoning used in this paper to establish these CSL guarantees.

\paragraph{Organization of the paper} We start by presenting in Section \ref{sec:wass_mmd} the relations between the Wasserstein distance and the MMD. We provide conditions so that $\W_p \lesssim \MMD^{\delta}$ holds for some $\delta \in (0,1]$. In Section \ref{sec:wasserstein_learnability} we study the relations between task-specific metrics between probability distributions and the Wasserstein distance. For this, we introduce the concept of \emph{Wasserstein regularity} of the learning task. In Section \ref{sec:compress_section} we introduce the compressive statistical learning framework which motivates our study. We study a generalization of the LRIP, namely the Hölder LRIP, and we show that this property has many advantages for CSL.

\subsection{Notations and Definitions}

We first detail the different usual notations and definitions used in this article.

\subsubsection{Metric Spaces} In this article the space $\Xcal$ will always be a complete, separable metric space. The relation $d(\xbf,\ybf) \lesssim d'(\xbf,\ybf)$ hides a multiplicative constant, \ie\ $d(\xbf,\ybf) \leq  C d'(\xbf,\ybf)$ with $C>0$ that \emph{does not depend} on $\xbf,\ybf$.  The class of $L$-Lipschitz continuous functions from a metric space $(\Xcal,d_\Xcal)$ to $(\Ycal,d_\Ycal)$ is denoted by $\operatorname{Lip}_{L}((\Xcal,d_\Xcal),(\Ycal,d_\Ycal))$ or simply by $\operatorname{Lip}_{L}(\Xcal,\Ycal)$ when it is clear from  context. If $f\in \operatorname{Lip}_{L}((\Xcal,d_\Xcal),(\Ycal,d_\Ycal))$ we have $\forall \xbf,\xbf' \in \Xcal, d_{\Ycal}(f(\xbf),f(\xbf'))\leq L d_{\Xcal}(\xbf,\xbf')$. In the following $\|\cdot\|_2$ denotes the $\ell_2$ norm, and vectors and matrices are written in bold. On a normed space $(\Xcal,\|\cdot\|)$, the ball centered at $\xbf_0 \in \Xcal$ and with radius $R>0$ is denoted $B_{\|\cdot\|}(\xbf_0,R)$ or simply by $B(\xbf_0,R)$ when it is clear from context. 

\subsubsection{Measures and Probability Distributions} We note $\P(\Xcal)$ the set of probability measures on $\Xcal$. $\Mcal(\Xcal)$ is the space of finite signed measures on $\Xcal$. For the sake of brevity, for a probability distribution $\pi \in \P(\R^{d})$ that admits a density $f$ \textit{w.r.t.} the Lebesgue measure on $\R^{d}$ we adopt the notation $\pi = f \dr \xbf$. Given a probability distribution $\pi \in \P(\Xcal)$ and a measurable function $T:\Xcal \rightarrow \Ycal$ the pushforward operator $\#$ defines a probability distribution $T\#\pi \in \P(\Ycal)$ \textit{via} the relation $T\#\pi(A)=\pi(T^{-1}(A))$ for every measurable set $A$ in $\Ycal$. In other words, if $X \sim \pi$ is a random variable then $Y=T(X)$ has the law $T\#\pi$. The support of a probability distribution is denoted as $\supp(\pi)$ and it is defined as the smallest closed set $S$ such that $\pi(S)=1$.

\subsubsection{Integrability, Fourier Transform and Sobolev Space} For a measurable space $\Xcal$ and a Borel measure $\mu$ on $\Xcal$ we note $L_p(\mu)$ the space of real-valued $p$-integrable functions \textit{w.r.t} $\mu$, \ie\ that satisfy $\int_{\Xcal} |f(\xbf)|^{p} \dr \mu(\xbf) <+\infty$. When $\Xcal=\R^{d}$ we note $L_p(\R^{d})$ the space of $p$-integrable functions with respect to the Lebesgue measure. For an integrable function $f \in L_1(\R^{d})$ we adopt the convention of the Fourier transform $\hat{f}(\w)=\mathcal{F}[f](\w):=\int_{\R^{d}} e^{-i\w^{\top}\xbf} f(\xbf) \dr \xbf$. The Fourier transform of a non-negative finite measure $\mu \in \Mcal_{+}(\R^{d})$ is defined for $\w \in \R^{d}$ by $\widehat{\mu}(\w):= \int_{\R^{d}} e^{-i\w^{\top} \xbf} \dr \mu(\xbf)$. For $s \geq 0$, we define the Sobolev space of order $s$ as \citep{adams2003sobolev}:
\begin{equation*}
H^{s}(\R^{d}):=\left\{f \in L_2(\R^{d}): \w \rightarrow (1+\|\w\|_2^2)^{s/2}\mathcal{F}[f](\w) \in L_2(\R^{d})\right\}\,.
\end{equation*}
It is a Hilbert space whose corresponding norm is $\|f\|_{H^{s}(\R^{d})} := \left(\int_{\R^d} (1+\|\w\|_2^2)^{s} |\mathcal{F}[f](\w)|^{2} \dr \w\right)^{1/2}$. It corresponds to the space of functions whose weak derivatives up to order $s$ are squared-integrable.

\section{Controlling Wasserstein Distances by Kernel Norms}
\label{sec:wass_mmd}

We focus in this section on the first main contributions of this paper, that is the comparison of optimal transport distances and maximum mean discrepancies. We begin by describing the main notions related to these two metrics.

The interest of optimal transport lies in both its ability to provide correspondences between sets of points and its ability to induce a geometric notion of distance between probability distributions thanks to the popular Wasserstein distances \citep{Villani,San15a,cot_peyre_cutu}. Considering a complete and separable metric space $(\Xcal,D)$ and $p \in [1,+\infty)$, the Wasserstein distance of order $p$ between two probability distributions $\pi, \pi' \in \P(\Xcal)$ is defined as
\begin{equation}
\label{eq:wassdef}
\W_p(\pi,\pi'):=\left(\inf_{\gamma \in \Pi(\pi,\pi')} \int_{\Xcal \times \Xcal} D(\xbf,\ybf)^{p} \dr \gamma(\xbf,\ybf)\right)^{1/p}\,,
\end{equation}
where $\Pi(\pi,\pi')$ is the set of couplings of $\pi$ and $\pi'$ \ie\ the set of joint distributions $\gamma \in \P(\Xcal \times \Xcal)$ such that both marginals of $\gamma$ are respectively $\pi$ and $\pi'$. More formally $\Pi(\pi,\pi')=\{ \gamma \in \P(\Xcal \times \Xcal): \forall A,B \subseteq \Xcal, \gamma(A \times \Xcal)=\pi(A), \ \gamma(\Xcal \times B)=\pi'(B)\}$. This quantity satisfies all the axioms of a distance and endows the space 
\begin{equation*}
\label{eq:pqdef}
\P_p(\Xcal):=\{\pi \in \P(\Xcal): \int_{\Xcal} D(\xbf_0,\ybf)^{p} \dr \pi(\ybf)<+\infty \text{ for some arbitrary } \xbf_0 \in \Xcal\}\,,
\end{equation*}
with a metric structure\footnote{The space $\P_p(\Xcal)$ is here to formalize that $\W_p$ is finite and thus defines a proper distance.}  \citep{Villani}. When $(\Xcal,D)$ is a normed space such as $(\R^{d},\|\cdot\|_2)$ the space $\P_p(\Xcal)$ is the space of probability distributions with finite $p$-th moment $\int_{\Xcal} \|\xbf\|_2^{p} \dr \pi(\xbf)<+\infty$. More generally, we can define OT problems by using a cost function $c: \Xcal \times \Xcal \rightarrow \R$ instead of a distance $D$ and by minimizing the quantity $\int c(\xbf,\ybf) \dr \gamma(\xbf,\ybf)$ over $\gamma \in \Pi(\pi,\pi')$.
With a slight abuse of terminology we will denote the optimal value of both problems by the term \emph{Wasserstein distance} and we will specify, when necessary, the choice of the cost function. A coupling $\gamma^{*}$ minimizing \eqref{eq:wassdef} is called \emph{optimal coupling} and it provides a probabilistic matching of the points in the support of the distributions $\pi,\pi'$. As such, computing an OT distance equals to finding the most cost-efficient way to ‘‘match'' one distribution to the other. An important property of the Wasserstein distance relies on its dual formulation. It allows, among others, to characterize $\W_1$ by considering the maximization problem 
\begin{equation*}
\label{eq:duality}
\W_1(\pi,\pi')=\sup_{f \in \operatorname{Lip}_1(\Xcal,\R)} |\int f(\xbf) \dr \pi(\xbf)-\int f(\ybf) \dr \pi'(\ybf)|\,,
\end{equation*}  
where $\operatorname{Lip}_1(\Xcal,\R)$ is the set of $1$-Lipschitz function from $(\Xcal,D)$ to $\R$ \citep{San15a}. 

The other important technical ingredient of this section, the theory of kernels, has a long history when it comes to learning problems or more generally to probability and statistics \citep{aronszajn50reproducing,Berlinet,Muandet_2017}. In the rest of the paper $\kappa$ will denote a \emph{positive semi-definite} (PSD) kernel\footnote{A function $\kappa:\Xcal \times \Xcal \rightarrow \C$ is a PSD kernel if it is \emph{Hermitian} \ie\ $\kappa(\xbf,\ybf)=\overline{\kappa(\ybf,\xbf)}$ and for all $\xbf_1,\cdots,\xbf_n \in \Xcal $ and any $c_1,\cdots,c_n \in \C$ we have $\sum_{i,j=1}^{n} c_i\overline{c_j} \kappa(\xbf_i,\xbf_j)\geq 0$.} on a space $\Xcal$. It defines a Hilbert space of functions from $\Xcal$ to $\C$ denoted by $\Hcal_\kappa$ endowed with an inner product $\langle\cdot,\cdot\rangle_{\Hcal_\kappa}$. This space is called a reproducing kernel Hilbert space and is characterized by the property $\forall \xbf \in \Xcal, \kappa(\cdot, \xbf) \in \Hcal_\kappa$ and the reproducing property: each $f \in \Hcal_\kappa$ can be evaluated as $f(\xbf)=\langle f, \kappa(\cdot,\xbf)\rangle_{\Hcal_\kappa}$ for any $\xbf \in \Xcal$. A PSD kernel also defines the so-called \emph{Maximum Mean Discrepancy} (MMD) which can be used to compare two probability distributions $\pi \in \P(\Xcal)$ and $\pi' \in \P(\Xcal)$ with the formula\footnote{When the kernel $\kappa$ is bounded, the MMD $\|\pi-\pi'\|_{\kappa}$ is finite for any probability distributions $\pi,\pi'$.}
\begin{equation*}
\MMD_{\kappa}(\pi,\pi'):=\left(\underset{\xbf,\xbf'\sim \pi}{\E}[\kappa(\xbf,\xbf')] +\underset{\ybf,\ybf'\sim \pi'}{\E}[\kappa(\ybf,\ybf')]-2 \operatorname{Re}(\underset{\xbf \sim \pi, \ybf \sim \pi'}{\E}[\kappa(\xbf,\ybf)])\right)^{1/2}\,.
\end{equation*}
This quantity defines a pseudo-metric on the space of probability distributions and is a true metric when the kernel is \emph{characteristic}: $\MMD_{\kappa}(\pi,\pi')=0 \iff \pi=\pi'$ \citep{simongabriel2020metrizing,Sriperumbudur}. The MMD is also characterized by the relation $\MMD_{\kappa}(\pi,\pi')=\sup_{\|f\|_{\Hcal_{\kappa}\leq 1}} |\int f(\xbf) \dr\pi(\xbf)-\int f(\xbf) \dr\pi'(\xbf)|$. Moreover, it can be extended to any finite signed measure $\mu \in \Mcal(\Xcal)$ by defining a semi-norm\footnote{A semi-norm $\|\cdot\|$ on a vector space is non-negative, satisfies the triangle inequality, is such that: a) if $\xbf=0$ then $\|\xbf\|=0$ (but not necessarily the converse); and b) for $\lambda \in \R$, $\|\lambda \xbf\|=|\lambda| \|\xbf\|$.} on $\Mcal(\Xcal)$ with the formula 
\begin{equation}
\label{eq:mmdnormdef}
\|\mu\|_{\kappa}:=\left(\int \int \kappa(\xbf,\ybf)\dr \mu(\xbf)\dr \mu(\ybf)\right)^{1/2}\,.
\end{equation}
When $\kappa$ is a PSD kernel this quantity is well defined, \ie\ the integral in \eqref{eq:mmdnormdef} is non-negative, and we have $\forall \pi,\pi' \in \P(\Xcal), \MMD_{\kappa}(\pi,\pi') = \|\pi-\pi'\|_{\kappa}$. In the rest of the paper we informally denote $\|\cdot\|_{\kappa}$ by the term \emph{kernel norm} or \emph{MMD norm}. An important family of kernels, namely \emph{translation-invariant (TI), PSD kernels}, are particularly interesting in our context. They are defined for $\Xcal=\R^{d}$ and when $\kappa(\xbf,\ybf)=\kappa_{0}(\xbf-\ybf)$ for some \emph{continuous} PSD function\footnote{A function $\kappa_{0}: \R^{d} \rightarrow \C$ is PSD if for all $\xbf_1,\cdots,\xbf_n \in \R^{d}$ and $c_1,\cdots,c_n \in \mathbb{C}$ we have $\sum_{i,j=1}^{n}c_{i} \overline{c_j}\kappa_{0}(\xbf_i-\xbf_j) \geq 0$. Such function is bounded $|\kappa_{0}(\xbf)| \leq \kappa_{0}(0)$ and  satisfies $\kappa_0(-\xbf)=\overline{\kappa_0(\xbf)}$ \citep[Theorem 6.2]{Wendland}. When $\kappa_0$ is even ($\kappa_0(-\xbf)=\kappa_0(\xbf)$) then $\kappa_0$ and thus $\kappa$ are real-valued.}  $\kappa_0:\R^{d} \rightarrow \C$ . This family encompasses many popular kernels such as Gaussian or Laplacian kernels, or kernels of the Matèrn class \citep{Sriperumbudur}. The following characterization of such kernels is due to the celebrated Bochner's theorem (see Theorem 6.6 and Theorem 6.11 in \citealp{Wendland}):
\begin{restatable}[Bochner]{theorem}{bochner}
\label{theo:bochner}
Let $\kappa_0: \R^{d} \rightarrow \C$. A function $\kappa$ of the form $\kappa(\xbf,\ybf)=\kappa_{0}(\xbf-\ybf)$, where $\kappa_0$ is continuous, is a PSD kernel if and only if there exists a probability distribution $\Lambda\in \P(\R^{d})$ such that
\begin{equation*}
\forall \xbf \in \R^{d}, \ \kappa_0(\xbf)= \kappa_0(0)\int_{\R^{d}} e^{-i\w^{\top}\xbf}\dr \Lambda(\w)\,.
\end{equation*}
If $\kappa_0$ is continuous \emph{and} in $L_1(\R^{d})$ then $\kappa(\xbf,\ybf)=\kappa_{0}(\xbf-\ybf)$ is a PSD kernel if and only if $\forall \w \in \R^{d}, \widehat{\kappa_{0}}(\w) \geq 0$.
\end{restatable}
Bochner's theorem shows that a translation invariant PSD kernel $\kappa$ (when properly scaled to ensure $\kappa_0(0)=1$) can be written as an expectation $\kappa(\xbf,\ybf)=\E_{\w \sim \Lambda}[\phi(\xbf, \w)\overline{\phi(\ybf, \w)}]$ where $\Lambda \in \P(\R^{d})$ and $\phi(\xbf, \w)=e^{-i \w^{\top}\xbf}$. An interesting property of such kernels is that they can be approximated using finite dimensional vectors by sampling from the frequencies $\w \sim \Lambda$ and approximating $\E_{\w \sim \Lambda}[\phi(\xbf, \w)\overline{\phi(\ybf, \w)}]$ using a Monte-Carlo algorithm  \citep{Zhu,sutherland2015error,NIPS2015_d14220ee}. This property is at the core of methods that rely on \emph{random Fourier features} to accelerate kernel learning algorithms \citep{Rahimi,rhaimi2}. 

\subsection{Controlling MMDs by Wasserstein distances} 

When it comes to comparing $\MMD$ and $\W_p$, one direction is easier: controlling $\MMD$ by $\W_p$. More precisely we have the following result (the proof can be found in Appendix \ref{sec:proof:prop:mmdboundedwass}): 
\begin{restatable}{proposition}{mmdboundedwass}
\label{prop:mmdboundedwass}
Let $(\Xcal,D)$ be a complete separable metric space, $\kappa: \Xcal \times \Xcal \rightarrow \R$ a PSD kernel, $\Hcal_{\kappa}$ the associated RKHS and $B_{\kappa}:=\{f \in \Hcal_{\kappa}: \|f\|_{\Hcal_{\kappa}} \leq 1\}$ the unit ball in $\Hcal_{\kappa}$. Consider the Wasserstein distances computed with the metric $D$. For any $C>0$ the following statements are equivalent:
\begin{enumerate}[label=(\roman*)]
\item 
\begin{equation}
B_\kappa \subseteq \operatorname{Lip}_{C}( (\Xcal,D), \R)
\end{equation}
\item \begin{equation}
\label{eq:wass_control_mmd}
\forall p \in [1,+\infty), \forall \pi,\pi' \in \P_p(\Xcal), \|\pi-\pi'\|_{\kappa} \leq C \W_{p}(\pi,\pi')
\end{equation}
\item 
\begin{equation}
\exists p \in [1,+\infty), \forall \pi,\pi' \in \P_p(\Xcal), \|\pi-\pi'\|_{\kappa} \leq C \W_{p}(\pi,\pi')
\end{equation}
\item 
\begin{equation}
\label{eq:eqcond}
 \forall \xbf,\ybf \in \Xcal, \  \kappa(\xbf,\xbf)+\kappa(\ybf,\ybf)-2\kappa(\xbf,\ybf) \leq C^{2}  D^{2}(\xbf,\ybf)
\end{equation}
\end{enumerate}
\end{restatable}

For the sake of clarity, we restrict ourselves to the case where $D$ is a proper metric but extensions of this result are possible by considering an OT problem with a more general cost. In particular, this type of bound has already been considered in \citet{Arbel,Sriperumbudur} with the pseudo-metric $D(\xbf,\ybf)=\|\kappa(\xbf,\cdot)-\kappa(\ybf,\cdot)\|_{\Hcal_{\kappa}}$ which gives $C=1$ and an equality in \eqref{eq:eqcond}.   As a corollary of this proposition we have the following result (see Appendix \ref{sec:proof:prop:mmdboundedwass} for a proof):

\begin{restatable}{corollary}{corrmmdboundedwass}
\label{corr:mmdboundedwass}
Consider $\Xcal=\R^{d}$ equipped with the Euclidean distance $D(\xbf,\ybf)=\|\xbf-\ybf\|_2$ and a PSD kernel $\kappa: \R^{d} \times \R^{d} \rightarrow \R$ that is normalized, \ie\ $\kappa(\xbf,\xbf)=1$ for every $\xbf \in \Xcal$. Assume that for each $\xbf \in \Xcal$ the function $\phi_{\xbf} : \ybf  \mapsto \kappa(\xbf,\ybf)$ is $C^{2}$ in a neighborhood of $\xbf$, and denote $\mathbf{H}_{\xbf} = -\nabla^{2} [\phi_{\xbf}](\xbf)$ its negative Hessian matrix evaluated at $\xbf$. Then the following holds:
\begin{enumerate}[label=(\roman*)]
\item Any of the four equivalent properties of Proposition \ref{prop:mmdboundedwass} implies
\begin{equation}
\label{eq:HessianCondition}
\sup_{\xbf \in \R^{d}} \lambda_{\max}(\mathbf{H}_{\xbf}) \leq C^{2}\,,
\end{equation}
where $\lambda_{\max}(\mathbf{H}_{\xbf})$ denotes the largest eigenvalue of $\mathbf{H}_{\xbf}$.
\item If $\kappa$ is translation invariant, \ie\  $\kappa(\xbf,\ybf)=\kappa_{0}(\xbf-\ybf)$ for every $\xbf,\ybf \in \Xcal$,
then conversely, \eqref{eq:wass_control_mmd} holds with $C:= \sqrt{\sup_{\xbf} \lambda_{\max}(\mathbf{H}_{\xbf})} = \sqrt{\lambda_{\max}(-\nabla^{2}[\kappa_{0}](0))}$.
\end{enumerate}
\end{restatable}

The second point of the previous result shows that under mild assumptions on a TI kernel the MMD is bounded by a constant times a Wasserstein distance, for \emph{any} distributions $\pi,\pi'$ for which these quantities are well-defined. In particular it holds for popular kernels such as the Gaussian kernel, or kernels of the Matérn class with parameter\footnote{In this case $\kappa_0$ is $C^{2}$ in a neighbourhood of $0$ since $\kappa_0 \in L_1(\R^{d})$ and $\w \rightarrow \|\w\|_2^{2} \widehat{\kappa_0}(\w) \in L_1(\R^{d})$ when $\nu >1$} $\nu > 1$:
\begin{myexample}\label{ex:matern}
An important family of TI kernels is the Matérn class \citep[Section 4.2.1]{Rasmussen}, given in any dimension by the relation $\kappa(\xbf,\ybf)\coloneqq \frac{2^{1-\nu}}{\Gamma(\nu)}(\frac{\sqrt{2\nu}\|\xbf-\ybf\|_2}{\sigma})^{\nu}K_{\nu}(\frac{\sqrt{2\nu}\|\xbf-\ybf\|_2}{\sigma})$ for $\nu >0, \sigma >0$ where $\Gamma$ is the gamma function, and $K_{\nu}$ is the modified Bessel function of the second kind of order $\nu$. This family of kernel admits the following Fourier transform\footnote{See \citet[Section 4.2.1]{Rasmussen} with slightly modified conventions on Fourier transforms.} :
\begin{equation}
\label{eq:matern_equation}
\widehat{\kappa_0}(\w)=\frac{2^{d+\nu} \pi^{d/2} \Gamma(\nu+d / 2) \nu^{\nu}}{\Gamma(\nu) \sigma^{2 \nu}}\left(\frac{2 \nu}{\sigma^{2}}+\|\w\|_{2}^{2}\right)^{-(\nu+d / 2)}\,.
\end{equation}
Interestingly, $\nu=\frac{1}{2}$ corresponds to the Laplacian kernel $\kappa(\xbf,\ybf)=\exp(-\|\xbf-\ybf\|_2/\sigma)$ whose Fourier transform is $\frac{2^{d}\pi^{\frac{d-1}{2}}\Gamma(\frac{d+1}{2})}{\sigma}(\frac{1}{\sigma^{2}} +\|\w\|_2^{2})^{-\frac{d+1}{2}}$ while $\nu \rightarrow +\infty$ recovers the RBF kernel see \citet[Section 4.2.1]{Rasmussen}\footnote{Likewise, with adapted conventions on Fourier transforms.}.
\end{myexample}

Note that when the kernel is TI but is not normalized the second point of Corollary \ref{corr:mmdboundedwass}  holds also with $C=\kappa_{0}(0)\sqrt{\lambda_{\max}(-\nabla^{2}[\kappa_{0}](0))}$.  For other types of \emph{normalized} kernels, condition \eqref{eq:eqcond} is a necessary and sufficient condition that amounts to checking if there is a constant $C>0$ such that $1-\kappa(\xbf,\ybf) \leq \frac{C^{2}}{2} D^{2}(\xbf,\ybf)$ for all $\xbf,\ybf \in \Xcal$. Interestingly, it echoes the ‘‘$C$-strongly locally characteristic'' property of the kernel as in \citet[Definition~5.14]{gribonval2020statistical} but with the reverse inequality. When the kernel is $C^{2}$ a necessary condition is given by the maximum eigenvalue of the negative Hessian as in \eqref{eq:HessianCondition}. 

Overall Proposition \ref{prop:mmdboundedwass} shows that it is not too difficult to find necessary and sufficient conditions under which the MMD can be controlled by a Wasserstein distance. What is more difficult to characterize is the inequality in the other direction. 

\subsection{Controlling Wasserstein distances by MMDs ?}

Thereafter, the objective is thus to find reasonable conditions on a subset of probability distributions $\Sfrak \subseteq \P(\Xcal)$ and on a PSD kernel $\kappa$ such that the Wasserstein distance can be controlled with the MMD with kernel $\kappa$ uniformly on $\Sfrak$. We adopt the following definition:
\begin{restatable}{definition}{wasslearn_emb}
\label{def:kappaembedable}
Let $\Sfrak \subseteq \P(\Xcal)$ be a subset of probability distributions, $p \in [1,+\infty)$, $\kappa$ a real-valued PSD kernel on $\Xcal$ and $\delta \in (0,1]$. We say that the space $(\Sfrak,\W_p)$ is $(\kappa,\delta)$-embeddable with error $\eta \geq 0$ if
\begin{equation}
\label{eq:eqtoprove}
\exists C >0, \forall \pi,\pi' \in \Sfrak, \W_p(\pi,\pi')\leq C \|\pi-\pi'\|^{\delta}_{\kappa}+\eta\,.
\end{equation}
When $\eta=0$ we simply say that $(\Sfrak,\W_p)$ is $(\kappa,\delta)$-embeddable.
\end{restatable}

Note that the constants $C,\eta,\delta$ in \eqref{eq:eqtoprove} \emph{do not depend} on the probability distributions $\pi,\pi'$: we want to bound uniformly on the whole subset $\Sfrak$. In the following, we will call \emph{model set} this subset $\Sfrak$. As discussed later in Section \ref{sec:compress_section}, introducing $\Sfrak$ will also be crucial in order to obtain compressive statistical learning guarantees. Moreover, we are particularly interested in establishing such an inequality for translation-invariant PSD kernels that at the core of the CSL theory since they admit a random Fourier feature expansion useful to find a sketching operator based on random Fourier features \citep{gribonval2020compressive}. 

\begin{remark}
An immediate consequence of Definition~\ref{def:kappaembedable} is that when $(\Sfrak,\W_p)$ is $(\kappa,\delta)$-embeddable (\ie\ with no error) then the kernel $\kappa$ is necessarily characteristic to $\Sfrak$ 
\citep[Section 1.2]{simongabriel2020metrizing},
in other words $\|\pi-\pi'\|_{\kappa} = 0 \iff \pi= \pi'$ for all $\pi,\pi' \in \Sfrak$ (indeed when the MMD vanishes then the Wasserstein distance also vanishes which implies equality of the distributions). Moreover, if $(\Sfrak,\W_p)$ is $(\kappa,\delta=1)$-embeddable and if the condition \eqref{eq:eqcond} is also fulfilled, then $\W_p$ and $\|\cdot\|_{\kappa}$ induce the same topology on $\Sfrak$ and define equivalent metrics on $\Sfrak$.
\end{remark}

\begin{remark}
\label{rem:monotonicityembedability}
If $\Sfrak \subseteq \Sfrak'$ where $(\Sfrak',\W_{p})$ is $(\kappa,\delta)$-embeddable then $(\Sfrak,\W_{p})$ is also $(\kappa,\delta)$-embeddable. In other words, if $\Sfrak$ is contained in a space that is $(\kappa,\delta)$-embeddable it is also $(\kappa,\delta)$-embeddable. On the other hand, if $\Sfrak'$ contains a subspace $\Sfrak$ for which there is a necessary condition to the $(\kappa,\delta)$-embeddability property then the same condition applies to $\Sfrak'$.
\end{remark}

In the following we focus on property~\eqref{eq:eqtoprove} with no error $\eta=0$. First we consider necessary conditions, that is, we argue that property~\eqref{eq:eqtoprove} with no error can only be expected to hold for a kernel $\kappa$ and a model set $\Sfrak$ if certain appropriate assumptions are made. 
Conversely, we then derive some sufficient conditions on $\Sfrak$ and $\kappa$ such that $(\Sfrak,\W_p)$ is $(\kappa,\delta)$-embeddable.

\subsection{Necessary Conditions}
Let us first review some necessary conditions for property~\eqref{eq:eqtoprove} with no errror.

\subsubsection{Boundedness of the Model Set is Necessary.}  Consider a model set $\Sfrak \subseteq \P_{1}(\R^{d})$ and denote by
\begin{equation*}
\label{eq:DefMean}
\operatorname{m}(\pi):=\int \xbf \dr \pi(\xbf)
\end{equation*} the mean of $\pi \in \P_{1}(\R^{d})$. On the one hand, simple calculus (Lemma \ref{lemma:wass_bound_mean} in Appendix \ref{sec:simple_bound_mean_wass}) shows that for any $\pi,\pi' \in \P(\R^{d})$ and $p\in [1,+\infty)$,
if $\W_{p}$ is defined based on some norm $\|\cdot\|$ and $\|\cdot\|_{\star}$ denotes the dual norm defined by $\|\zbf\|_{\star} = \sup_{\|\xbf\| \leq 1} \langle \xbf, \zbf \rangle$, then 
\begin{equation*}
\W_p(\pi,\pi') \geq  \| \operatorname{m}(\pi)-\operatorname{m}(\pi')\|_\star.
\end{equation*}
On the other hand, if $\kappa$ is a bounded PSD kernel (\ie,\ $\sup_{\xbf} \kappa(\xbf,\xbf) \leq K<+\infty$) then, by the Cauchy-Schwarz inequality for kernels we have $\forall \xbf, \ybf, \ |\kappa(\xbf,\ybf)| \leq \sqrt{\kappa(\xbf,\xbf)} \sqrt{\kappa(\ybf,\ybf)} \leq K$. Hence, for any $(\pi,\pi') \in \Sfrak, \ \|\pi-\pi'\|_{\kappa} \leq \sqrt{2K}$. As a result, if $\Sfrak$ is unbounded in the sense that
$\sup_{\pi,\pi' \in \Sfrak} \| \operatorname{m}(\pi)-\operatorname{m}(\pi')\|_{\star} = +\infty$, then for each $\delta>0$,
\begin{equation}
\label{eq:NoEmbedability}
\sup_{(\pi,\pi') \in \Sfrak} \frac{\W_p(\pi,\pi')}{\|\pi-\pi'\|^{\delta}_{\kappa}}=+\infty\,.
\end{equation}
Consequently, we can not have \eqref{eq:eqtoprove} for any $\delta>0$. Since all norms are equivalent in finite dimension the following lemma holds:
\begin{restatable}{lemma}{BoundedModelNecessary}
\label{lem:BoundedModelNecessary}
Consider $\Xcal = \R^{d}, p \in [1,+\infty)$ and assume that $\W_{p}$ is based on a norm on $\R^{d}$.
If $\kappa$ is bounded and $(\Sfrak,\W_p)$ is $(\kappa,\delta)$-embeddable for some $\delta > 0$ then $\Sfrak$ is 
bounded:  $$\operatorname{m-diam}(\Sfrak) : = \sup_{\pi,\pi' \in \Sfrak} \| \operatorname{m}(\pi)-\operatorname{m}(\pi')\|_2 < +\infty\,.$$
\end{restatable}

\subsubsection{Bounds on $\delta$ due to the Convergence Rate of Empirical Measures.}
Another obstacle to \eqref{eq:eqtoprove} concerns the samples rate of convergence of both terms with empirical measures : it is known that the Wasserstein distance suffers from the curse of dimensionality while the MMD does not. More precisely if $\pi \in \P_{1}(\R^{d})$ is absolutely continuous with respect to the Lebesgue measure on $\R^{d}$ then it is known that $\E[\W_1(\pi,\pi_n)]\gtrsim n^{-1/d}$ where $\pi_n=\frac{1}{n} \sum_{i=1}^{n} \delta_{\xbf_i}$, $\xbf_i \sim \pi$ and the expectation is taken \textit{w.r.t.} the draws of $\xbf_i $ \citep{dudley1969,weedbach2017}. By monotonicity of $\W_p$ in $p$ this is also true for $\W_p$ with $p\geq 1$ (since for $p\leq q, \W_p(\pi,\pi)\leq \W_q(\pi,\pi')$ for any\footnote{This is a consequence of Jensen inequality \citep[Section 5.1]{San15a}.} $\pi,\pi'$).  On the contrary, it is not difficult to see that if the PSD kernel $\kappa$ is bounded by $K$ then $\E[\|\pi-\pi_n\|_{\kappa}^{\delta}]\leq (2K)^{\delta/2} n^{-\delta/2}$ (see Lemma \ref{lemma:convergence_finite_sample_lemma_mmd} in Appendix \ref{sec:conv_finite_samples}). 
Consequently, even when the model set $\Sfrak \subseteq \P_{1}(\R^{d})$ satisfies $\operatorname{m-diam}(\Sfrak)<+\infty$ (to avoid the obstacles to \eqref{eq:eqtoprove}  already identified in Lemma~\ref{lem:BoundedModelNecessary}), if $\Sfrak$
is rich enough to contain a distribution $\pi$ that is absolutely continuous \textit{w.r.t.} the Lebesgue measure, as well 
as its empirical distributions $\pi_{n}$ for every $n$, then \eqref{eq:eqtoprove} implies $n^{-1/d} \lesssim n^{-\delta/2}$, so necessarily $\delta \leq 2/d$. An example of such a model set is the set of all probability distributions producing almost surely vectors in a prescribed ball, leading to the following result:
\begin{restatable}{lemma}{conditionond}
Consider $R>0$, $\Omega=B(0,R) \subseteq \Xcal = \R^{d}$,  $\Sfrak:=\{\pi \in \P(\Xcal): \pi(\Omega) =1\}$, $\kappa$ a bounded PSD kernel, and $\W_{p}$ based on a norm in $\R^{d}$ with $p\in [1,+\infty)$. If $(\Sfrak,\W_p)$ is $(\kappa,\delta)$-embeddable then $\delta \leq 2/d$.
\end{restatable}

In the context of CSL, as described in Section \ref{sec:compress_section}, such $\delta\leq 2/d$ would imply a very slow convergence rate of the order of $O(n^{-\frac{1}{d}})$. In other words, if the strategy described in Section \ref{sec:compress_section} is followed we would require an exponential amount of samples in order to have reasonable CSL guarantees which is problematic for a large scale scenario where $d$ is usually large. This discussion suggests that we must find suitable constraints on $p, \delta, \kappa$ and $\Sfrak$ to avoid such a curse of dimensionality. Sufficient conditions to achieve this goal will be discussed later, but first we continue with some additional necessary conditions.

\subsubsection{Another Bound on $\delta$ for Certain Model Sets}
Another restriction comes from the type of distributions in the model set. We will prove that, as soon as $\Sfrak$ contains two distributions whose supports are disjoint, as well as the convex segment between these distributions, we cannot hope to have \eqref{eq:eqtoprove} with error $\eta=0$ when $p\cdot \delta >1$. \begin{restatable}{proposition}{weedprop}
\label{weedprop}
Let $(\Xcal,D)$ be a complete and separable metric space and consider the Wasserstein distances computed with the distance $D$. Let
$\kappa$ be any PSD kernel. Consider two arbitrary probability distributions $\pi_0,\pi_1 \in \P(\Xcal)$  such that 
 $\|\pi_0-\pi_1\|_{\kappa}<+\infty$ and $\supp(\pi_0)$ and $\supp(\pi_1)$ are disjoint\footnote{We recall that the support $\supp(\pi)$ of a probability distribution $\pi \in \P(\Xcal)$ is  the smallest closed set $S$ such that $\pi(S)=1$.}.  Consider $\Sfrak:=\{(1-t)\pi_0+t\pi_1: t \in [0,1]\}$. If $(\Sfrak,\W_p)$ is $(\kappa,\delta)$-embeddable then $\delta \leq 1/p$.
\end{restatable}
The result is mostly based on \citet{nilesweed2020minimax}. Its proof in Appendix \ref{sec:proof_weed_prop} essentially amounts to showing~\eqref{eq:NoEmbedability} as soon as $p\cdot\delta>1$. Following Remark~\ref{rem:monotonicityembedability}, the same conclusion holds if $\Sfrak$ only \emph{contains} the convex combinations of distributions $\pi_0,\pi_1$ as in the above proposition. For a bounded kernel, since $\|\pi_0-\pi_1\|_{\kappa}$ is always finite, the same result is thus valid in particular when the model set $\Sfrak$ contains a segment whose extreme points have disjoint supports. This is notably the case when $\Sfrak$ is convex and contains two distributions with disjoint supports. As a consequence, given any  PSD kernel $\kappa$, $(\Sfrak,\W_p)$ is \emph{not} $(\kappa,\delta)$-embeddable for $\delta >1/p$ when $\Sfrak$ contains for example mixtures of two Diracs or more generally mixtures of two compactly supported distributions. We emphasize that this result does not depend on the dimension of the ambient space and is true for any PSD kernel. 

\subsubsection{Bound on $\delta$ for Mixture Models and Smooth TI Kernels}
In most concrete applications, one often has to compare \emph{discrete} distributions. We show in this section that the regularity of the kernel plays an important role when trying to control the Wasserstein distance with an MMD for model sets made of discrete distributions. In the following we define, for $K \in \mathbb{N}^{*}$ and $\Omega \subseteq \Xcal = \R^{d}$, the space of mixtures of $K$ diracs located in $\Omega$:
\begin{equation*}
\Sfrak_{K}(\Omega) := \Big\{\sum_{i=1}^{K} a_i \delta_{\xbf_i}: a_{i} \in \R_{+}, \sum_{i=1}^{K} a_i = 1, \forall i \in \integ{K}, \xbf_i \in \Omega\Big\}\,.
\end{equation*}
This type of model with $\Omega = B(0,R)$ for some $R>0$ plays a central role in compressive learning theory and is used to show that the LRIP (Section \ref{sec:compress_section}) does not hold for tasks such as K-means without separability assumptions on the diracs \citep{gribonval2020statistical}. We show in the next theorem (proof in Appendix \ref{sec:proof_theo_regularitykerneltheo}) that there is a trade-off between the exponent $\delta$ and the regularity of the kernel provided that the model set is rich enough to contain discrete distributions with enough diracs. 
\begin{restatable}{theorem}{regularitykerneltheo}
\label{theo:regularitykerneltheo}
Consider a TI, PSD kernel $\kappa(\xbf,\ybf)=\kappa_{0}(\xbf-\ybf)$ on $\R^{d}$ such that $\kappa_{0}$ is $k$ times differentiable at $0$ with $k \in \mathbb{N}^{*}$. Consider $p \in [1,+\infty)$,  a Wasserstein distance $\W_{p}$ based on a norm in $\R^{d}$, a vector $\xbf_{0} \in \R^{d}$, $R>0$ and $\Omega = B(\xbf_{0},R)$. 
If $(\Sfrak_{\lfloor \frac{k}{2} \rfloor +1}(\Omega),\W_p)$ is $(\kappa,\delta)$-embeddable then $\delta \leq 2/k$.
\end{restatable}
Following Remark~\ref{rem:monotonicityembedability}, the same conclusion holds if $\Sfrak$ only \emph{contains} all mixtures of Dirac supported in some arbitrary Euclidean ball. Theorem~\ref{theo:regularitykerneltheo} proves that if the kernel is $k$ times differentiable and if $\Sfrak$ is rich enough to contain $\lfloor \frac{k}{2} \rfloor +1$ diracs then we can not control the Wasserstein distance with $\operatorname{MMD}^{\delta}$ \emph{uniformly} over $\Sfrak$ when $\delta > 2/k$. As an immediate consequence we have the following corollary when the kernel is smooth:

\begin{restatable}{corollary}{corrolaryregularity}
\label{corr:corrolary_regularity}
Consider a TI, PSD kernel $\kappa(\xbf,\ybf)=\kappa_{0}(\xbf-\ybf)$ on $\R^{d}$ such that $\kappa_{0} \in C^{\infty}(\R^{d},\R)$ and a model set $\Sfrak \subseteq \P(\R^{d})$. Assume that $\Sfrak_{K}(\Omega) \subseteq \Sfrak$ with $K \geq 2$ where $\Omega \subseteq \R^{d}$ is an open set. \\
If $(\Sfrak,\W_p)$ is $(\kappa,\delta)$-embeddable, where  $\W_{p}$ is based on a norm in $\R^{d}$ and $p \in [1,+\infty)$, then $\delta \leq 2/K$.
\end{restatable}
These results have many consequences. First it shows that when $\kappa$ is smooth and $\Sfrak$ \emph{contains} mixtures of arbitrarily many diracs located in some open set, $(\Sfrak,\W_p)$ is \emph{not} $(\kappa,\delta)$-embeddable for any $\delta>0$. In other words, it proves that finding a absolute constant $C>0$ such that $\W_p(\pi,\pi') \leq C \operatorname{MMD}_{\kappa}^{\delta}(\pi,\pi')$ for all discrete distributions $\pi,\pi'$ is hopeless when the kernel $\kappa$ is smooth \emph{even if} these distributions lie also in some fixed ball of $\R^{d}$ (to take care of the necessary condition associated to Lemma~\ref{lem:BoundedModelNecessary}). It suggest that finding suitable constraints  on the model set $\Sfrak$ \emph{and} on the kernel $\kappa$ is required in order to have the control \eqref{eq:eqtoprove}. We will show in the next sections how to obtain these types of control with additional hypotheses on the regularity of the distributions in $\Sfrak$. The Figure \ref{fig:diagram} summarizes the necessary conditions established in the previous sections.

\begin{figure}[t!]
\begin{center}
\includegraphics[width=1\linewidth]{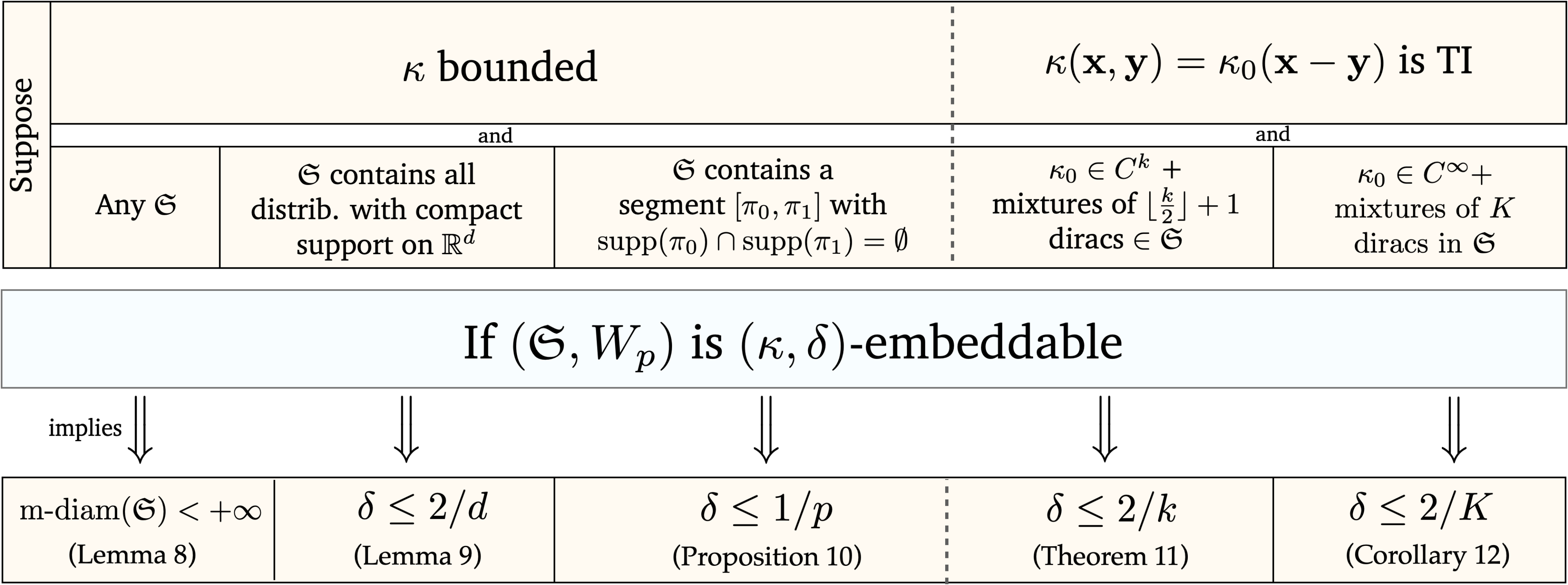}
\end{center}
\caption{\label{fig:diagram} Summary of the established necessary conditions to the $(\kappa,\delta)$-embedabbility property.}
\end{figure}

\subsection{Sufficient Conditions: Regular Distributions \label{sec:thecaseofregular}}

We are now interested in sufficient conditions allowing to uniformly control the Wasserstein distance by $\MMD^{\delta}$ on a subset of distributions $\Sfrak \subset \P(\R^{d})$. In the following we consider Wasserstein distances defined with respect to the Euclidean norm $\|\cdot\|_2$, and denote $$\moment_{\expo}[\pi]: = \left(\E_{\xbf \sim \pi}[\|\xbf\|^{\expo}_2]\right)^{1/\expo}$$
the moment of order $\expo$ of $\pi \in \P(\R^d)$. At first we restrict to the case of ``regular'' distributions, in the sense that probability distributions in $\Sfrak$ are assumed to admit densities with respect to the Lebesgue measure (non-regular distributions will be studied in the next section). We recall that the shorthand $\pi = f \dr \xbf$ indicates that $\pi$ has density $f$ with respect to the Lebesgue measure. 

Our first Lemma (proved in Appendix \ref{proof:lemma:bounding_wass}) controls $\W_p$ by a distance $L_2$ between densities, under the assumption that distributions in the model set $\Sfrak$ have a certain number of bounded moments:

\begin{restatable}{proposition}{boundingwass}
\label{lemma:bounding_wass}
Consider $\pi,\pi' \in \P(\R^{d})$ with densities $f,g$ with respect to the Lebesgue measure, \ie\ $\pi = f \dr \xbf, \pi' = g \dr \xbf$. If $\max\{\moment_{\expo}[\pi], \moment_{\expo}[\pi']\} \leq \cte$, where $r>1$, then
for each $1\leq p< \expo$ we have
\begin{equation}\label{eq:first_bounding_wass_L2}
\W_p(\pi,\pi') \leq C \left(\int_{\R^{d}}|f(\xbf)-g(\xbf)|^{2} \dr \xbf \right)^{\frac{\expo-p}{(d+2\expo)p}}\,,
\end{equation}
with $C = 2(\max\{V_d,1\})^{\frac{1}{2p}} \cte^{\frac{(d+2p)\expo}{(d+2\expo)p}}$ with $V_{d}=\pi^{d / 2}/ \Gamma(d / 2+1)$ the volume of the $d$-dimensional unit sphere.
\end{restatable}
The $L_2$ distance between densities that appears in the right hand side of~\eqref{eq:first_bounding_wass_L2} can be further bounded by an MMD with an appropriate kernel. Indeed, using Plancherel's formula 
and  introducing the Fourier transform $\widehat{\kappa_0}$ of a TI, PSD kernel, Cauchy-Schwarz inequality yields
\[
\int_{\R^{d}}|f(\xbf)-g(\xbf)|^{2} \dr \xbf \propto \int_{\R^{d}}|\hat{f}(\w)-\hat{g}(\w)|^{2} \dr \w \leq 
\Big(\int_{\R^{d}}\frac{|\hat{f}(\w)-\hat{g}(\w)|^{2}}{\widehat{\kappa_0}(\w)} \dr \w\Big)^{\frac{1}{2}}
\Big(\int_{\R^{d}} \widehat{\kappa_0}(\w)|\hat{f}(\w)-\hat{g}(\w)|^{2} \dr \w\Big)^{\frac{1}{2}}.
\]
where  $\hat{f}, \hat{g}$ denote the Fourier transforms of $f,g$. The second integral of the right hand side of this expression being proportional to the MMD (Lemma \ref{lemma:mmdform}) one can transform the bound~\eqref{eq:first_bounding_wass_L2} into a bound involving an MMD if we can control the integral $\int_{\R^{d}} \widehat{\kappa_0}(\w)^{-1}|\hat{f}(\w)-\hat{g}(\xbf)|^{2} \dr \w$ by a constant. Moreover, we also have the following relation (see\footnote{With adapted conventions on Fourier transforms.} \citealt[Theorem 10.12]{Wendland}):
\begin{equation*}
(2\pi)^{-d} \int_{\R^{d}} \frac{|\hat{f}(\w)-\hat{g}(\w)|^{2}}{\widehat{\kappa_0}(\w)} \dr \w = \|f-g\|^2_{\Hcal_{\kappa}}\,,
\end{equation*}
where $\Hcal_{\kappa}$ is the RKHS associated to the kernel $\kappa$ and $\|\cdot\|_{\Hcal_\kappa}$ is the corresponding RKHS norm. Consequently, when the distributions in $\Sfrak$ have densities in some RKHS ball, we can bound $\int_{\R^{d}} \widehat{\kappa_0}(\w)^{-1}|\hat{f}(\w)-\hat{g}(\xbf)|^{2} \dr \w$ by a constant:
\begin{restatable}{theorem}{regularcaserkhs}
\label{theo:regularcase_rkhs}
Let $\kappa(\xbf,\ybf)=\kappa_0(\xbf-\ybf)$ be a TI, PSD kernel on $\R^{d}$ such that $\kappa_{0} \in L_1(\R^{d})$, $\widehat{\kappa_{0}}(\w) > 0$ for every $\w$. For $B, \cte, \expo  \geq 0$, denote
\begin{equation}
\label{eq:modelsetforrkhs}
\Sfrak_{B, \cte, \expo, \kappa} :=  \left\{\pi \in \P(\R^d): \ \pi= f\dr \xbf,\ \|f\|_{\Hcal_{\kappa}}\leq B \text{ and } \ \moment_{\expo}[\pi] \leq \cte\right\} \subset \P_{r}(\R^{d})\,.
\end{equation}
If $\expo > 1$ then for each $1\leq p < \expo$ we have
\begin{equation*}
\forall \pi,\pi' \in \Sfrak_{B, \cte, \expo,\kappa},\ \W_p(\pi,\pi') \leq  C'\|\pi-\pi'\|_{\kappa}^{\frac{\expo-p}{p(d+2\expo)}}\,,
\end{equation*}
where  $C' = 8(\max\{V_d,1\})^{\frac{1}{2p}} B^{\frac{\expo-p}{(d+2\expo)p}} \cte^{\frac{(d+2p)\expo}{(d+2\expo)p}}$.
\end{restatable}

The proof is given in Appendix \ref{proof:theo:regularcase}. With the model set $\Sfrak = \Sfrak_{B, \cte, \expo, \kappa}$, this theorem implies that $(\Sfrak,\W_p)$ is $(\kappa,\delta = \frac{\expo-p}{p(d+2\expo)})$-embeddable for every $1 \leq p < \expo$ as soon as $\kappa$ is a TI, PSD kernel with very few assumptions. A limitation of this result is that the model set $\Sfrak$ depends on the kernel $\kappa$ so that it is not clear which family of distributions belongs to $\Sfrak$. In the next theorem we decouple the assumptions on the kernel from those on the model set. Assuming that the distributions have densities that are sufficiently regular (Sobolev), a certain number of bounded moments and with some assumptions on the kernel $\kappa$ the following holds:
\begin{restatable}{theorem}{regularcase}
\label{theo:regularcase}
Let $\kappa(\xbf,\ybf)=\kappa_0(\xbf-\ybf)$ be a TI, PSD kernel on $\R^{d}$ such that $\kappa_{0} \in L_1(\R^{d})$, $\widehat{\kappa_{0}}(\w) > 0$ for every $\w$, and assume there is $s_{\kappa}>0$ such that
\begin{equation}
\label{eq:regularity_kernel}
\frac{1}{\widehat{\kappa_0}(\w)}=O(\|\w\|_2^{s_{\kappa}}) \text{ as } \|\w\|_2 \rightarrow +\infty\,.
\end{equation}
For $\expo, B, \cte,s \geq 0$, denote
\begin{equation}
\label{eq:modelsetforsobol}
\Sfrak_{B, \cte, \expo,s} :=  \left\{\pi \in \P(\R^d): \ \pi= f\dr \xbf,\ \|f\|_{H^{s}(\R^d)}\leq B \text{ and } \ \moment_{\expo}[\pi] \leq \cte\right\} \subset \P_{r}(\R^{d})\,.
\end{equation}
If $s\geq s_{\kappa}/2$ and $\expo > 1$ then for each $1\leq p < \expo$ there exists  $C = C(B,\cte, \expo, s, d, \kappa, p) >0$ such that
\begin{equation*}
\forall \pi,\pi' \in \Sfrak_{B, \cte, \expo,s},\ \W_p(\pi,\pi') \leq  C\|\pi-\pi'\|_{\kappa}^{\frac{\expo-p}{p(d+2\expo)}}\,.
\end{equation*}
\end{restatable}

The proof is given in Appendix \ref{proof:theo:regularcase}. With the model set $\Sfrak = \Sfrak_{B, \cte, \expo,s}$, this theorem implies that $(\Sfrak,\W_p)$ is $(\kappa,\delta = \frac{\expo-p}{p(d+2\expo)})$-embeddable for every $1 \leq p < \expo$ as soon as $\kappa$ is a TI, PSD kernel with some regularity, and the distributions in $\Sfrak$ are sufficiently regular with bounded $\expo$-moments. This latter hypothesis is not very limiting in practice since it is also required in order to have finite Wasserstein distances. The Sobolev condition on the densities requires that densities are in $L_2$ and have at least $s \geq s_{\kappa}/2$ (weak)derivatives in $L_2$. In particular this is the case for the classical model sets considered in compressive statistical learning literature such as Gaussian mixtures \citep{gribonval2020statistical}. 

\begin{remark}
Since the distributions in $\Sfrak$ admit a density, the constraints of Theorem \ref{theo:regularitykerneltheo} (mixtures of Diracs) do not apply here and, as such, the kernel is allowed to be smooth. 
\end{remark}
An important family of TI kernels satisfying the hypothesis of Theorem \ref{theo:regularcase} is the Matérn class \citep[Section 4.2.1]{Rasmussen}, with parameter $\nu$, as detailed in Example~\ref{ex:matern}. The limit of a Matèrn kernel when the parameter $\nu \to \infty$ is the RBF kernel, which is too regular: its Fourier transform decays too fast to satisfy the assumption \eqref{eq:regularity_kernel} of Theorem \ref{theo:regularcase}. In the context of compressive learning, translation invariant kernels are most useful if they can be approximated with random Fourier features with good concentration properties (see Section~\ref{sec:compress_section}). An interesting question for future work is thus whether the ‘‘slow decay'' of the Fourier transform needed to apply Theorem \ref{theo:regularcase} appears as a strong constraint in such a context. 

Observe that for fixed $p$ and large $r$ the exponent $\delta = \frac{r-p}{p(d+2r)}$ tends to $\frac{1}{2p}$. Another consequence of Theorem \ref{theo:regularcase} is for distributions that have infinitely many bounded moments. In this case the exponent $\delta$ can be \emph{independent of the dimension}, as shown in the following two examples:

\begin{myexample}[Uniformly bounded moments]
\label{example:unif_bounded_moments}
Consider a kernel $\kappa$ and an exponent $s$ with the same assumptions as in Theorem \ref{theo:regularcase} and a function $m : \R \rightarrow \R_{+}^{*}$ along with the following model set:
\begin{equation}
\Sfrak_{B,m,s}:=  \left\{
\pi \in \P(\R^{d}): \pi = f\dr \xbf,\ \|f\|_{H^s(\R^d)}\leq B \text{ and } \forall \expo > 1, \ \moment_{\expo}[\pi] \leq m(\expo)\right\}\,.
\end{equation}
i.e., the intersection of the model sets $\Sfrak_{B,m(r),r,s}$, $r>1$. For any $p \in [1,+\infty)$ and $0< \delta < \frac{1}{2p}$ we can find a constant\footnote{It suffices to apply Theorem \ref{theo:regularcase} with $\Sfrak_{B, m(\expo), \expo,s}$ where $\expo = \frac{(1+\delta d)p}{1-2 \delta p} > p$ since $\delta = \frac{\expo-p}{p(d+2\expo)} $.} $C = C(B,m(\cdot),\delta,s,d,\kappa,p) > 0$ such that $\forall \pi,\pi' \in \Sfrak_{B,m,s}, \W_p(\pi,\pi') \leq C \|\pi-\pi'\|_{\kappa}^{\delta}$. In other words 
$(\Sfrak_{B,m,s},\W_{p})$ is $(\kappa,\delta)$-embeddable for 
an exponent that is as close as we want to $\delta^{*} = \frac{1}{2p}$.

A notable example where such a model is relevant is in compressive statistical learning, where the model set associated to Gaussian mixtures with bounded parameters fits into this framework \citep{gribonval2020compressive}. 
More generally one can also consider a model set made of sub-Gaussian variables with smooth densities and bounded sub-Gaussiannity parameter $\sigma$.
  In this case $m(\expo) = c \sigma_{\max} \sqrt{\expo}$ for some constant $c> 0$ since, by the sub-Gaussian property, we have $\forall \expo \geq 1, \moment_{\expo}[\pi] \leq c \sigma \sqrt{\expo} \leq c \sigma_{\max} \sqrt{\expo}$ (see \textit{e.g.} \citealt[Section 7.4]{foucart13}).
\end{myexample}

\begin{myexample}[Compactly supported distributions]
With the same assumptions of $\kappa$ and $s$, when all the distributions in $\Sfrak$ are smooth and have the same compact support, they can be shown to belong to $\Sfrak_{B,m,s}$ where the function $m : \R \rightarrow \R_{+}^{*}$ can be chosen as constant. Indeed if $\supp(\pi) \subseteq B(0,\cte)$ for some ball of radius $\cte$ then $\forall \expo >1, \moment_{\expo}[\pi] \leq \cte$. In this case the exponent $\delta = \frac{1}{2p}$ is \emph{exactly} attainable as shown in Appendix \ref{proof:compactly_supported_case}. 
\end{myexample}

\begin{remark}
We recall that, due to the constraints of Proposition \ref{weedprop}, the best possible rate achievable is $\delta = 1/p$ since the model set $\Sfrak_{B, \cte, \expo,s}$ in \eqref{eq:modelsetforsobol} contains a convex combination of two probability distributions whose support are disjoint. Indeed, it is not difficult to construct two measures in the model set $\pi_1 = f_1 \dr \xbf$ and $\pi_2 = f_2 \dr \xbf$ with $\|f_1\|_{H^s(\R^{d})}, \|f_2\|_{H^s(\R^{d})} \leq B$ and such that $\supp(\pi_1) \cap \supp(\pi_2) = \emptyset$. Then for any $t \in [0,1], (1-t) \pi_1+ t\pi_2 \in \Sfrak_{s, B, \cte, \expo}$ since it has density $(1-t)f_1+tf_2$ such that $\|(1-t)f_1+tf_2\|_{H^s(\R^{d})} \leq B$ and $\moment_{\expo}^{\expo}[(1-t) \pi_1+ t\pi_2] = (1-t) \moment_{\expo}^{\expo}[\pi_1] +t\moment_{\expo}^{\expo}[\pi_2]$ by linearity (with respect to the distribution) thus $\moment_{\expo}^{\expo}[(1-t) \pi_1+ t\pi_2] \leq \cte^{\expo}$ which implies $\moment_{\expo}[(1-t) \pi_1+ t\pi_2] \leq M$. It remains open whether exponents $\delta \in (1/2p,1/p)$ are actually achievable on $\Sfrak_{B, \cte, \expo,s}$.
\end{remark}

\subsection{Sufficient Conditions: Non-Regular Distributions}

The case of measures on $\R^{d}$ and that do not admit a density is more delicate to study. We will however prove that, at the price of an arbitrary small additive term $\eta >0$, we have the control \eqref{eq:eqtoprove} under mild assumptions on the model set $\Sfrak$. The core idea is to regularize the probability distributions $\pi,\pi'$ and to obtain bounds between the true Wasserstein and the ‘‘smoothed'' Wasserstein distance which is easier to relate to an MMD. We adopt the following definition:
\begin{restatable}[Regularizer]{definition}{regularizer}
\label{def:regularizer}
We say that a function $\alpha: \R^{d} \rightarrow \R_{+}$ is a \emph{regularizer} if it is a \emph{non-negative, continuous, even} and \emph{bounded} function such that $\int_{\R^{d}} \alpha(\zbf)\dr \zbf=1$ and $\alpha \in L_2(\R^{d})$. We say that the regularizer has $\expo$-finite moments if $\int \|\zbf\|_2^{\expo} \alpha(\zbf) \dr \zbf<+\infty$ for some $\expo\geq 1$.
\end{restatable}
When considering a regularizer $\alpha$ and a probability distribution $\pi \in \P(\R^{d})$ (not necessarily regular) the convolution $\alpha*\pi$ defines a probability density function\footnote{Since $\alpha$ is a regularizer we have $\int \alpha=1$ and consequently $\int (\int \alpha(\xbf-\ybf) \dr \pi(\ybf))\dr \xbf=\int (\int \alpha(\xbf-\ybf) \dr \xbf)\dr \pi(\ybf)=1$ by using Fubini's theorem ($\alpha$ is non-negative) and the fact that the Lebesgue measure is invariant by translation.} on $\R^{d}$ \textit{via} $\alpha*\pi(\xbf)=\int_{\R^{d}} \alpha(\xbf-\ybf) \dr \pi(\ybf)$. In the following we will note $\pi_\alpha$ the  probability distribution associated to the density $\alpha*\pi$. Note that $\pi_\alpha$ is usually regular by imposing that $\alpha$ is (such as when $\alpha$ is the Gaussian density). The interpretation behind $\pi_\alpha$ is the following: if $X \sim \pi$ and $Y_\alpha$ is a random variable independant of $X$ and whose distribution has density $\alpha$ then the random variable $X+Y_\alpha$ has distribution $\pi_\alpha$. The idea of regularizing the measure to derive properties on the Wasserstein distance is not new and was used in various contexts \citep{dedecker2013minimax,nilesweed2020minimax,Goldfeld2020GaussianSmoothedOT,XuanLong}. We have the following lemma which relates the Wasserstein distance $\W_p$ to its regularized counterpart:
\begin{restatable}{lemma}{regwasscompared}
\label{lemma:regwass_compared}
Consider a regularizer $\alpha$ with $p$-finite moments where $p \geq 1$. Then
\begin{equation*}
\forall \pi,\pi' \in \P(\R^{d}), \ \W_p(\pi,\pi') \leq \W_p(\pi_\alpha,\pi'_\alpha)+2 \left(\int 
\|\zbf\|_2^{p}
 \alpha(\zbf) \dr \zbf\right)^{1/p}\,.
\end{equation*}
\end{restatable}
\begin{proof}
Using the triangle inequality we have $\W_p(\pi,\pi')\leq \W_p(\pi,\pi_\alpha)+\W_p(\pi_\alpha,\pi'_\alpha)+\W_p(\pi',\pi'_\alpha)$. Let $X \sim \pi$ and $Y_\alpha$ be a random variable independent of $X$ and whose distribution has density $\alpha$ so that $X+Y_\alpha \sim \pi_\alpha$. By definition of $\W_p$ we have $\W^{p}_p(\pi,\pi_\alpha)=\inf_{\gamma \in \Pi(\pi,\pi_\alpha)} \E_{(Z_1,Z_2)\sim \gamma}[\|Z_1-Z_2\|_2^{p}] $ hence taking $(Z_{1},Z_{2}) = (X,X+Y_{\alpha})$ we obtain $\W^{p}_p(\pi,\pi_\alpha)\leq \E[\|X-(X+Y_\alpha)\|_2^{p}]=\E[\|Y_\alpha\|_2^{p}]$. Consequently $\W^{p}_p(\pi,\pi_\alpha)\leq \int  \|\ybf\|_2^{p} \alpha(\ybf) \dr \ybf$. The same applies for the term $\W_p(\pi',\pi'_\alpha)$. 

\end{proof}

When $\alpha$ is the density of the Gaussian $\mathcal{N}(0,\sigma^{2} \mathbf{I})$ the distance $\W_p(\pi_\alpha,\pi'_\alpha)$ is usually called the Gaussian-smoothed OT and enjoys good properties in terms of sample-complexity and topological properties \citep{Goldfeld2020GaussianSmoothedOT,Nietert2021}. Our formalism is more general as it considers any type of regularizers. The main idea now is to show that, given the regularizer, $\W_p(\pi_\alpha,\pi'_\alpha)$ can be controlled by the MMD associated to a TI kernel. Since $\pi_{\alpha}, \pi'_{\alpha}$ admit a density we will use the same idea as in the Proposition \ref{lemma:bounding_wass} to control $\W_p(\pi_\alpha,\pi'_\alpha)$ by $\|\alpha*\pi-\alpha*\pi'\|_{L_2(\R^{d})}^{\delta}$ for some $\delta \in (0,1)$. To connect with the MMD we will rely on the following result whose proof is given in Appendix \ref{proof:prop:regu_wass}:

\begin{restatable}{lemma}{characterizationmmd}
\label{lemma:characterization_mmd}
Let $\alpha$ be a regularizer and $\kappa_{0}:=\alpha *\alpha$. Then $\kappa_{0} \in L_1(\R^{d})$ is even, bounded, continuous and has non-negative Fourier transform. Consider the kernel $\kappa(\xbf,\ybf):=\kappa_{0}(\xbf-\ybf)$. Then $\kappa$ defines a TI, PSD kernel. Moreover, for $\pi,\pi' \in \P(\R^{d})$,
\begin{equation*}
\|\pi-\pi'\|_{\kappa}=\|\alpha*\pi-\alpha*\pi'\|_{L_2(\R^{d})}\,.
\end{equation*}
\end{restatable} 

Based on these results we have the following upper-bound on $\W_p(\pi_\alpha,\pi'_\alpha)$ using the MMD associated to a TI, PSD kernel (the proof can be found in Appendix \ref{proof:prop:regu_wass}):
\begin{restatable}{proposition}{reguwass}
\label{prop:regu_wass}
Let $\expo>1$. Consider a regularizer $\alpha$ with $\expo$-finite moments and the kernel $\kappa(\xbf,\ybf)=\kappa_{0}(\xbf-\ybf)$ where $\kappa_{0}:=\alpha*\alpha$. It defines a TI, PSD kernel by Lemma \ref{lemma:characterization_mmd}. Moreover, for any $\pi,\pi' \in \P_\expo(\R^{d})$ and $1\leq p <\expo$, $\W_{p}$ defined with the Euclidean norm on $\R^{d}$ satisfies
\begin{equation*}
\W_p(\pi_\alpha,\pi'_\alpha) \leq C_{d,\expo,p} \left(\E_{\xbf \sim \pi_\alpha}[\|\xbf\|_2^{\expo}]+\E_{\ybf \sim \pi_\alpha'}[\|\ybf\|_2^{\expo}]\right)^{\frac{2p+d}{(d+2\expo)p}}\|\pi-\pi'\|_{\kappa}^{\frac{2(\expo-p)}{(d+2\expo)p}}\,,
\end{equation*}
for some constant $C_{d,\expo,p}> 0$. 
\end{restatable}

As a corollary of Proposition \ref{prop:regu_wass} and Lemma \ref{lemma:regwass_compared} we are now able to prove the main theorem of this section (the proof is in Appendix \ref{proof:prop:regu_wass}):

\begin{restatable}{theorem}{maintheononcompact}
\label{theo:maintheo_noncompact}
Let $\expo>1$. Consider a regularizer $\alpha$ with $\expo$-bounded moments. Consider the kernel $\kappa(\xbf,\ybf)=\kappa_{0}(\xbf-\ybf)$ where $\kappa_{0}:=\alpha*\alpha$. It defines a TI, PSD kernel by Lemma \ref{lemma:characterization_mmd}. We consider the model set
\begin{equation*}
\Sfrak_{\cte} := \{\pi \in \P(\R^{d}): \moment_{\expo}[\pi]\leq \cte\} \subset \P_{\expo}(\R^{d})\,.
\end{equation*}
Then for any $1\leq p <\expo$ there exists a constant $C'=C'_{d,\expo,p}>0$ such that
\begin{equation*}
\forall \pi,\pi' \in \Sfrak, \W_p(\pi,\pi') \leq C'\left(\cte^{\expo}+ \int \|\zbf\|_2^{\expo}\alpha(\zbf)\dr \zbf\right)^{\frac{2p+d}{p(d+2\expo)}}  \|\pi-\pi'\|^{\frac{2(\expo-p)}{(d+2\expo)p}}_{\kappa}+2 \left(\int \|\zbf\|_2^{p} \alpha(\zbf) \dr \zbf\right)^{1/p}\,.
\end{equation*}
\end{restatable}
This theorem has multiple implications. First it shows that, for a wide range of TI, PSD kernels, and under mild assumptions, $(\Sfrak,\W_p)$ is $(\kappa,\delta=\frac{2(\expo-p)}{p(d+2\expo)})$-embeddable with error $\eta > 0$. Note that the exponent $\delta$ is twice the exponent found in Section \ref{sec:thecaseofregular} for regular distributions, which is due to the fact that we directly regularize the distributions using the kernel associated to the MMD. Consequently, it leads to a slightly better better exponent (closer to $1$) than the one of the regular case, but at a price of an additive error term. We will also see in Example \ref{sec:example_rbf} how this error term $\eta > 0$ can be controlled. We emphasize that few assumptions on $\Sfrak$ are required: the distributions in the model set must have uniformly bounded $\expo$-moment, \ie\ $\sup_{\pi \in \Sfrak} \E_{\xbf \sim \pi}[\|\xbf\|_2^{\expo}]<+\infty$. This assumption is verified when, for example, $\Sfrak$ is the space of Gaussian mixtures whose parameters are in a compact subspace as considered in compressive statistical learning \citep{gribonval2020statistical}. Interestingly, if $\expo$ is big compared to $d,p$ then we have $\delta \approx \frac{1}{p}$.

\begin{myexample}[RBF kernel]
\label{sec:example_rbf}
As an example of use of Theorem \ref{theo:maintheo_noncompact} consider the Gaussian density function $\varphi(\xbf):=(2\pi)^{-d/2}\exp(-\|\xbf\|_2^{2}/2)$. Define for $\sigma >0$ the regularizer $\alpha(\xbf):=\sigma^{-d}\varphi(\frac{\xbf}{\sigma})$. The function $\alpha$ is continuous, even, bounded, all $\expo$-moments are finite, $\int_{\R^{d}} \alpha=1$. The associated kernel is then defined by $\widehat{\kappa_0}(\w)=(\hat{\varphi}(\sigma\w))^{2}=(e^{-\frac{1}{2}\sigma^{2}\|\w\|^{2}_2})^{2}=e^{-\sigma^{2}\|\w\|^{2}_2}$, hence $\kappa(\xbf,\ybf)=\pi^{d/2} \sigma^{-d} \exp(-\frac{\|\xbf-\ybf\|_2^{2}}{4 \sigma^{2}})$. Consider the case $p=1$ and $\expo>1$ of Theorem \ref{theo:maintheo_noncompact}. The error term $2\int \|\zbf\|_2 \alpha(\zbf) \dr \zbf= 2\sigma \int \|\zbf\|_2 \varphi(\zbf) \dr \zbf$ can be controlled as $$2\sigma \int \|\xbf\|_2(2\pi)^{-d/2}\exp(-\|\xbf\|_2^{2}/2) \dr \xbf\leq 2\sigma (\int  \|\xbf\|^{2}_2(2\pi)^{-d/2}\exp(-\|\xbf\|_2^{2}/2) \dr \xbf)^{1/2}$$ by Jensen since $\xbf \rightarrow (2\pi)^{-d/2}\exp(-\|\xbf\|_2^{2}/2)$ is a probability density function. Thus, we can bound the error therm by $2\sigma (\E_{\xbf \sim \mathcal{N}(0,\mathbf{I})}[\|\xbf\|_2^{2}])^{1/2}=2\sigma \sqrt{d}$. Moreover, $\int \|\zbf\|_2^{\expo} \alpha(\zbf) \dr \zbf=\sigma^{\expo}\int \|\zbf\|_2^{\expo} \varphi(\zbf) \dr \zbf= \E_{\xbf \sim \mathcal{N}(0,\mathbf{I})}[\|\xbf\|_2^{\expo}]=2^{\expo/2} \frac{\Gamma(\frac{\expo+d}{2})}{\Gamma(\frac{\expo}{2})}$ (it is the $\expo$-th moment of a $\chi_2$ distribution). Then, using Theorem \ref{theo:maintheo_noncompact} we have
\begin{equation*}
\forall \pi,\pi' \in \Sfrak, \ \W_1(\pi,\pi') \leq C' \left(\cte^{\expo}+2^{\expo/2} \sigma^{\expo}\frac{\Gamma(\frac{\expo+d}{2})}{\Gamma(\frac{\expo}{2})} \right)^{\frac{d+2}{d+2\expo}}\|\pi-\pi'\|^{\frac{2(\expo-1)}{d+2\expo}}_{\kappa}+2\sigma \sqrt{d}\,.
\end{equation*}
Interestingly enough, the error term behaves as $O(\sigma)$ and can me made as small as possible at a price of a ‘‘sharper'' kernel (the bound is true for any $\sigma >0$). Implications of this result wil be discussed in the context of CSL in Section \ref{sec:compress_section}.
\end{myexample}

\begin{remark}
The condition $\kappa_{0} =\alpha * \alpha$ in Theorem \ref{theo:maintheo_noncompact} can be met in two ways. First, as done in Example \ref{sec:example_rbf}, fixing a regularizer $\alpha$ with $\expo$-bounded moments gives a TI, PSD kernel so that Theorem \ref{theo:maintheo_noncompact} holds. This can be achieved for example by considering a PSD function $\alpha \in L_1(\R^{d})$ with a sufficient number of bounded moments and that is even, continuous and positive (continuous, integrable and PSD functions are bounded \citealp{Wendland}). A simple normalization $\alpha \leftarrow \alpha/ \int \alpha$ will then produce a suitable $\alpha$. The second way is to fix the kernel $\kappa(\xbf,\ybf)=\kappa_{0}(\xbf-\ybf)$ and to check that it can be decomposed as $\kappa_{0}=\alpha * \alpha$ with $\alpha$ a regularizer with $\expo$-bounded moments and $\widehat{\alpha} \geq 0$. This problem is related to the one of finding a so-called \emph{convolution root}, or \emph{Boas–Kac root} of a positive definite function which can be shown to exist under certain assumptions on the function \citep{convolution_root,Akopyan2017BoasKacRO,Boas}.
\end{remark}

\subsection{Conclusion and Related Works}
\label{related_work_mmd_wass}

\begin{figure}[t!]
\begin{center}
\includegraphics[width=1\linewidth]{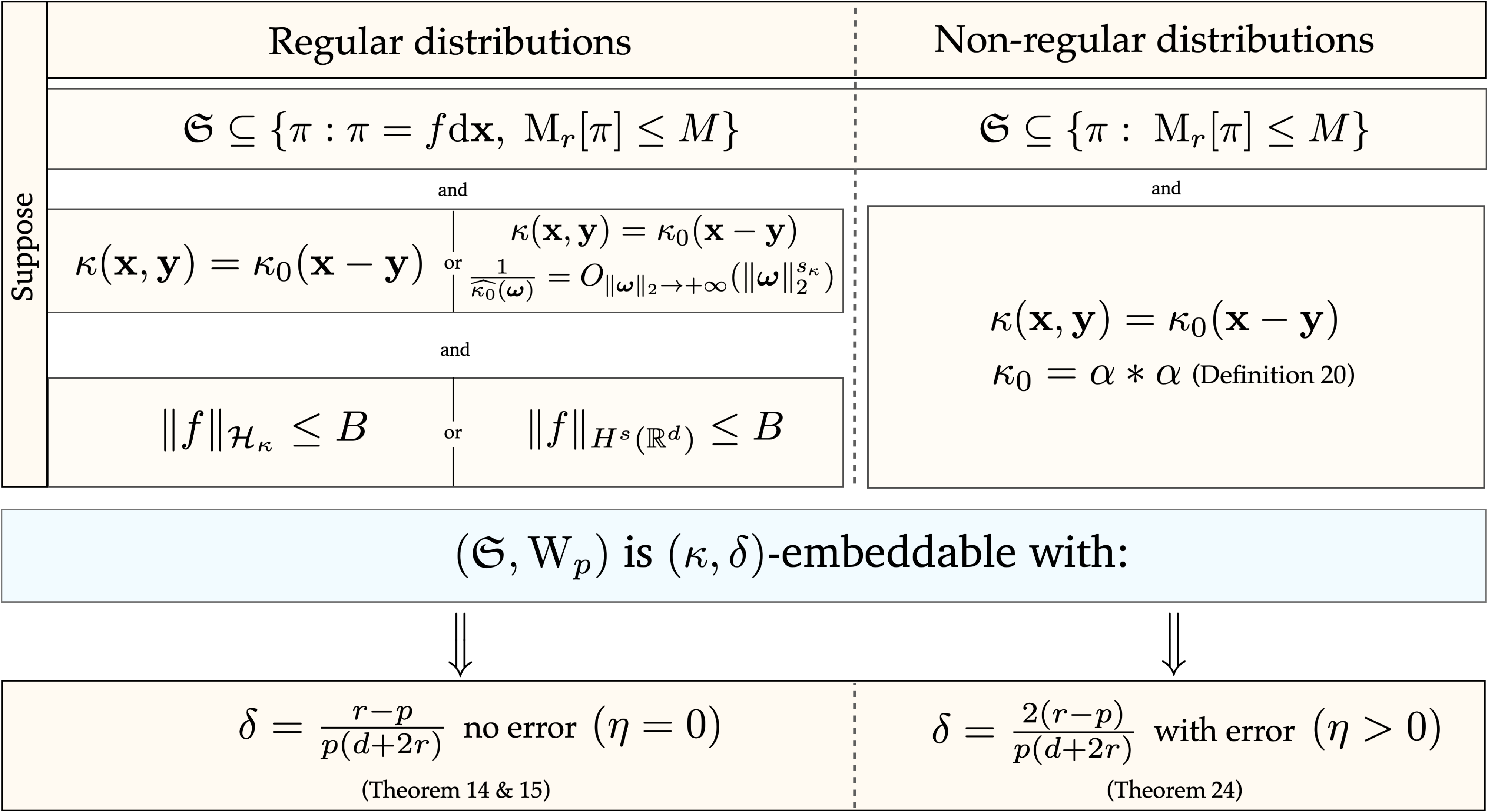}
\end{center}
\caption{\label{fig:summary} Summary of the different results of Section \ref{sec:wass_mmd}. The mention ‘‘with  error'' means that the relation holds when adding an error $\eta > 0$ that does not depends on $\Sfrak$. $\pi = f \dr \xbf$ means that the measure has density $f$ with respect to the Lebesgue measure.}
\end{figure}

We established in this section various controls of the form $\W_p \lesssim \operatorname{MMD}_{\kappa}^{\delta}$ that depend on $\delta \in (0,1]$, the properties of the model set and the kernel $\kappa$. All these results are summarized in Figure \ref{fig:summary}. Some other connections between MMDs and  Wasserstein distances have been explored in the literature.  The most simple one is when the metric $D$ used to define the Wasserstein distance is the metric in the RKHS corresponding the the kernel $\kappa$, \ie\ $D(\xbf,\ybf)=\|\kappa(\cdot,\xbf)-\kappa(\cdot,\ybf)\|_{\Hcal_\kappa}$. In this case it is known that we can control the Wasserstein distance $\W_1$ by $\sqrt{\operatorname{MMD}^{2}_{\kappa}+K}$ when $\kappa$ is bounded by $K$ \citep{Sriperumbudur}. 

\subsubsection{Relaxing the Translation-Invariance Property} Other interesting connections are based on the Gaussian-smoothed Wasserstein distance \citep{Goldfeld2020GaussianSmoothedOT} where authors consider $\alpha$ the probability density function of the Gaussian $\mathcal{N}(0,\sigma^{2} \mathbf{I})$ and the Wasserstein distance between the regularized distributions $\pi_{\alpha}= \alpha * \pi$. In \citet{zhang2021convergence} authors show that we can control the Gaussian-smoothed Wasserstein distance with the MMD, by considering a PSD kernel that is \emph{not} translation-invariant and \emph{not} bounded but defined as $\kappa(\xbf, \ybf)=\exp \left(-\frac{\|\xbf-\ybf\|_2^{2}}{4 \sigma^{2}}\right) I_{f}\left(\frac{\|\xbf+\ybf\|_2}{\sqrt{2} \sigma}\right)$  where $I_{f}$ is a function parametrized by some probability density function $f$ such as generalized beta-prime distributions. More precisely they prove 
\begin{equation*}
\forall \pi,\pi' \in \Sfrak_{\kappa}, \W_p(\pi_{\alpha},\pi'_{\alpha})\leq 2 \sigma \|\pi-\pi'\|^{1/p}_{\kappa}\,,
\end{equation*}
where $\Sfrak_{\kappa}:=\{\pi \in \P(\R^{d}): \int \sqrt{\kappa(\xbf,\xbf)} \dr \pi(\xbf) <+\infty\}$ \citep[Theorem 2]{zhang2021convergence}. With the same type of arguments as those presented in Lemma \ref{lemma:regwass_compared} we can prove that for any $\pi,\pi' \in \Sfrak_{\kappa}$ we have $\W_p(\pi,\pi') \leq  2 \sigma  \|\pi-\pi'\|_{\kappa}^{1/p}+\eta$ where $\eta = 2 \left(\int \|\zbf\|_2^{p} \alpha(\zbf) \dr \zbf\right)^{1/p}$ and $\W_p$ is computed with $\|\cdot\|_2$. As a corollary, for this kernel that is not TI we can use the result of \citet{zhang2021convergence} to prove that $(\Sfrak_{\kappa},\W_p)$ is $(\kappa,\frac{1}{p})$-embeddable with error $\eta=2 \left(\int \|\zbf\|_2^{p} \alpha(\zbf) \dr \zbf\right)^{1/p}$ that will behave as $O(\sigma)$ as shown in Example \ref{sec:example_rbf}. We can mention another line of works which draws connections between the Wasserstein distance and some specific dual Sobolev norms which can be related to the MMD. In \citet{nietert2021smooth} authors control the Wasserstein distance with an MMD whose kernel, which is not TI, is defined by $\kappa(\xbf,\ybf)=-\sigma^{2} \operatorname{Ein}(-\langle \xbf,\ybf\rangle/\sigma^{2})$ where $\operatorname{Ein}(z)=\int_{0}^{z}\frac{(1-e^{-t})}{t} \dr t$. Despite the fact that our two approaches are related our work differs from the Gaussian-smoothed OT in the sense that we do not want to estimate precisely the smoothed Wasserstein distance $\W_p(\pi_{\alpha},\pi'_{\alpha})$ by controlling it with an MMD based on a \emph{specific} kernel but instead to control $\W_p(\pi,\pi')$ by kernel norms for \emph{many} types of TI kernels.

\subsubsection{Relaxing the PSD Assumption on the Kernel}
Beyond PSD kernels other types of kernels can be used to define interesting divergences between probability distributions that can be linked with the Wasserstein distance. These divergences are not \emph{stricly speaking} MMD norms as defined in \eqref{eq:mmdnormdef} with PSD kernels but share similar topological properties. For example, by considering the \emph{conditionally} PSD\footnote{A conditionally PSD kernel on $\Xcal$ satisfies $\sum_{i,j=1}^{n} c_i c_j \kappa(\xbf_i,\xbf_j)\geq 0$ for any $\xbf_1,\cdots,\xbf_n \in \Xcal$ and $c_1,\cdots c_n \in \R$ such that $\sum_{i=1}^{n} c_i=0$ \citep{berg84harmonic}} kernel $\kappa(\xbf,\ybf)=-\|\xbf-\ybf\|^{\beta}_2$ for $\beta \in (0,2]$,  and $\pi, \pi' \in \P(\R^{d})$, the integral in \eqref{eq:mmdnormdef} is non-negative for $\mu = \pi-\pi'$ so that the term $\|\pi-\pi'\|_{\kappa}$ is well defined \citep[Example 15]{10.1214/13-AOS1140}. It is called the energy, or \emph{Cramér}, distance \citep{enrgy_distance,Szekely,10.1214/13-AOS1140} and it connects with OT distances in the sense that the Sinkhorn divergence (regularized OT) was shown to interpolate between this MMD and the Wasserstein distance \citep{pmlr-v89-feydy19a}. Another notable example is when one considers the so called $d$-dimensional \emph{Coulomb} kernel defined by $\kappa(\xbf,\ybf)=\kappa_{0}(\xbf-\ybf)$ where
\begin{equation*}
\kappa_{0}(\xbf):=\left\{\begin{array}{ll}
-\log \|\xbf\|_2 & \text { if } d=2 \\
\|\xbf\|_2^{2-d} & \text { if } d \geq 3
\end{array}\right.
\end{equation*}
In this case, for compactly supported $\pi,\pi' \in \P(\R^{d})$ with $\int \int \kappa(\xbf,\xbf') \dr \pi(\xbf) \dr \pi(\xbf') <+\infty $ and $\int \int \kappa(\ybf,\ybf') \dr \pi'(\ybf) \dr \pi'(\ybf') <+\infty$, the quantity $\|\pi-\pi'\|_{\kappa}$ is well defined, finite, and vanishes if and only if $\pi=\pi'$ \citep{CHAFA2016ConcentrationFC,saff2013logarithmic}. Consequently it defines a valid MMD that remarkably controls the $\W_1$ distance associated to an arbitrary norm in $\R^{d}$, as described in \citet{CHAFA2016ConcentrationFC}. More precisely consider, for $\Omega \subseteq \R^{d}$ \emph{compact}, the model set
$$ \Sfrak: = \{ \pi \in \P(\R^{d}): \supp(\pi) \subseteq \Omega, \int \int \kappa(\xbf,\xbf') \dr \pi(\xbf) \dr \pi(\xbf') <+\infty\} \,.$$
Then \citet[Theorem 1]{CHAFA2016ConcentrationFC} proves that there exists $C=C(\Omega)>0$ such that
\begin{equation*}
\forall \pi,\pi' \in \Sfrak, \W_1(\pi,\pi') \leq C \|\pi-\pi'\|_{\kappa}\,.
\end{equation*}
In particular, with the above $\Sfrak$, $(\Sfrak,\W_1)$ 
is $(\kappa,\delta=1)$-embeddable with no error. It is remarkable in the sense that few assumptions on the model set are required (the distributions can be even discrete). An important remark is that the kernel is TI \emph{but not PSD} and, consequently, this result is not in contradiction with Theorem \ref{theo:regularitykerneltheo}. Finally, other connections between $\W_p$ and the Cramér distance regarding asymptotic convergence in law can be found in \citep{modeste2022characterization}.

\section{Statistical Learning and Wasserstein Regularity}
\label{sec:wasserstein_learnability}

The bounds obtained previously allow us to control the Wasserstein distance by an MMD under certain conditions. These results will be at the heart of the theoretical guarantees of compressive learning (Section \ref{sec:compress_section}). These guarantees require, in addition, to control metrics related to the learning task (see the reasoning described in Figure \ref{fig:fantasticfig}). In this section we recall the statistical learning framework and introduce more formally these task metrics (referred as $\operatorname{Task Metric}$ in the introduction). We then show how to control them by a Wasserstein distance for various learning tasks.

\subsection{Statistical Learning \& Task Metrics \label{section:gentle}}

Statistical learning is a formalism that offers many tools to study the guarantees of learning algorithms. The problem is usually expressed as follows: given a collection of data $(\xbf_i)_{i \in \integ{n}}$, where $\xbf_i$ is a \emph{sample} in the data space $\Xcal$, how do we select a hypothesis $h \in \Hcal$ (where $\Hcal$ is called the \emph{hypothesis space}) that best performs the task at hand ? The ideal hypothesis minimizes a certain \emph{risk} which provides a performance measure and is derived from a certain loss function $\ell: \Xcal \times \Hcal \rightarrow \R$.

 For example, in the context of linear regression the loss is defined as $\ell(\xbf=(\zbf,y),h=\thetab)=(y-\thetab^{\top}\zbf)^{2}$ where $y \in \R$ is the value to predict, $h=\thetab \in \R^{d}$ is the parameters to choose and $\zbf\in \R^{d}$ is the vector of input features. Given a data-generating distribution $\pi \in \P(\Xcal)$, \ie\ the law under which our samples are produced, most of the machine learning algorithms attempt to minimize the so-called \emph{expected risk} (or generalization error):
\begin{equation*}
\Rcal(\pi,h)=\E_{\xbf \sim \pi}[\ell(\xbf,h)]\,.
\end{equation*}
This quantity reflects how effective is $h$ on average on the data-generating distribution. The optimal hypothesis $h^{*} \in \Hcal$, known as the \emph{Bayes prediction function} \citep{steinwart2008support}, is such that $h^{*}\in \arg\min_{h \in \Hcal} \Rcal(\pi,h)$.  The major difficulty is that the generating distribution $\pi$ is unknown and that we only have access to finitely many samples $(\xbf_i)_{i \in \integ{n}}$. Methods such as \emph{empirical risk minimization} (ERM) produce an estimated hypothesis $\hat{h}$ from the training dataset by minimizing the risk $\Rcal(\pi_n,\cdot)$ associated to the empirical probability distribution $\pi_n=\frac{1}{n} \sum_{i=1}^{n} \delta_{\xbf_i}$. One aims at guaranteeing, with high probability, the following bound on the \emph{excess risk}:
\begin{equation}
\label{eq:control_excess}
\Rcal(\pi,\hat{h})-\Rcal(\pi,h^{*}) \leq \eta_n\,,
\end{equation}
where $\eta_n$ decays as $1/\sqrt{n}$ or better. This simply reflects that we may expect a hypothesis that is close to the best one as the training set grows, \ie\ when we have access to enough data. To obtain a control of the excess risk by $\eta_n$ one often relies on the following bound\footnote{This can be proved by noting that $\Rcal(\pi,\hat{h})-\Rcal(\pi,h^{*})= \left\{\Rcal(\pi,\hat{h})-\Rcal(\pi_n,\hat{h})\right\}+\left\{\Rcal(\pi_n,\hat{h})-\Rcal(\pi_n,h^{*})\right\}+\left\{\Rcal(\pi_n,h^{*})-\Rcal(\pi,h^{*})\right\}$. Since $\Rcal(\pi_n,h^{*})-\Rcal(\pi_n,\hat{h})\leq 0$ by definition of $\hat{h}$ we have $\Rcal(\pi,\hat{h})-\Rcal(\pi,h^{*})\leq 2\sup_{h \in \Hcal}|\Rcal(\pi,h)-\Rcal(\pi_n,h)|$.}:
\begin{equation*}
\Rcal(\pi,\hat{h})-\Rcal(\pi,h^{*}) \leq 2\sup_{h \in \Hcal} | \Rcal(\pi,h)-\Rcal(\pi_n,h)|\,.
\end{equation*}
Consequently, being able to control the right term in the previous equation is a central problem in statistical learning and for example arguments involving Rademacher complexities can lead to the desired bound in \eqref{eq:control_excess} (see \citealp{bendavid}). The term $\sup_{h \in \Hcal} | \Rcal(\pi,h)-\Rcal(\pi_n,h)|$, that was reffered as $\operatorname{TaskMetric}(\pi,\pi')$ in the introduction, defines a central quantity for our analysis and we introduce the following notation for $\pi,\pi' \in \P(\Xcal)$:
\begin{equation}
\label{eq:lnorm}
\|\pi-\pi'\|_{\mathcal{L}(\Hcal)}:=\sup_{h\in \Hcal}|\Rcal(\pi,h)-\Rcal(\pi',h)|\,. 
\end{equation}
The quantity $\|\cdot\|_{\mathcal{L}(\Hcal)}$ defines a semi-norm on the space of finite signed measures $\Mcal(\Xcal)$ and an integral probability metric \eqref{eq:ipmdefinition} with $\Gcal=\mathcal{L}(\Hcal):=\{ \xbf \rightarrow \ell(\xbf,h); h \in \Hcal\}$. It is important to note that this semi-norm is \emph{task-specific} \ie\ that it depends on the learning task \textit{via} the family $\mathcal{L}(\Hcal)$. In the rest of the paper we will denote, as a language shortcut, $\mathcal{L}(\Hcal)$ as ‘‘the learning task''. As just described, when $\|\pi-\pi_n\|_{\mathcal{L}(\Hcal)}\leq \eta_n$ one can control the excess risk as in \eqref{eq:control_excess}. Consequently, controlling $\|\cdot\|_{\mathcal{L}(\Hcal)}$ with other metrics that are more easily computable is of certain interest. When the loss function is non-negative, $\ell:\Xcal \times \Hcal \rightarrow \R_{+}$, we introduce for $p\geq 1$ the semi-norm
\begin{equation}
\label{eq:lnorm2}
\|\pi-\pi'\|_{\mathcal{L}(\Hcal),p}:=\sup_{h\in \Hcal}|\Rcal^{1/p}(\pi,h)-\Rcal^{1/p}(\pi',h)|\,.
\end{equation}
A control of this semi-norm implies a slighlty different control of the excess risk as $\|\pi-\pi_n\|_{\mathcal{L}(\Hcal),p}\leq \eta_n$ implies that $\Rcal(\pi,\hat{h})^{1/p}-\Rcal(\pi,h^{*})^{1/p} \leq \eta_n$. In the following we often write $\|\pi-\pi_n\|_{\mathcal{L}(\Hcal),p}$ without specifying that the loss function is non-negative and that $p\geq 1$ (this will be implicitly assumed).
\begin{remark}
Controlling the quantity $\|\pi-\pi_n\|_{\mathcal{L}(\Hcal)}$ sometimes leads to pessimistic bounds on the excess risk. A sharper bound can be produced  by considering the following semi-norm $\|\pi-\pi'\|_{\D\mathcal{L}(\Hcal)}:=\sup_{h,h_0\in \Hcal} \left[\{\Rcal(\pi,h)-\Rcal(\pi,h_0)\}-\{\Rcal(\pi',h)-\Rcal(\pi',h_0)\}\right]$ which is related to $\|\pi-\pi'\|_{\mathcal{L}(\Hcal)}$ via the inequality $\|\pi-\pi'\|_{\D\mathcal{L}(\Hcal)}\leq 2\|\pi-\pi'\|_{\mathcal{L}(\Hcal)}$ \citep{gribonval2020compressive}. However in this work we focus on the quantities defined in \eqref{eq:lnorm} and \eqref{eq:lnorm2} and leave the analysis of $\|\cdot\|_{\D\mathcal{L}(\Hcal)}$ for further works.
\end{remark}

\subsection{Wasserstein Regularity}
The main question investigated in this section, which will find applications to compressive statistical learning in Section~\ref{sec:compress_section}, is to understand
when the task-specific norm $\|\pi-\pi'\|_{\mathcal{L}(\Hcal),p}$ can be bounded by the Wasserstein distance between $\pi$ and $\pi'$. We formalize this in the following definition:
\begin{restatable}[Wasserstein regularity]{definition}{wasslearn}
\label{def:wass_learn}
Given $p\in [1,+\infty)$, we say that a task $\mathcal{L}(\Hcal)$ is $p$-Wasserstein regular if there exists $C>0$, such that
\begin{equation*}
\label{eq:wass_learn}
\forall \pi,\pi \in \P_p(\Xcal), \ \|\pi-\pi'\|_{\mathcal{L}(\Hcal),p}=\sup_{h\in \Hcal}|\Rcal^{1/p}(\pi,h)-\Rcal^{1/p}(\pi',h)| \leq C \W_p(\pi,\pi')\,.
\end{equation*}
\end{restatable}

\bgroup
\def\arraystretch{1.5}
\newcolumntype{Y}{>{\centering\arraybackslash}X}

\begin{table}[t]
\centering
\begin{tabularx}{\textwidth}{|Y|Y|}
\hline
\multicolumn{2}{|c|}{When do we have $\forall \pi,\pi' \in \P_p(\Xcal), \|\pi-\pi'\|_{\mathcal{L}(\Hcal),p}\lesssim \W_p(\pi,\pi')$ for some $p \geq 1$ and task $\mathcal{L}(\Hcal)$?} \\ \hline
   \textbf{Condition on the task}   &  \textbf{Examples}     \\ \hline \hline
   \emph{Compression type-tasks.} Loss: $\ell(\xbf,h)=D(\xbf,P_{h}(\xbf))^{p}$, $P_{h}$ projection function   &  PCA, K-means, K-medians, NMF, dictionary learning (Section \ref{sec:compression_task})   \\ \hline
    \emph{Regression tasks.} Hypothesis: $h$ Lipschitz, loss: $\ell(\xbf=(\zbf,\ybf),h)=\|\ybf-h(\zbf)\|^p$  &  Linear regression, regression using MLP with bounded parameters (Section \ref{sec:regression_tasks})     \\ \hline
    \emph{Binary classification.} Hypothesis: $h$ Lipschitz, loss: convex surrogate $\ell(\xbf=(\zbf,y),h)=\varphi^{p}(y h(\zbf))$  &  MLP classifier with bounded parameters + Lipschitz ouput layer (Section \ref{sec:binary_classif}) \\ \hline
 \end{tabularx}

\caption{Summary of the differents results of Section \ref{sec:wasserstein_learnability}. \label{tab:summary_wass_task}}
\end{table}

At first sight the Wasserstein regularity seems a bit unexpected since the Wasserstein distance does not take into account the underlying learning task $\mathcal{L}(\Hcal)$. However we will show below that this property is quite natural for several learning tasks. We provide a summary of the different results of this section in Table \ref{tab:summary_wass_task}.

\begin{remark}
When the task is Wasserstein regular, we can show that the excess-risk is always bounded by a Wasserstein distance, \ie\ if $\pi \in \P_p(\Xcal)$ is any data generating distribution, and $\pi_n$ the empirical distribution, then
\begin{equation*}
\Rcal^{1/p}(\pi,\hat{h})-\Rcal^{1/p}(\pi,h^{*})\leq 2C \W_p(\pi,\pi_n)\,,
\end{equation*}
where $h^{*} \in \arg\min_{h \in \Hcal} \Rcal(\pi,h)$ is an optimal hypothesis and $\hat{h} \in \arg\min_{h \in \Hcal} \Rcal(\pi_n,h)$ the hypothesis found by empirical risk minimization. Therefore, the smaller the Wasserstein distance between $\pi_n$ and $\pi$, the better $\hat{h}$ is. 
\end{remark}
We start by showing that many unsupervised tasks, called \emph{compression-type tasks}, are Wasserstein regular. Then we focus on supervised tasks and demonstrate, under certain Lipschitz assumptions on the hypothesis class $\Hcal$, that these tasks are also Wasserstein regular. Unless stated otherwise, until the end of Section~\ref{sec:wasserstein_learnability}, Wasserstein distances are defined with respect to the metric $D$ associated to the ambient metric space $(\Xcal,D)$.

\subsection{Compression-type Tasks are Wasserstein Regular \label{sec:compression_task}}

The most straightforward case of Wasserstein regularity is when the risk \emph{itself} can be rewritten as a Wasserstein distance. Interestingly, a wide range of unsupervised learning tasks can be recast in this setting. For example, problems such as K-means or PCA can be shown to be performing exactly the task of estimating the data-generating distribution $\pi$ in the sense of a Wasserstein distance \citep{NIPS2012_c54e7837}. Such problems will be very connected with \emph{compression-type tasks} as defined below :
\begin{definition}[\citealp{gribonval2020compressive}]
\label{def:compresstypetask}
Consider a metric space $(\Xcal,D)$ and a hypothesis space $\Hcal$. A task $\mathcal{L}(\Hcal)$ is called a \emph{compression-type} task if the loss can be written as $\ell(\xbf,h)=D(\xbf,P_h(\xbf))^{p}$ where $p\geq 1$ and $P_h:\Xcal \rightarrow \Xcal$ is a \emph{measurable} projection function that satisfies $P_h \circ P_h=P_h$ and $D(\xbf,P_h(\xbf))\leq D(\xbf,P_h(\xbf'))$ for all $\xbf,\xbf' \in \Xcal$.
\end{definition}
Notable examples of such tasks are K-means and PCA. In the former, $\ell$ is defined by $\ell(\xbf,h=(\mathbf{c}_1,\cdots, \mathbf{c}_K))=\min_{i \in \integ{k}}\|\xbf-\mathbf{c}_i\|_2^{2} = \|\xbf-P_{h}(\xbf)\|_2^2$ where $P_h(\xbf)$ is the projection of $\xbf$ on its closest centroid. In the latter, $P_h(\xbf)$ is the projection of $\xbf$ on the linear subspace spanned by $h$.
These two problems are actually related to a wider class of problems, namely $k$-dimensional coding schemes which are particular types of compression-type tasks. As described in \citet{Pontil}, one encounters these problems when $\Xcal$ is a Hilbert space (with some norm $\|\cdot\|$) and when the loss can be written as $\ell(\xbf,h)=\min_{\ybf \in Y}\|\xbf-h \ybf\|^{2}$ for $Y \subseteq \R^{k}$ a prescribed set of \emph{codes} (or \emph{codebook}) and $h: \R^{k} \rightarrow \Xcal$ is a linear map. In particular, non-negative matrix factorization (NMF) \citep{Lee,Udell} and dictionary learning (also known as \emph{sparse coding} \citealp{Honglak,MairalJulien2009,10.1145/1553374.1553463}) are other well known unsupervised learning methods which correspond to projection-type tasks.
As described in \citet{NIPS2012_c54e7837} there are interesting connections between these problems and the Wasserstein distance. More precisely, we have the following lemma (see a proof in Appendix \ref{proof:lemma:rosaco2} adapted to our notational context):
\begin{restatable}[\citealp{NIPS2012_c54e7837}]{lemma}{rosaco}
\label{lemma:rosaco2}
Let $S\subseteq \Xcal$, $p\in [1,+\infty)$ and $\pi \in \P_p(\Xcal)$. Consider $P_S:\Xcal \rightarrow S$, measurable, such that $D(\xbf,P_S(\xbf))\leq D(\xbf,\ybf)$ for all $\xbf \in \Xcal$ and $\ybf \in S$. Then 
\begin{equation*}
\E_{\xbf \sim \pi}[D(\xbf,P_S(\xbf))^{p}]=\W_p^{p}(\pi,P_S\#\pi)\,.
\end{equation*}
Moreover for any $\nu \in \P_p(\Xcal)$ such that $\supp(\nu) \subseteq S$ we have $\W_p(\pi,P_S\#\pi)\leq \W_p(\pi,\nu).$
\end{restatable}
We recall that $P_{S}\# \pi$ is the probability measure defined by $P_{S}\# \pi(A) := \pi(P_{S}^{-1}(A))$ for every measurable set $A$. Based on this lemma we now prove that compression-type tasks are Wasserstein regular, \ie\ that the task-specific norm $\|\cdot\|_{\mathcal{L}(\Hcal),p}$ can be bounded by a Wasserstein distance.
\begin{restatable}[Compression-type tasks are Wasserstein regular]{proposition}{compressiontypetasks}
\label{prop:compressiontypetasks}
Consider a metric space $(\Xcal,D)$, a hypothesis space $\Hcal$, $p\in [1,+\infty[$, and a compression-type task $\mathcal{L}(\Hcal)$ as in Definition \ref{def:compresstypetask}. Then 
\begin{equation*}
\begin{split}
&\forall h \in \Hcal, \pi\in \P_p(\Xcal), \ \Rcal(\pi,h) = \W_{p}^{p}(\pi,P_{h}\#\pi) \text{ and }\\
&\forall \pi,\pi' \in \P_p(\Xcal), \|\pi-\pi'\|_{\mathcal{L}(\Hcal),p}\leq \W_p(\pi,\pi')\,.
\end{split}
\end{equation*}
\end{restatable}
\begin{proof}
Let $h \in \Hcal$ and $P_h$ be the projection function. We denote $S=\{P_h(\xbf); \xbf \in \Xcal\}$ the image of $P_h$. Using Lemma \ref{lemma:rosaco2} we have, for $\pi \in \P_p(\Xcal),$
\begin{equation*} 
\Rcal(\pi,h)=\E_{\xbf\sim \pi}[\ell(\xbf,h)]=\E_{\xbf\sim \pi}[D(\xbf,P_h(\xbf))^{p}]=\W_p^{p}(\pi,P_h\#\pi)\,.
\end{equation*} 
Hence, for $\pi,\pi' \in \P_{p}(\Xcal)$ and $h \in \Hcal,$ 
\begin{equation*}
\begin{split}
\Rcal(\pi,h)^{1/p}-\Rcal(\pi',h)^{1/p}&=\W_p(\pi,P_h\#\pi)-\W_p(\pi',P_{h}\#\pi') \leq \W_p(\pi,P_{h}\#\pi') -\W_p(\pi',P_{h}\#\pi') \\ 
&\leq \W_p(\pi,\pi')\,,
\end{split}
\end{equation*}
where we used $\W_p(\pi,P_{h}\#\pi)\leq \W_p(\pi,\nu)$ if $\supp(\nu) \subseteq S$ (Lemma \ref{lemma:rosaco2}) and applied it to $\nu=P_{h}\#\pi'$ (since $\supp(P_{h}\#\pi') \subseteq S$ by definition of $S$). The last inequality is due the the triangle inequality. By symmetry $|\Rcal(\pi,h)^{1/p}-\Rcal(\pi',h)^{1/p}|\leq \W_p(\pi,\pi')$. Taking the supremum over $h \in \Hcal$ concludes. 
\end{proof}

\begin{remark}
As described in Proposition \ref{prop:compressiontypetasks}, compression-type tasks can be interpreted as finding a ‘‘simple'' distribution $\pi_h = P_{h} \# \pi$ that bests describe the data distribution $\pi$ in the sense of the Wasserstein distance. In PCA this distribution $\pi_h$ is given by the best low dimensional projection of $\pi$, and in K-means $\pi_h$ by the best discrete distribution of $K$ centroids. This idea is also related to the problem of fitting densities, \ie\ estimating the parameters $h \in \Hcal \subseteq \R^{M}$ of a parametrized distribution $\pi_h$ that best fits $\pi$. Two notable examples of such a learning task are \emph{Gaussian Mixture Modeling} (GMM) \citep{Dasgupta} and generative adversarial netwoks \citep{Goodfellow}. In order to find $h\in \R^{M}$ a principled way is to consider the negative likehood loss function $\ell(\xbf,h)=-\log(\pi_h(\xbf))$ that corresponds to minimizing the risk $\KL(\pi||\pi_{h})$ where $\KL$ is the Kullback-Leibler divergence. However, this approach is sometimes flawed, \textit{e.g.} when the data distribution is supported on a low-dimensional space or does not admit a density so that $\KL(\pi||\pi_{h})$ is undefined or infinite \citep{arjovsky2017towards}. As described in many contexts such as generative modeling \citep{pmlr-v84-genevay18a,arjovsky17a} or deconvolution problems \citep{RIGOLLET20181228,dedecker2013minimax} the Wasserstein distance, or its entropic regularized counterpart, is an interesting alternative fitting criterion to $\KL$. It boils down to minimizing a different risk $\widetilde{\Rcal}(\pi,h):=\W_p(\pi,\pi_h)$ which is not based on a loss function but can also be written as a Wasserstein distance. In this context, we directly have the bound $\sup_{h \in \Hcal}|\widetilde{\Rcal}(\pi,h)-\widetilde{\Rcal}(\pi',h)| \leq \W_p(\pi,\pi')$ using the triangle inequality.
\end{remark}

\subsection{Loss Functions that are $p$-th Power of a Lipschitz Function \label{sec:lip_tasks}}

Compression-type tasks are special cases of loss functions that can be written as the $p$-th power of a Lipchitz continuous function. Indeed, if $P_{h}$ is a projection function then $D(\xbf, P_{h}(\xbf))-D(\ybf, P_{h}(\ybf)) \leq D(\xbf, P_{h}(\ybf))-D(\ybf, P_{h}(\ybf)) \leq  D(\xbf,\ybf)$, thus, by a symmetrical argument, $|D(\xbf, P_{h}(\xbf))-D(\ybf, P_{h}(\ybf))|\leq D(\xbf,\ybf)$. Interestingly, these more general tasks are also Wasserstein regular: 
\begin{restatable}{proposition}{villanisuperbound}
\label{theo:villani_super_bound}
Let $(\Xcal,D)$ be a complete separable metric space. Consider a loss function that can be written for $h \in \Hcal$ as $\ell(\cdot,h) = \phi_{h}^{p},$ where $p \geq 1$ and $\phi_{h} \in \operatorname{Lip}_{L}(\Xcal,\R_+),$ then
\begin{equation*}
\forall \pi,\pi' \in \P_p(\Xcal), \|\pi-\pi'\|_{\mathcal{L}(\Hcal),p}\leq L \W_p(\pi,\pi')\,.
\end{equation*} 
In other words, the task $\mathcal{L}(\Hcal)$ is $p$-Wasserstein regular with constant $L$.
\end{restatable}
\begin{proof}
Using \citet[Proposition 7.29]{Villani} we have 
\begin{equation*}
\label{eq:eqvillaniii}
\left|\left(\int \phi_h(\xbf)^{p}\dr \pi(\xbf)\right)^{1/p}-\left(\int \phi_h(\ybf)^{p}\dr \pi'(\ybf)\right)^{1/p}\right|
\leq L \W_p(\pi,\pi')\,,
\end{equation*}
since $\phi_{h} \in \operatorname{Lip}_{L}(\Xcal,\R_+)$. The conclusion follows by taking the supremum over $h \in \Hcal$.
\end{proof} 
As described previously, this argument can be used to recover Wasserstein regularity of compression-type tasks as $\ell(\xbf,h) = D(\xbf, P_{h}(\xbf))^p$ is the $p$-th power of a $1$-Lipschitz function. More importantly, the previous property allow us to prove that many \emph{supervised} learning tasks are also Wasserstein regular as described in the next example sections.

\subsubsection{Regression Tasks \label{sec:regression_tasks}}

The first example we consider is that of the regression tasks where $\Xcal=\R^{d+K}$ is endowed with the metric $D(\xbf=(\zbf,\ybf),\xbf'=(\zbf',\ybf')) = \|\zbf-\zbf'\|_{\R^{d}}+\|\ybf-\ybf'\|_{\R^{K}}$ for some norm $\|\cdot\|_{\R^{d}}$ (\textit{resp.} $\|\cdot\|_{\R^{K}}$) on $\R^{d}$ (\textit{resp.} $\R^{K}$). The loss function is given by $\ell(\xbf=(\zbf,\ybf),h)=\|\ybf-h(\zbf)\|^{p}_{\R^{K}}$ for some $p \geq 1$ and a regressor $h$ that belongs to the hypothesis space $\Hcal\subseteq \operatorname{Lip_{L}}(\R^{d}, \R^{K})$. In particular when $p = 2, \|\cdot\|_{\R^{K}} = \|\cdot\|_{2},$ the setting corresponds to a standard regression problem with the squared loss, and when $p=1, \|\cdot\|_{\R^{K}} = \|\cdot\|_{1},$ to the least absolute deviation regression problem. Then, for $\xbf=(\zbf,\ybf),\xbf'=(\zbf',\ybf')$, we have
\begin{equation*}
\begin{split}
\left|\|\ybf-h(\zbf)\|_{\R^{K}}-\|\ybf'-h(\zbf')\|_{\R^{K}}\right| &\leq \|\ybf-\ybf'-(h(\zbf)-h(\zbf'))\|_{\R^{K}} \\
&\leq \|\ybf-\ybf'\|_{\R^{K}}+ \|h(\zbf)-h(\zbf')\|_{\R^{K}} \\
&\leq \|\ybf-\ybf'\|_{\R^{K}}+ L\|\zbf-\zbf'\|_{\R^{d}} \\
&\leq \max\{L,1\} D(\xbf=(\zbf,\ybf),\xbf'=(\zbf',\ybf'))\,.
\end{split}
\end{equation*} 
Consequently the loss can be written as the $p$-th power of a Lipschitz function and the task is $p$-Wasserstein regular with constant $\max\{L,1\}$ using Proposition \ref{theo:villani_super_bound} (with $\W_p$ computed with the distance $D$).

This setting encompasses regressors such as multi-layer perceptron (MLP) $h(\zbf)=f_{\MLP}(\zbf)= T_{J} \circ \rho_{J-1} \circ \cdots \circ \rho_{1} \circ T_1(\zbf)$ where $T_j(\wbf)=\Mbf_j \wbf +\bbf_j$ is an affine function with bounded weights and $\rho_j$ is a non-linear activation function. Designing Lipschitz-continuous neural networks and computing precisely their Lipschitz constant is an (NP)hard problem and is an active line of research \citep{Bengio_lip,Fazlyab,Latorre2020Lipschitz,kim2021lipschitz}. However, for fully-connected networks such as MLP with $1$-Lipschitz activation functions (\textit{e.g.} ReLU, Leaky ReLU, SoftPlus, Tanh, Sigmoid, ArcTan or Softsign) a simple upper-bound of the Lipschitz constant of $f_{\MLP}$ is given by $L = \Pi_{j=1}^{J} \|\Mbf_j\|_{2 \rightarrow 2}$ \citep{Bengio_lip} where $\|\cdot\|_{2 \rightarrow 2}$ denotes the $2$-operator norm for matrices. This bound is not necessarily tight, however we can use it to prove that regression tasks using MLP with bounded parameters and with $1$-Lipschitz activation functions is Wasserstein regular as soon as $\forall j \in \integ{J}, \|\Mbf_j\|_{2 \rightarrow 2} \leq R$ for some $R>0$.

\subsubsection{Classification Tasks \label{sec:binary_classif}}

Binary classifications tasks can also be related to Wasserstein regularity. These problems corresponds to $\Xcal= \R^{d} \times \{+1,-1\}$ and often rely on \emph{convex surrogates} of the $0-1$ loss such as $\ell(\xbf=(\zbf,y),h)=\beta(y h(\zbf))$ where $y\in \{-1,+1\}$, $h:\R^{d} \rightarrow \R$ and $\beta:\R \rightarrow \R_+$ is convex \citep{bartlett2006convexity}. Well known examples include the logistic loss $\beta(t)=\log(1+e^{-t}),$ the hinge loss $\beta(t)=\max(1-t,0)$ or the squared hinge loss $\beta(t)=\max(1-t,0)^{2}$. In all of these cases $\beta$ can be written as $\varphi^{p}$ for some Lipschitz function $\varphi$ and $p \geq 1$. If the hypothesis space is made of uniformly bounded and Lipschitz classifiers then the previous reasoning also applies. Indeed if $h \in \Hcal \subseteq \operatorname{Lip}_{L}(\R^{d},\R),$ with $\|h\|_{\infty} \leq B$ then, for any $\xbf=(\zbf,y),\xbf'=(\zbf',y'),$ we have 
\begin{equation*}
\begin{split}
|\varphi(yh(\zbf))-\varphi(y'h(\zbf'))|&\leq |\varphi(yh(\zbf))-\varphi(yh(\zbf'))|+|\varphi(yh(\zbf'))-\varphi(y'h(\zbf'))|\\
&\leq L_\varphi (|yh(\zbf)-yh(\zbf')|+|yh(\zbf')-y'h(\zbf')|) \\
&\leq L_\varphi(|y| |h(\zbf)-h(\zbf')|+|h(\zbf')| |y-y'|) \\
&\leq L_\varphi\max(L,B)(\|\zbf-\zbf'\|_2+|y-y'|).
\end{split}
\end{equation*}
Consequently, by Proposition \ref{theo:villani_super_bound}, the task is $p$-Wasserstein regular with constant $L_\varphi\max(L,B)$ with $\W_p$ computed with the distance $D((\zbf,y),(\zbf',y'))=\|\zbf-\zbf'\|_2+|y-y'|$. In particular, this example includes classifiers of the type $h =\rho \circ f_{\MLP}$ where $f_{\MLP}: \R^{d} \rightarrow \R$ as in Section \ref{sec:regression_tasks} and $\rho:\R \rightarrow [-1,1]$ is an ‘‘output-layer'' function that is Lipschitz such as the $\operatorname{tanh}$ function (in this case $B = 1$).

\section{Application to Compressive Statistical Learning}
\label{sec:compress_section}

In the previous sections, we identified conditions allowing to 
1) upper bound task-specific metrics by a Wasserstein distance $\W_{p}$ (notion of Wasserstein regularity, Section \ref{sec:wasserstein_learnability});
2) control $\W_p$ by an MMD, modulo an exponent $\delta \in (0,1]$, under certain conditions on the model set of distributions $\Sfrak$ at stake and the kernel of the MMD (Section \ref{sec:wass_mmd}). We apply in this section these results to the theory of compressive statistical learning. The goal is to establish theoretical guarantees for CSL. This section is organized as follows: we first recall the main concepts and objectives of CSL, then we introduce a generalization of the existing framework (namely the Hölder LRIP) which we finally connect with the results of Section \ref{sec:wasserstein_learnability} and \ref{sec:wass_mmd} to establish the guarantees.

\subsection{Compressive Statistical Learning}

In contrast to the empirical risk minimization approach described in Section \ref{section:gentle} the principle of compressive statistical learning is to learn a hypothesis $\hat{h}$ by relying on a single \emph{sketch} vector $\sbf \in \R^{m}$ instead of the full dataset $(\xbf_i)_{i \in \integ{n}}$ (or equivalently the empirical distribution $\pi_n$). This sketch aims to summarize the properties of the empirical distribution that are essential for the learning task. The benefits of this approach are numerous. First, as a side effect of its definition, the sketching mechanism is adapted for distributed and streaming scenarios since the sketch of a concatenation of datasets is a simple average of the sketches of those datasets. More importantly, when $m\ll nd $ the data are drastically compressed, which facilitates their storage and transfer. Finally, it has be shown that sketching can preserve privacy \citep{chatalic:tel-03023287,balog} since the transformation which turns a dataset into a single vector discard the individual-user informations.

The compressive statistical learning framework requires two steps: 1) to compute a sketch vector $\sbf \in \R^{m}$ of size $m$ driven by the complexity of the learning task 2) to address a nonlinear least-squares optimization problem on this sketch to learn the hypothesis $\hat{h}$ that best solves our learning task. As described latter, this step is an inverse problem in the space of measures and can be related to the generalized method of moments \citep{hall2005generalized}. We summarize in the following the main concepts related to the CSL theory established in \citet{gribonval2020compressive,gribonval2020statistical} that will be useful to describe our contributions.

\begin{figure}[t!]
\begin{center}
\includegraphics[width=1\linewidth]{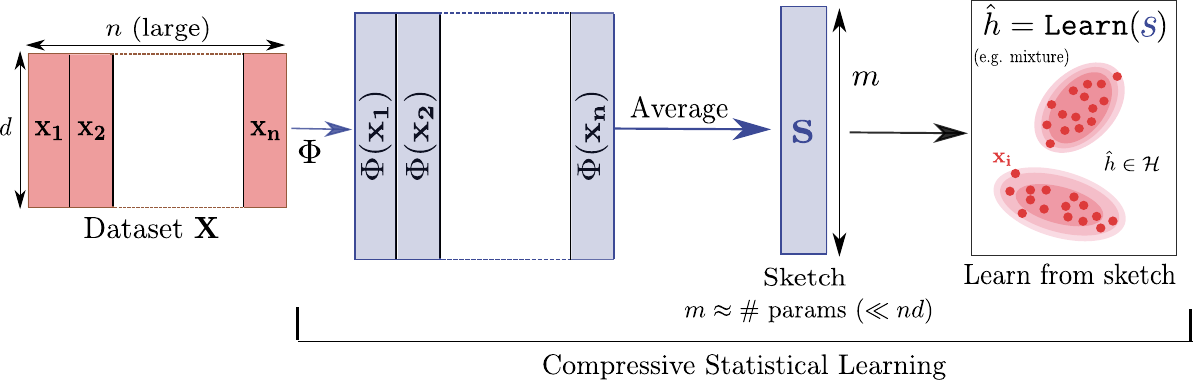}
\end{center}
\caption{\label{fig:sketch}The principle of CSL (when $\Xcal=\R^{d}$). From a dataset $\X$ with $n$ samples (usually $n$ is large) we push each sample $\xbf_i \in \R^{d}$ to either $\R^{m}$ or $\C^{m}$ using a well-chosen feature function $\Phi(\xbf_i)$. The second step is to average all the $\Phi(\xbf_i)$ to form a \emph{sketch} of of the dataset $\sbf=\frac{1}{n} \sum_{i=1}^{n} \Phi(\xbf_i)$ (which is convenient for distributed data and data streams). We finally learn a hypothesis $\hat{h} \in \Hcal$ based only the sketch whose size is driven by the learning task and is usually of the order of the number of parameters to learn. }
\end{figure}

\subsubsection{The Sketching Operator} Given a collection of data points $\X=(\xbf_i)_{i\in \integ{n}}$ where $\xbf_i \in \Xcal$, the CSL procedure relies on an operator $\Phi$ which maps a sample $\xbf_i \in \Xcal$ to either $\Phi(\xbf_i) \in \R^{m}$ or $\mathbb{C}^{m}$. Based on this operator, a sketch of a dataset $(\xbf_i)_{i \in \integ{n}}$ is defined \textit{via} the \emph{vector}
\begin{equation*}
\sbf:=\frac{1}{n} \sum_{i=1}^{n} \Phi(\xbf_i)\,.
\end{equation*} 
The main challenge is to find, depending on the task, an adequate $\Phi$ and a reasonable sketch size $m$ to learn the specific task (see Figure \ref{fig:sketch}). As described in the next sections this can be achieved by exploiting links with the formalism of linear inverse problems, compressive sensing, and low complexity recovery. Given $\Phi$, the associated \emph{sketching operator} is \begin{equation}
\label{eq:sketching_operator_eq}
\begin{split}
\Acal:  \ \P(\Xcal) &\rightarrow \R^{m} \text{ or } \mathbb{C}^{m} \\
 \pi &\rightarrow \Acal(\pi):=\int_{\Xcal} \Phi(\xbf)\dr \pi(\xbf)\,.
\end{split}
\end{equation}
This operator is ‘‘linear''\footnote{We can extend $\Acal$ to the space of finite signed measure $\Mcal(\Xcal)$ where it is a linear operator in the usual sense.}  in $\pi$ in that $\Acal((1-\lambda)\pi+\lambda \pi')=(1-\lambda)\Acal(\pi)+\lambda\Acal(\pi')$ for $\lambda \in [0,1]$. When applied to the empirical distribution $\pi_n=\frac{1}{n} \sum_{i=1}^{n} \delta_{\xbf_i}$ we recover the sketch $\sbf$ as
\begin{equation*}
\Acal(\pi_n)=\Acal(\frac{1}{n} \sum_{i=1}^{n} \delta_{\xbf_i}) = 
\frac{1}{n} \sum_{i=1}^{n} \Phi(\xbf_i)=\sbf\,.
\end{equation*}
This sketch can be understood as a the average of generalized empirical moments on the training collection based on the feature function $\Phi$ \citep{hall2005generalized}.

\subsubsection{The Model Set and the Decoder} A central operator in CSL is the \emph{decoder} that is, informally, an operator $\D$ that goes in the other direction than $\Acal$: it takes as input a vector and outputs a probability distribution. Ideally we would like to be able to perfectly decode our original distribution from the sketch, \ie\ to find $\D$ such that $\D\circ \Acal = \operatorname{id}$. However, as described in \citet{gribonval2020compressive}, we can not hope to perfectly recover any distribution without assumptions. These assumptions are formalized by the means of a \emph{model set} $\Sfrak\subseteq \P(\Xcal)$ which describes a subset of probability distributions where the decoding is perfect and robust to noise. A \emph{decoder} is defined very generally as an operator
\begin{equation*}
\D:  \sbf \rightarrow \D[\sbf] \in \Sfrak \,.
\end{equation*}
Suppose for the moment that we know how to sketch and how to decode \ie\ we know $\Acal$ and $\D$. Given a sketch $\sbf$ of the dataset and a decoder $\D$ we can find a hypothesis based on the following risk minimization:
\begin{equation*}
\hat{h} \in \underset{h \in \Hcal}{\arg\min} \ \Rcal(\D[\sbf],h)\,.
\end{equation*}
As such in CSL the risk $\Rcal(\D[\sbf],\cdot)$ acts as a proxy for the empirical risk $\Rcal(\pi_n,\cdot)$, and one hopes to produce a hypothesis which is as good as the one obtained by empirical risk minimization (ERM). At first sight it seems that solving $\underset{h \in \Hcal}{\arg\min} \ \Rcal(\D[\sbf],h)$ is as hard as doing ERM. The crucial point is that, by definition, $\D[\sbf]$ is a probability distribution in the model set $\Sfrak$ and thus usually admits a simple expression. Consequently finding $\hat{h}$ with this procedure is most of the time simpler than doing ERM.\\

\emph{How to obtain statistical guarantees ?} Theoretical guarantees of CSL can be derived when the operator $\Acal$ satisfies the so-called \emph{Lower Restricted Isometric Property} (LRIP) \citep{gribonval2020compressive,KerivenIOP}:
\begin{equation}
\label{eq:informalLrip}
\forall \pi,\pi' \in \Sfrak, \|\pi-\pi'\|_{\mathcal{L}(\Hcal)} \lesssim \|\Acal(\pi)-\Acal(\pi')\|_2\,.
\end{equation}
This property implies that two distributions in the model set $\Sfrak$ (\ie\ ‘‘simple'' distributions for which we hope that everything works ‘‘fine'') 
 have the same sketches then they are equivalent with respect to the task-dependent metric $\|\cdot\|_{\mathcal{L}(\Hcal)}$, \ie, they lead to the same risk for every hypothesis. 
 When this condition holds, the following decoder $\D$ provides many interesting guarantees:
\begin{equation}
\label{eq:idealdecoder}
\D[\sbf]\in \underset{\pi \in \Sfrak}{\arg\min} \|\Acal(\pi)-\sbf\|_2\,.
\end{equation}
Indeed it can be shown \citet{gribonval2020compressive} that this decoder is \emph{ideal} in the sense that it satisfies the \emph{Instance Optimality Property} (IOP) which allows to have a control on the excess risk for \emph{all} probability distributions. We will describe this property more in depth in Section \ref{sec:holderlripsection} and only give now its consequence when we consider any data generating distribution $\pi \in \P(\Xcal)$ associated to the optimal hypothesis $h^{*} \in \arg\min_{h \in \Hcal} \Rcal(\pi,h)$ and $\pi_n$ an empirical distribution associated to samples from $\pi$. Suppose that we have access only to a sketch $\sbf=\Acal(\pi_n)$ of this empirical distribution with $\Acal$ that satisfies the LRIP. Consider the decoder $\D$ defined in \eqref{eq:idealdecoder} and $\hat{h}$ such that  $\hat{h} \in \arg\min_{h \in \Hcal}\Rcal(\D[\sbf],h)$. Using the IOP property it can be shown that 
\begin{equation*}
\|\pi-\D[\sbf]\|_{\mathcal{L}(\Hcal)} \lesssim \operatorname{Bias}(\pi,\Sfrak)+\|\Acal(\pi)-\Acal(\pi_n)\|_2\,,
\end{equation*}
where $\operatorname{Bias}(\pi,\Sfrak)$ is a \emph{bias term} (which will be properly defined latter) which is large when $\pi$ is far from the model set and vanishes when $\pi \in \Sfrak$. This leads to the following bound on the excess risk:
\begin{equation*}
\Rcal(\pi,\hat{h})-\Rcal(\pi,h^{*}) \lesssim \operatorname{Bias}(\pi,\Sfrak)+ \|\Acal(\pi)-\Acal(\pi_n)\|_2\,.
\end{equation*}
This inequality echoes the well-known risk decomposition in statistical learning: the first term $\operatorname{Bias}(\pi,\Sfrak)$ resembles the approximation error coming from the chosen model and $\|\Acal(\pi)-\Acal(\pi_n)\|_2$ resembles the estimation error and typically converges to zero with a $n^{-1/2}$ rate. Consequently, if the model set $\Sfrak$ is such that the bias term is of the order of the true risk $\Rcal(\pi,h^{*})$ (this can be ensured for certain learning tasks \citealp{gribonval2020statistical}) then $\Rcal(\pi,\hat{h})$ converges to the order of the true risk as $n$ grows.

\subsection{Extending Compressive Statistical Learning Guarantees with Hölder LRIP and Hölder IOP}
\label{sec:holderlripsection}

In this section we define an extended notion of LRIP, namely the Hölder LRIP, and show that it can be exploited to control the statistical performance of compressive statistical learning. The H\"older LRIP is basically a relaxation of the LRIP with a Hölder exponant $\delta \in (0,1]$. To connect with the previous sections, this exponent will also be related to the one found in Section \ref{sec:wass_mmd} to control $\W_p$ by the MMD. We consider the following definition:
\begin{restatable}[Hölder LRIP and IOP]{definition}{holderlripandiop}
\label{def:holderlrip_and_iop}
Consider a learning task $\mathcal{L}(\Hcal)$, an exponent $p \in [1,+\infty)$, and a model set $\Sfrak$. A sketching operator $\Acal: \P(\Xcal) \rightarrow \mathbb{C}^{m}$ satisfies the Hölder LRIP for $\delta \in (0,1]$ with error $\eta \geq 0$ and constant $C>0$ if
\begin{equation}
\label{eq:lrip}
\tag{Hölder-LRIP}
\forall \pi, \pi' \in \Sfrak, \ \|\pi-\pi'\|_{\mathcal{L}(\Hcal),p}\leq C\ \|\Acal(\pi)-\Acal(\pi')\|^{\delta}_2+\eta\,.
\end{equation}
A decoder $\D: \mathbb{C}^{m} \rightarrow \Sfrak$ satisfies the Hölder IOP for $\delta \in (0,1]$ with error $\eta \geq 0$  and constant $C>0$ if
\begin{equation}
\label{eq:iop}
\tag{Hölder-IOP}
\forall \pi \in \P(\Xcal), \forall \ebf \in \mathbb{C}^{m}, \|\pi-\D[\Acal(\pi)+\ebf]\|_{\mathcal{L}(\Hcal),p}\leq \operatorname{Bias}(\pi,\Sfrak)+C\ \|\ebf\|^{\delta}_2+\eta\,,
\end{equation}
where $\operatorname{Bias}(\cdot,\Sfrak):\P(\Xcal)\rightarrow \R_{+}$ is a function such that $\forall \pi \in \Sfrak, \ \operatorname{Bias}(\pi,\Sfrak)=0$.
\end{restatable}

The instance optimality property means that the decoder is able to retrieve (with error $\eta$) any probability distribution when the modeling is exact (\ie\ $\pi \in \Sfrak$ and $\ebf=0$). As this condition is rarely met in practice, the IOP property also captures robustness to some noise $\ebf$ and modeling error. As such, the decoding error $\|\pi-\D[\Acal(\pi)+\ebf]\|_{\mathcal{L}(\Hcal),p}$ is bounded by the amplitude of the noise and the bias term. The previous definition generalizes the classical LRIP and IOP property (including their definition with an error term $\eta$ \citealp{gribonval2020compressive}) since both are met when $\delta=1$.  It turns out that both Hölder LRIP and IOP are equivalent as stated in the next result:

\begin{restatable}[Equivalence of Hölder LRIP and IOP]{proposition}{holderlripandiopareequivalent}
\label{prop:holderlripandiopareequivalent}
Consider a learning task $\mathcal{L}(\Hcal)$, an exponent $p \in [1,+\infty)$ , and a model set $\Sfrak$. 
\begin{enumerate}[label=(\roman*)]
\item If $\Acal$ satisfies \eqref{eq:lrip} with error $\eta \geq 0$ and constant $C>0$ then the ''ideal" decoder defined by
\begin{equation}
\label{eq:optimal_decoder}
\D[\sbf]\in \underset{\pi \in \Sfrak}{\arg\min} \|\Acal(\pi)-\sbf\|_2\,,
\end{equation}
satisfies \eqref{eq:iop} with constant $2C>0$, error $\eta \geq 0$ and $$\operatorname{Bias}(\pi,\Sfrak):=\inf_{\tau \in \Sfrak} \|\pi-\tau\|_{\mathcal{L}(\Hcal),p}+2C \|\Acal(\pi)-\Acal(\tau)\|_2^{\delta}\,.$$
\item Conversely if the decoder $\D$ defined in \eqref{eq:optimal_decoder} satisfies \eqref{eq:iop} with error $\eta \geq 0$, constant $C>0$ and $\operatorname{Bias}(\pi,\Sfrak)$ defined above, then $\Acal$ satisfies \eqref{eq:lrip} with constant $C>0$ and error $2 \eta$. 
\end{enumerate}
\end{restatable}

The proof is deferred to Appendix \ref{app:proofseccompress_section}. In this paper we always assume that the minimization problem \eqref{eq:optimal_decoder} has at least one solution and, as in \citet{bourriergribonval}, the result can be adjusted to handle the case where the $\arg\min$ defining the ideal decoder is only approximated to a certain accuracy. This proposition states that if the Hölder LRIP is satisfied, then the decoder that returns the element in the model that best matches the measurement $\Acal(\pi)$ is instance optimal. On the other hand, if some instance optimal decoder exists, then the Hölder LRIP must be satisfied. In other words, when the Hölder LRIP is satisfied,  we know that a negligible amount of information is lost when encoding a probability measure in $\Sfrak$. As advertised the Hölder LRIP allows us to have some guarantees on the excess risk as described in the next theorem:

\begin{restatable}[Compressed statistical learning guarantees]{theorem}{stat_guarantees}
\label{theo:stat_guarantees}
Consider a sketching operator $\Acal: \P(\Xcal) \rightarrow \mathbb{C}^{m}$ that satisfies the Hölder LRIP with $\delta \in (0,1]$, constant $C>0$ and error $\eta \geq 0$. Let $\pi \in \P(\Xcal)$ be the data generating distribution and $\xbf_1, \cdots, \xbf_n \sim \pi$ (not necessarily \textit{i.i.d.}). Consider the empirical distribution $\pi_n=\frac{1}{n} \sum_{i=1}^{n} \delta_{\xbf_i}$ and a sketch of the dataset $\sbf=\Acal(\pi_n)$.

Let $h^{*} \in \arg\min_{h \in \Hcal} \Rcal(\pi,h)$ be the optimal hypothesis and $\hat{h} \in \arg\min_{h \in \Hcal} \Rcal(\D[\sbf],h)$ where $\D[\sbf]\in \arg\min_{\pi \in \Sfrak} \|\Acal(\pi)-\sbf\|_2$. Then
\begin{equation*}
\Rcal(\pi,\hat{h})^{1/p}-\Rcal(\pi,h^{*})^{1/p} \leq 2\operatorname{Bias}(\pi, \Sfrak)+2C\|\Acal(\pi)-\Acal(\pi_n)\|^{\delta}_2+2\eta\,,
\end{equation*}
where $\operatorname{Bias}(\pi, \Sfrak)=\inf_{\tau \in \Sfrak} \|\pi-\tau\|_{\mathcal{L}(\Hcal),p}+2C \|\Acal(\pi)-\Acal(\tau)\|_2^{\delta}$.
\end{restatable}
\begin{proof}
Using Proposition \ref{prop:holderlripandiopareequivalent} we know that the decoder is instance optimal and satisfies the Hölder IOP \eqref{eq:iop}. Consider $\ebf=\Acal(\pi_n)-\Acal(\pi)$ we have by definition $\|\pi-\D[\Acal(\pi)+\ebf]\|_{\mathcal{L}(\Hcal),p}\leq \operatorname{Bias}(\pi,\Sfrak)+C\ \|\ebf\|^{\delta}_2+\eta$ which gives $\|\pi-\D[\Acal(\pi_n)]\|_{\mathcal{L}(\Hcal),p}\leq \operatorname{Bias}(\pi,\Sfrak)+C\ \|\Acal(\pi_n)-\Acal(\pi)\|^{\delta}_2+\eta$. 
We conclude the proof by using $\Rcal(\pi,\hat{h})^{1/p}-\Rcal(\pi,h^{*})^{1/p} \leq 2 \|\pi-\D[\sbf]\|_{\mathcal{L}(\Hcal),p}=2 \|\pi-\D[\Acal(\pi_n)]\|_{\mathcal{L}(\Hcal),p}$.
\end{proof}

When the samples $\xbf_1,\cdots, \xbf_n$ are \textit{i.i.d.}\footnote{We emphasize that the \textit{i.i.d.} assumption is not required in order to obtain the bound in Theorem \ref{theo:stat_guarantees}. It is only used to guarantee that $\|\Acal(\pi)-\Acal(\pi_n)\|_2\underset{n \rightarrow+\infty}{\rightarrow} 0$. the term $\|\Acal(\pi)-\Acal(\pi_n)\|_2$, which is the empirical estimation error, goes to zero as $n \rightarrow + \infty$ with a typical $n^{-1/2}$ rate.} This result is essential: it illustrates that if we have carefully designed $\Sfrak$ so that the bias term is \emph{of the order of} $\Rcal(\pi,h^{*})^{1/p}$, and if we know a sketching operator with the Hölder LRIP property, then $\Rcal(\pi,\hat{h})^{1/p}$ converges to a constant times the order of the true risk as $n$ grows (when the error term $\eta = 0$). The notable price to pay between this result and the one presented in the context of the LRIP ($\delta=1$) is that while the usual guaranteed speed of convergence is $O(n^{-1/2})$ here it becomes $O(n^{-\delta/2})$, which is slower. The next section outlines how the various results presented in this work can be applied to establish the Hölder LRIP.

\subsection{Connecting the Hölder LRIP with the Results of Section \ref{sec:wass_mmd} and \ref{sec:wasserstein_learnability} \label{sec:existence_of_sketch}}

As described in Theorem \ref{theo:stat_guarantees}, guarantees on the excess risk can be achieved with a sketching operator $\Acal$ that satisfies the Hölder LRIP. In this section, we provide elements to obtain this property. In line with the approach developed in \citet{gribonval2020compressive}, the core of our reasoning is based on the theory of kernel embedding of probability distributions and random features. 

\subsubsection{Restricted Wasserstein Regularity is Necessary to the Hölder LRIP}

Firstly, a prerequisite for the Hölder LRIP is the Wasserstein regularity condition (Definition \ref{def:wass_learn}) of the learning task when restricted to the model set $\Sfrak$. More precisely we have the following result:

\begin{restatable}[Restricted Wasserstein regularity is necessary]{proposition}{wasslearnnecessary}
\label{prop:wass_learn_is_ness}
Consider $\Xcal= \R^{d}$ equipped with a norm $\|\cdot\|, p \in [1, +\infty)$, and a model set $\Sfrak \subseteq \P_p(\R^{d})$. Consider a sketching operator $\Acal$ defined by $\Phi: \R^d \rightarrow \R^m$ with $\Phi \in \Lip_{L}\left((\R^{d},\|\cdot\|), (\R^{m},\|\cdot\|_2)\right)$. If $\Acal$ satisfies \eqref{eq:lrip} with error $\eta=0$, constant $C>0$ and $\delta=1$ then 
\begin{equation*}
\forall \pi,\pi' \in \Sfrak, \|\pi-\pi'\|_{\mathcal{L}(\Hcal),p} \leq C L \W_{1}(\pi,\pi') \leq C L \W_p(\pi,\pi')\,,
\end{equation*}
where the Wasserstein distance is computed with the distance $D(\xbf,\ybf)= \|\xbf-\ybf\|$.
\end{restatable}
The proof is deferred to Appendix \ref{proof:prop:wass_learn_is_ness} and simply amounts to showing that $\|\Acal(\pi)-\Acal(\pi')\|_2 \leq L\W_1(\pi,\pi')$. According to this proposition, if $\Phi$ is Lipschitz and $\Acal$ satisfies the Hölder LRIP with $\delta=1$ then $\mathcal{L}(\Hcal)$ is \emph{necessarily} $p$-Wasserstein regular when we restrict the Definition \ref{def:wass_learn} to distributions belonging to the model set $\Sfrak$. In particular this proposition applies to the classical LRIP setting of \citet{gribonval2020compressive}. More importantly the Lipschitz hypothesis encompasses the case where $\Phi$ is defined with random Fourier features\footnote{In this setting $\Phi(\xbf)=\frac{1}{\sqrt{m}}\left( \sqrt{2}\sin(\xbf^{\top} \w_1), \sqrt{2}\cos(\xbf^{\top}\w_1), \cdots, \sqrt{2}\sin(\xbf^{\top} \w_{m/2}), \sqrt{2}\cos(\xbf^\top \w_{m/2}) \right)^{\top}$ for some random draw of $\w_j$.} as usually considered in the compressive statistical learning literature \citep{gribonval2020compressive, gribonval2020statistical, belhadji, shi2022compressive, shi_compressive_patch}. This result thus shows that a \emph{restricted} Wasserstein regularity is necessary for establishing statistical guarantees of CSL through the Hölder LRIP. 

\begin{remark}
The previous result can be easily generalized to the case where $\delta \in (0,1)$. Under the same assumptions on $\Phi$, if $\Acal$ satisfies \eqref{eq:lrip} with an error of $\eta=0$, a constant $C>0$, and $\delta \in (0,1)$, we can show that $\forall \pi,\pi' \in \Sfrak$, $\|\pi-\pi'\|_{\mathcal{L}(\Hcal),p} \leq C L^\delta \W_{1}(\pi,\pi')^{\delta} \leq C L^\delta \W_p(\pi,\pi')^{\delta}$. This condition extends the Wasserstein regularity property, and it raises the question of which learning tasks satisfy it.
\end{remark}

\subsubsection{From Wasserstein Regularity to the Kernel Hölder LRIP and Hölder LRIP}

Interestingly, a converse of Proposition \ref{prop:wass_learn_is_ness} is also true. Indeed, as shown in Section \ref{sec:wasserstein_learnability} many learning tasks are Wasserstein regular, and this, \emph{independently} of the choice of the model set $\Sfrak$. For instance, this is true for compression-type tasks such as K-means/medians, PCA, or supervised learning tasks such as regression and binary classification (see Table \ref{tab:summary_wass_task}).

Consequently, if we add the elements of Section \ref{sec:wass_mmd}, namely that $(\Sfrak, \W_p)$ is $(\kappa, \delta)$-embeddable (Definition \ref{def:kappaembedable}), we can obtain, under certain assumptions about $\kappa, \Sfrak$, that the metric associated with the task satisfies the following chain of inequalities:
\begin{equation}
\label{eq:eqroadmap}
\forall \pi,\pi' \in \Sfrak, \ \|\pi-\pi'\|_{\mathcal{L}(\Hcal),p} \stackrel{\text{Section} \ \ref{sec:wasserstein_learnability}}{\lesssim} \W_p(\pi,\pi') \stackrel{\text{Section} \ \ref{sec:wass_mmd}}{\lesssim} \|\pi-\pi'\|^{\delta}_{\kappa}\,.
\end{equation}
As shown in Section \ref{sec:wass_mmd}, the last inequality can be obtained with an MMD associated with TI, PSD kernels and under certain assumptions on the moments of the distributions in $\Sfrak$ and their regularity. In other words, by combining the results of Section \ref{sec:wass_mmd} and \ref{sec:wasserstein_learnability}, our analysis shows that for many learning tasks and with some hypothesis on the kernel $\kappa, \Sfrak$ the task metric is  bounded by $\MMD^\delta$ uniformly on $\Sfrak$. We refer to this property as the \emph{kernel Hölder LRIP}, \ie\ when there exists $C > 0, \delta \in (0,1], \eta \geq 0$ such that
\begin{equation}
\forall \pi,\pi' \in \Sfrak, \ \|\pi-\pi'\|_{\mathcal{L}(\Hcal),p} \leq C \|\pi-\pi'\|_{\kappa}^\delta + \eta\,.
\end{equation}
The echoes the \emph{kernel} LRIP  described in \citet{gribonval2020compressive} but with a Hölder exponent $\delta \in (0,1]$. Informally, our findings show that a kernel Hölder LRIP is not so difficult to obtain for many learning tasks. Therefore, as long as the MMD can be uniformly controlled on $\Sfrak$ by a distance between finite-dimensional sketches, \textit{i.e.} when
\begin{equation}
\label{eq:finite_mmd}
\forall \pi, \pi' \in \Sfrak, \ \|\pi-\pi'\|_{\kappa} \lesssim \|\Acal(\pi)-\Acal(\pi')\|_2\,,
\end{equation}
we can use all the results from the previous sections to obtain the Hölder LRIP. 

The property described in \eqref{eq:finite_mmd} depends only on the operator $\Acal$, the kernel $\kappa$, and the model set $\Sfrak$. To establish it, several strategies have been considered in the literature. For the sake of conciseness, we only provide some intuition here and refer the reader to \citet{gribonval2020compressive} for a more detailed discussion. The general idea is to construct, from a kernel $\kappa$, a function $\Phi: \R^d \rightarrow \R^m$ such that
\begin{equation}
\label{eq:pointwise}
\forall \xbf, \ybf, \ \langle \Phi(\xbf), \Phi(\ybf) \rangle_{\R^m} \approx \kappa(\xbf,\ybf)\,,
\end{equation}
and to ‘‘extend'' this approximation to pairs of probability distributions as 
\begin{equation}
\label{eq:approx_all_model}
\forall \pi,\pi' \in \Sfrak, \ \|\pi-\pi'\|^2_{\kappa} \approx \|\Acal(\pi)-\Acal(\pi')\|^2_2\,,
\end{equation}
where $\Acal$ is given by $\Phi$ as in \eqref{eq:sketching_operator_eq}. Ensuring \eqref{eq:pointwise} is a well established area of research and, when $\kappa$ is TI, PSD, approaches such as random Fourier features (RFF) \citep{Rahimi}, which rely on Bochner's theorem, can be used (see \textit{e.g.} \citealt{liu2021random} for a review). On the other hand, condition \eqref{eq:approx_all_model} is much more challenging to obtain. For TI, PSD kernels RFF can also be used: given a pair $\pi,\pi'$, the main strategy is to prove a pointwise control of the form $(1-\rho)\|\pi-\pi'\|^2_{\kappa} \leq \|\Acal(\pi)-\Acal(\pi')\|^2_2 \leq (1+\rho)\|\pi-\pi'\|^2_{\kappa}$ with high probability for $\rho \in (0,1]$, and then being able to control certain \emph{covering numbers} related to $\Sfrak$ to obtain a uniform control \citep{gribonval2020compressive, belhadji}. Another approach, considered for example in \citet{chatalic2022mean}, is to construct $\Phi$ based on data-dependent Nyström approximation which exploits a small random subset of the dataset (and also requires controlling covering numbers). These approaches ensure that for a sufficiently large but controlled $m$, the condition \eqref{eq:approx_all_model} is satisfied and therefore also \eqref{eq:finite_mmd}.

\subsubsection{Discussion}

As a consequence, when the task $\mathcal{L}(\Hcal)$ is $p$-Wasserstein regular and the space $(\Sfrak,\W_p)$ is $(\kappa,\delta)$-embeddable the approach presented in this paper combined with the one of \citet{gribonval2020compressive} to obtain \eqref{eq:finite_mmd} show that sketching operators based on random Fourier features are suited for a wide range of tasks and lead to CSL guarantees. With this strategy, the convergence rate of the empirical risk (Theorem \ref{theo:stat_guarantees}) is governed by the exponent $\delta \in (0,1]$ resulting from the comparison between $\W_p$ and the MMD. This can be placed in the context of results already obtained in CSL for compressive clustering and compressive mixture modeling.

Firstly, it is already established that for mixtures of $K$ Diracs (used in compressive $K$-means) separation assumptions on the centers are necessary to establish the LRIP \citep[Lemma 3.4.]{gribonval2020statistical}. One might ask if these assumptions can be dispensed at the cost of slower convergence with the Hölder LRIP. In this framework, our results demonstrate that the distance $\W_p$ cannot be controlled by the MMD when $\delta > 2/K$ (Corollary \ref{corr:corrolary_regularity}). This raises the question of whether this rate is indeed achievable without separation assumptions, and if, in such a case, \eqref{eq:finite_mmd} could also be obtained, which would imply the Hölder LRIP with $\delta = 1/(2K)$ \emph{without separation}.

Furthermore, these same separation assumptions are also used for compressive learning of Gaussian mixture (for compressive GMM estimation). Interestingly, in this case, Theorem \ref{theo:regularcase} ensures that we can control $\W_p$ by $\MMD^{\delta}$ with an exponent $\delta$ as close as desired to $\delta = \frac{1}{2p}$ and with a kernel of the Matérn class. Establishing control \eqref{eq:finite_mmd} without separation for these models would enable obtaining learning rates of the order of $n^{-1/(4p)}$ for compressive GMM with relaxed assumptions.

\section{Conclusion \& Perspectives}

The main contributions of this paper are the following. We establish different bounds between metrics between probability distributions. We show that for many learning tasks, the task-related metric can be controlled by a Wasserstein distance. In particular, many supervised and unsupervised tasks fall into this category (PCA, K-Means, GMM learning, linear and nonlinear regression...). We show that the Wasserstein distance can be controlled by kernel norms to the power of a Hölder exponent smaller than $1$ and under certain conditions on the regularity of the kernel and of the distributions at stake (by introducing a \emph{model set} of distributions). These different results allow us to establish learning guarantees in the context of \emph{compressive learning} whose goal is to summarized the training data in a single vector, by a so-called \emph{sketching operator}, and to rely solely on this vector to solve the learning task. The different bounds allow us to establish a property called the \emph{Hölder LRIP} that generalizes the LRIP property in compressive learning and provide a control of the excess risk related to the compressive learning procedure. Therefore, one of the contributions of this article is to provide a general framework for obtaining compressive learning guarantees.

This work opens many perspectives. The first one is to use our results for new compressive learning tasks that have been tackled in practice but for which theoretical guarantees are missing. In particular, we envision applications of our framework for learning generative models based on sketching \citep{schellekens2020compressive}, denoising \citep{shi2022compressive} or for classification tasks \citep{schellekens2018compressive}. Related to the compressive statistical learning theory, another interesting line of works would be to see if we can construct interesting sketching operators from the different kernels used in this paper for tasks for which there are already compressive learning guarantees. More precisely, for compressive learning tasks such as K-means and GMM one question would be to see if we can obtain compressive learning guarantees without separation assumptions \citep{gribonval2020statistical}, possibly at the price of a H\"older exponent $\delta<1$ hence with reduced rate of convergence with respect to the number of samples. Another interesting perspective concern the bounds between the Wasserstein distance and the MMD. We believe that the different results presented in this paper could be used for specific problems related to the statistical estimation of the Wasserstein distance.

\acks{This project was supported in part by the AllegroAssai ANR project ANR-19-CHIA-0009. This work was supported by the ACADEMICS grant of the IDEXLYON, project of the Université de Lyon, PIA operated by ANR-16-IDEX-0005.}

\newpage

\appendix
\section{Proofs of Section \ref{sec:wass_mmd}}
\label{section:appendixwassmmd}

\label{sec:proof_wass_mmd}

\subsection{Proof of Proposition~\ref{prop:mmdboundedwass} and Corollary~\ref{corr:mmdboundedwass} \label{sec:proof:prop:mmdboundedwass}}

We recall the proposition: 

\mmdboundedwass*

\begin{proof}
We will prove (i) $\implies$ (ii) $\implies$ (iii) $\implies$ (iv) $\implies$ (i). \\
{\bf (i) $\implies$ (ii).} Assuming (i) we prove (ii) for $p=1$. By monotonicity of the Wasserstein distance with respect to $p$ we have the conclusion for any $p \in [1,+\infty)$. Considering $\pi,\pi' \in \P_1(\Xcal)$, we have $\|\pi-\pi'\|_{\kappa} = \sup_{f \in B_{\kappa}} |\int f(\xbf) \dr \pi(\xbf)-\int f(\ybf) \dr \pi'(\ybf)|$ \citep{Sriperumbudur}. For any $f \in B_{\kappa}$, by hypothesis (i) we have $\frac{1}{C} f \in \operatorname{Lip}_1((\Xcal,D),(\R,|\cdot|))$ thus by the dual characterization of the $1$-Wasserstein distance \eqref{eq:duality}
we obtain
$\|\pi-\pi'\|_{\kappa}  \leq C \W_1(\pi,\pi')$. The implication {\bf (ii) $\implies$ (iii)} is straightforward. \\
{\bf (iii) $\implies$ (iv).} Consider $\pi= \delta_{\xbf}, \pi'= \delta_{\ybf}$ for arbitrary $\xbf,\ybf \in \Xcal$. We have $\|\pi-\pi'\|_{\kappa}^{2}= \kappa(\xbf,\xbf)+\kappa(\ybf,\ybf)-2 \kappa(\xbf,\ybf)$ and $\W_p(\pi,\pi')=D(\xbf,\ybf)$, hence the conclusion.
\\
{\bf (iv) $\implies$ (i).} Considering $f \in B_{\kappa}$, we have for any $\xbf,\ybf \in \Xcal$:
\begin{equation}
\begin{split}
|f(\xbf)-f(\ybf)|^{2}&=|\langle f,\kappa(\xbf,\cdot)\rangle_{\Hcal_{\kappa}}-\langle f,\kappa(\ybf,\cdot)\rangle_{\Hcal_{\kappa}}|^{2}=|\langle f,\kappa(\xbf,\cdot)-\kappa(\ybf,\cdot)\rangle_{\Hcal_{\kappa}}|^{2} \\
&\leq \|f\|_{\Hcal_{\kappa}}^{2}\|\kappa(\xbf,\cdot)-\kappa(\ybf,\cdot)\|_{\Hcal_{\kappa}}^{2} \leq 1\cdot( \|\kappa(\xbf,\cdot)\|_{\Hcal_{\kappa}}^{2}+\|\kappa(\ybf,\cdot)\|_{\Hcal_{\kappa}}^{2}-2 \kappa(\xbf,\ybf))\\
&=\kappa(\xbf,\xbf)+\kappa(\ybf,\ybf)-2 \kappa(\xbf,\ybf) \stackrel{(iv)}{\leq} C^{2} D^{2}(\xbf,\ybf)\,.
\end{split}
\end{equation}
This gives $|f(\xbf)-f(\ybf)| \leq C D(\xbf,\ybf)$ hence $f$ is $C$-Lipschitz with respect to the metric $D$. 
\end{proof}

\corrmmdboundedwass*

\begin{proof}
For the first part (i). Since the kernel is normalized, using formulation (iv) of the four equivalent properties of Proposition \ref{prop:mmdboundedwass} and setting $\h = \ybf-\xbf$ yields:
\begin{equation}
\label{eq:pluggedinto}
\forall \xbf, \h \in \R^{d}, \kappa(\xbf,\xbf+\h) \geq 1-\frac{C^{2}}{2} \|\h\|_2^{2}
\end{equation} 
Given any $\xbf \in \R^{d}$, since $\phi_{\xbf}$ is $C^{2}$ in a neighborhood of $\xbf$, a Taylor expansion yields:
\begin{equation}
\phi_{\xbf}(\xbf+\h)= \phi_{\xbf}(\xbf)+ \langle \nabla \phi_{\xbf}(\xbf), \h \rangle +\frac{1}{2} \h^{\top} \nabla^{2}\phi_{\xbf}(\xbf) \h + \|\h\|_2^{2} g(\xbf+\h)
\end{equation}
where $g$ is a function such that $\underset{\h \to 0}{\lim} g(\xbf+\h)=0$. Moreover $\phi_{\xbf}(\xbf)=\kappa(\xbf,\xbf)=1$ and $\nabla \phi_{\xbf}(\xbf)=0$ since the maximum of $\ybf \rightarrow \kappa(\xbf,\ybf)$ is always attained at $\ybf = \xbf$ when $\kappa$ is a PSD kernel. Hence
\begin{equation}
\kappa(\xbf,\xbf+\h)= 1 -\frac{1}{2} \h^{\top} \mathbf{H}_{\xbf} \h + o_{\|\h\|_{2} \to 0}(\|\h\|_2^{2})
\end{equation}
Considering an arbitrary unit vector $\ubf$ and $\h = \epsilon \ubf$ and using \eqref{eq:pluggedinto} gives: 
$-\epsilon^{2}\ubf^{\top} \mathbf{H}_{\xbf} \ubf \geq - \epsilon^{2}(C^{2}+o_{\epsilon \to 0}(1))$ hence
$\ubf^{\top} \mathbf{H}_{\xbf} \ubf \leq C^{2}$.
Since $\phi_{\xbf}$ is $C^{2}$ in a neighborhood of $\xbf$, by Schwarz's theorem its Hessian matrix is symmetric hence diagonalizable, and the above property implies that $\lambda_{\max}(\mathbf{H}_{\xbf}) \leq C^{2}$. As this holds for every $\xbf$ we get the desired conclusion.

For (ii), observe first that $\mathbf{H}_{\xbf} = -\nabla^{2} \phi_{\xbf}(\xbf) = -\nabla^{2} [\kappa_{0}](0)$ is independent of $\xbf$ . Since $\phi_{\xbf}$ is $C^{2}$ the matrix $\mathbf{H}_{\xbf}$ is also symmetric, and since $\phi_{\xbf}(\ybf)$ is maximum at $\ybf = \xbf$, $\mathbf{H}_{\xbf}$ is also positive semi-definite, hence $\sup_{\xbf}\lambda_{\max}(\mathbf{H}_{\xbf}) = \lambda_{\max}(-\nabla^{2}[\kappa_{0}](0)) \geq 0$ and $C := \sqrt{\lambda_{\max}(-\nabla^{2}[\kappa_{0}](0))}$ is well-defined.
Now, by Bochner's theorem, since the kernel is normalized, real-valued, and twice continuously differentiable in the neighborhood of zero, there is a frequency distribution $\Lambda \in \P_2(\R^{d})$ such that $\kappa_0(\xbf)=\E_{\w \sim \Lambda}[\cos(\w^{\top}\xbf)]$. 
It follows by standard arguments that the gradient and Hessian can be written as 
$\nabla \kappa_{0}(\xbf) = -\E_{\w \sim \Lambda} [\w \sin(\w^{\top}\xbf)]$, $\nabla^{2}\kappa_{0}(\xbf) = -\E_{\w \sim \Lambda} [\w\w^{\top} \cos(\w^{\top}\xbf)]$ 
Consequently, $\mathbf{H}_{\xbf}= -\nabla^{2}[\kappa_{0}](0) = \E_{\w \sim \Lambda} [\w\w^{\top}]$
and $C = \sqrt{\lambda_{\max}(\E_{\w \sim \Lambda} [\w\w^{\top}])}$.
Consider $\zbf \in \R^{d}$, we will show that:
\begin{equation}
2(1-\E_{\w \sim \Lambda}[\cos(\w^{\top}\zbf)]) \leq C^{2} \|\zbf\|_2^{2}
\end{equation}
which will prove property (iv) of Proposition~\ref{prop:mmdboundedwass}, and consequently all other equivalent properties. Indeed, using that $1-\cos(t) \leq \frac{t^{2}}{2}$ for all $t \in \R$ we have $1-\E_{\w \sim \Lambda}[\cos(\w^{\top}\zbf)]=\E_{\w \sim \Lambda}[1-\cos(\w^{\top}\zbf)] 
\leq \E_{\w \sim \Lambda}[\frac{|\w^{\top}\zbf|^{2}}{2}] 
= \zbf^{\top}  \left(\E_{\w \sim \Lambda} [\w\w^{\top}]\right) \zbf
\leq 
\lambda_{\max}(\E_{\w \sim \Lambda} [\w\w^{\top}]) \|\zbf\|_{2}^{2} = 
C^{2} \|\zbf\|_{2}^{2}$. \end{proof}

\subsection{Rate of Convergence of the MMD \label{sec:conv_finite_samples}}

We have the following result which is a direct consequence of Lemma 2 in \citet{briol2019statistical}:
\begin{restatable}{lemma}
Let $\pi \in \P(\Xcal)$ and $\pi_n=\frac{1}{n} \sum_{i=1}^{n} \delta_{\xbf_i}$ where $\xbf_i \sim \pi$ \textit{i.i.d}. Then
\begin{equation}
\E[\|\pi-\pi_n\|_{\kappa}^{2}]=n^{-1}(\int \kappa(\xbf,\xbf)\dr \pi(\xbf)-\int\int \kappa(\xbf,\ybf)\dr \pi(\xbf)\dr \pi(\ybf))\,,
\end{equation}
where the expectation is taken on the draws of the $(\xbf_i)_{i\in \integ{n}}$. 
\end{restatable}
\begin{restatable}{lemma}{nothing}
\label{lemma:convergence_finite_sample_lemma_mmd}
Let $\pi \in \P(\Xcal)$ and $\pi_n=\frac{1}{n} \sum_{i=1}^{n} \delta_{\xbf_i}$ where $\xbf_i \sim \pi$ \textit{i.i.d}. If $ \sup_{\xbf} k(\xbf,\xbf) \leq K$ then, for any 
$\delta \in (0,2]$, we have
\begin{equation}
\E[\|\pi-\pi_n\|_{\kappa}^{\delta}]\leq (2K)^{\delta/2} n^{-\delta/2}\,.
\end{equation}

\end{restatable}
\begin{proof}
By the previous lemma, since $ \sup_{\xbf} k(\xbf,\xbf) \leq K$ we have $\E[\|\pi-\pi_n\|_{\kappa}^{2}] \leq 2Kn^{-1}$ since for all $\xbf,\ybf \in \Xcal$ $|k(\xbf,\ybf)|\leq \sup_{\xbf \in \Xcal} k(\xbf,\xbf) \leq K$ because the kernel is positive semi-definite (the maximum value of a PSD kernel is necessarily on the diagonal). The fact that $\E[\|\pi-\pi_n\|_{\kappa}^{\delta}]\leq (2K)^{\delta/2} n^{-\delta/2}$ is a direct consequence of Jensen's inequality as $(\E[\|\pi-\pi_n\|_{\kappa}^{\delta}])^{2/\delta}\leq \E[\|\pi-\pi_n\|_{\kappa}^{2}])$ when $2/\delta \geq 1$.
\end{proof}

\subsection{Simple Bound Between Wassersein Distance and Distance Between the Means \label{sec:simple_bound_mean_wass}}

\begin{restatable}{lemma}{wass_bound_mean}
\label{lemma:wass_bound_mean}
Let $\pi,\pi' \in \P(\R^{d})$ and $\|\cdot\|$ a norm on $\R^{d}$ with the associated dual norm $\|\cdot\|_{\star}$ defined by $\|\zbf\|_{\star} = \sup_{\|\xbf\| \leq 1} \langle \xbf, \zbf \rangle$. Then for every $1 \leq p < \infty$ we have
\begin{equation}
\W_p(\pi,\pi') \geq 
\|\operatorname{m}(\pi)-\operatorname{m}(\pi')\|_\star\,,
\end{equation}
where the Wassertein distance is computed with the distance $D(\xbf,\ybf) = \|\xbf-\ybf\|$.
\end{restatable}
\begin{proof}
Consider $\ubf \in \R^{d}$ an arbitrary vector such that $\|\ubf\|=1$ and denote $f_{\ubf}(\xbf) = \langle \ubf,\xbf\rangle \in \R$ for any $\xbf \in \R^{d}$. Since $\|\ubf\|=1$ the function $f_{\ubf} : \R^{d} \to \R$ is $1$-Lipschitz with respect to $D(\xbf,\ybf) = \|\xbf-\ybf\|$, hence by duality of the Wasserstein distance \eqref{eq:duality}
\[
\left| \langle \ubf, \operatorname{m}(\pi)-\operatorname{m}(\pi') \rangle \right|
=
\left|\int f_{\ubf}(\xbf) \dr\pi(\xbf) - \int f_{\ubf}(\ybf) \dr \pi(\ybf)\right|
\leq \W_{1}(\pi,\pi').
\]
The supremum  with respect to unitary vectors $\ubf$ yields $\|\operatorname{m}(\pi)-\operatorname{m}(\pi')\|_{\star} \leq \W_{1}(\pi,\pi')$. The last step uses the fact that $\W_1(\pi,\pi') \leq \W_p(\pi,\pi')$ for any $p \in [1,+\infty)$ which concludes the proof.

\end{proof}

\subsection{Proof of Proposition \ref{weedprop} \label{sec:proof_weed_prop}}

We will prove the following result:
\weedprop*

In order to prove this proposition we will use the following lemma:
\begin{restatable}{lemma}{weedlemma}
\label{lemma:weedlemma}
\citep[Lemma 9]{nilesweed2020minimax}
Let $\pi_0,\pi_1 \in \P(\R^{d})$ be any probability distributions. Suppose that there exist two compact sets $S,T \subseteq \R^{d}$ such that $d(S,T):=\inf_{(\xbf,\ybf) \in S \times T} \|\xbf-\ybf\|_2 \geq c>0$ and that the supports of $\pi_0$ and $\pi_1$ lie in $S\cup T$. Then
\begin{equation}
\forall p \in [1,+\infty), \W_p(\pi_0,\pi_1)\geq c|\pi_0(S)-\pi_1(S)|^{1/p}\,.
\end{equation}

\end{restatable}

\begin{proof}[Of Proposition \ref{weedprop}]
This result is mainly taken from Theorem 9 in \citet{nilesweed2020minimax} but we rewrite it in our context for completeness. For any $\lambda \in[0,1]$, set
$$
\begin{array}{l}
\pi_{\lambda}:=\frac{1}{2}\left((1+\lambda) \pi_0+(1-\lambda) \pi_1\right)\,, \\
\pi'_{\lambda}:=\frac{1}{2}\left((1-\lambda) \pi_0+(1+\lambda) \pi_1\right)\,.
\end{array}
$$
Note that $\pi_{\lambda}, \pi_{\lambda}' \in \Sfrak$ by assumption and $\|\pi_{\lambda}-\pi'_{\lambda}\|_{\kappa}=\lambda\|\pi_0-\pi_1\|_{\kappa}$. Since the sets $\supp(\pi_0)$ and $\supp(\pi_1)$ are disjoint, there exist two sets $S$ and $T$ and $c>0$ such that $\supp\left(\pi_0\right) \subseteq S$ and $\supp\left(\pi_1\right) \subseteq T$ and $d(\xbf,\ybf) \geq c>0$ for any $\xbf \in S, \ybf \in T$. Moreover it is clear by definition that $\supp(\pi_\lambda)$ and $\supp(\pi'_\lambda)$ lie in $S \cup T$. The Lemma \ref{lemma:weedlemma} gives, for any $p$,
\begin{equation}
\W_p(\pi_{\lambda},\pi'_{\lambda})\geq c\left|\pi_{\lambda}(S)- \pi'_{\lambda}(S)\right|^{1 / p}=c\lambda^{1/p}\,.
\end{equation}
We obtain, for $\delta \in (0,1]$,
$$
\sup _{(\pi,\pi') \in \Sfrak} \frac{\W_{p}(\pi, \pi')}{\|\pi-\pi'\|^{\delta}_{\kappa}} \geq \sup_{\lambda \in(0,1)} \frac{\W_{p}(\pi_\lambda, \pi'_\lambda)}{\|\pi_\lambda-\pi'_\lambda\|^{\delta}_{\kappa}} \gtrsim \sup _{\lambda \in[0,1]} \lambda^{1 / p-\delta}=+\infty\,.
$$
The last equality is true because $p\delta>1$.
\end{proof}

\subsection{Proof of Theorem \ref{theo:regularitykerneltheo} \label{sec:proof_theo_regularitykerneltheo}}

We recall that, for $K \in \mathbb{N}^{*}$ and $\Omega \subseteq \R^{d}$, the space of mixtures of $K$ diracs located in $\Omega$ is defined by
\begin{equation}
\Sfrak_{K}(\Omega) := \left\{\sum_{i=1}^{K} a_i \delta_{\xbf_i}: a_{i} \in \R_{+}, \sum_{i=1}^{K} a_i = 1, \forall i \in \integ{K}, \xbf_i \in \Omega\right\}\,.
\end{equation}
The goal of this section is to prove the following theorem:
\regularitykerneltheo*

We will need the following lemma which states that if the kernel is regular at zero and that we can construct some vectors $\alphab,\betab$ that satisfy certain conditions then we have a constraint on the Hölder exponent $\delta$. 

\begin{restatable}{lemma}{superlemmmma}
\label{lemma:superlemmmma}
Consider a TI, PSD kernel $\kappa(\xbf,\ybf)=\kappa_{0}(\xbf-\ybf)$ on $\R^{d}$ such that $\kappa_{0}$ is $k$ times differentiable at $0$ with $k \in \mathbb{N}^{*}$. Let $M \in \mathbb{N}^{*}$ and define for $1 \leq s \leq k$ and $\alphab, \betab \in \R^{M}$ the function $c_s(\alphab,\betab):= \sum_{i,j=1}^{M} \beta_i \beta_j (\alpha_i-\alpha_j)^{s}$. 
Suppose that there exists $\alphab \in \R^{M}\setminus \{0\}$ with $\alpha_i \neq \alpha_j$ for $i \neq j$ and $\betab \in \R^{M}\setminus \{0\}$ with $\sum_{i=1}^{M}\beta_i=0$ such that
\begin{equation}
c_1(\alphab,\betab) = c_2(\alphab,\betab) =  \cdots = c_{k-1}(\alphab,\betab) = 0\,.
\end{equation}
Define $r(\betab):= \max\{ \# T_{+}(\betab), \#T_{-}(\betab)\}$ where $T_{+}(\betab):=\{i \in \integ{M}: \beta_i \geq 0\}$ and $T_{-}(\betab):=\{i \in \integ{M}: \beta_i < 0\}$. 

Consider $\Sfrak = \Sfrak_{r(\betab)}(\Omega)$ with $\Omega = B(\xbf_{0},R)$ 
where $\xbf_0 \in \R^{d},R >0$ are arbitrary. If $(\Sfrak,W_p)$ is $(\kappa,\delta)$-embeddable, where $\W_{p}$ is based on a norm $\|\cdot\|$ in $\R^{d}$ with $p \in [1,+\infty)$, then $\delta \leq 2/k$.
\end{restatable}

\begin{proof}
Recall that for a finite signed measure $\mu \in \Mcal(\R^{d})$ we have $\|\mu\|_{\kappa}^{2} =\int \int \kappa(\xbf,\ybf) \dr \mu(\xbf) \dr \mu(\ybf)$. Consider $M \in \integ{N}^{*}, \betab \in \R^{M}$ such that $\sum_{i=1}^{M} \beta_i =0$ and $\alphab \in \R^{M} \setminus \{0\}$ with $\alpha_i \neq \alpha_j$ when $i\neq j$. We define the measure
\begin{equation}
\mu_{\varepsilon} : = \sum_{i=1}^{M} \beta_i \delta_{\xbf_0+\varepsilon \alpha_i \ubf}\,,
\end{equation}
where $\ubf \in \R^{d}\setminus \{0\}$ and $0<\varepsilon < \frac{R}{\|\alphab\|_{\infty} \|\ubf\|_2}$ is sufficiently small to ensure that $\xbf_0+\varepsilon \alpha_i \ubf \in \Omega = B(\xbf_0,R)$. We define $T_{+}:=\{i \in \integ{M}: \beta_i \geq 0\}$ and $T_{-}:=\{i \in \integ{M}: \beta_i < 0\}$ such that $T_{-} \cup T_{+} =\integ{M}$ and $T_{-} \cap T_{+} = \emptyset$. We define also $\rho:=\sum_{i \in T_{+}} \beta_i = -\sum_{i \in T_{-}} \beta_i >0$ and
\begin{equation}
\pi_{\varepsilon} := \sum_{i \in T_{+}} \frac{\beta_i}{\rho} \delta_{\xbf_0 + \varepsilon \alpha_i \ubf} \text{ and } \pi'_{\varepsilon} :=  \sum_{i \in T_{-}} -\frac{\beta_i}{\rho} \delta_{\xbf_0+\varepsilon \alpha_i \ubf}\,.
\end{equation} 
We have that $\#T_{+}\leq r(\betab)$ and $\#T_{-}\leq r(\betab)$ by definition of $r(\betab)$. Since $\varepsilon$ is small enough we have that $\pi_{\varepsilon},\pi_{\varepsilon}' \in \Sfrak_{r(\betab)}(\Omega)$. Moreover $\mu_{\varepsilon} = \frac{1}{\rho}(\pi_{\varepsilon} - \pi'_{\varepsilon})$. Hence
\begin{equation}
\|\pi_{\varepsilon} - \pi'_{\varepsilon}\|_{\kappa}^{2} = \rho^{2} \|\mu_{\varepsilon}\|_{\kappa}^{2} = \rho^{2} \sum_{i,j=1}^{M} \beta_{i} \beta_{j} \kappa(\xbf_0+\varepsilon \alpha_i \ubf,\xbf_0+\varepsilon \alpha_j \ubf) = \rho^{2} \sum_{i,j=1}^{M} \beta_{i} \beta_{j}\kappa_{0}(\varepsilon(\alpha_i-\alpha_j)\ubf)\,.
\end{equation}
Since the kernel is $k$ times differentiable at $0$, the function $g: t \mapsto \kappa_{0}(t\ubf)$ is also $k$ times differentiable at $0$. A Taylor expansion yields
\begin{equation}
\kappa_{0}(\varepsilon \ubf) := g(\varepsilon) =   g(0) + \sum_{n=1}^{k} \frac{g^{(n)}(0)}{n!} \varepsilon^{n}+ o_{\varepsilon \rightarrow 0}(\varepsilon^{k})\,,
\end{equation}
hence
\begin{equation}
\begin{split}
\|\pi_{\varepsilon} - \pi'_{\varepsilon}\|_{\kappa}^{2} &= \rho^{2} \sum_{i,j=1}^{M} \beta_i \beta_j \left(  g(0) + \sum_{n=1}^{k} \frac{g^{(n)}(0)}{n!}(\alpha_i-\alpha_j)^{n} \varepsilon^{n}+ o_{\varepsilon \rightarrow 0}(\varepsilon^{k}) \right) \\
& = \rho^{2} \sum_{n=1}^{k} \left(\sum_{i,j=1}^{M} \beta_i \beta_j (\alpha_i-\alpha_j)^{n}\right) \varepsilon^{k}\frac{g^{(n)}(0)}{n!} + o_{\varepsilon \rightarrow 0}(\varepsilon^{k})\,,
\end{split}
\end{equation}
where we used that $\sum_{i,j=1}^{M} \beta_{i} \beta_{j} g(0) = 0$ since $(\sum_{i=1}^{M} \beta_i)^{2} =0$. With the notations of the Lemma we have
\begin{equation}
\|\pi_{\varepsilon} - \pi'_{\varepsilon}\|_{\kappa}^{2} =\rho^{2} \sum_{n=1}^{k} c_{n}(\alphab,\betab) \varepsilon^{k}\frac{g^{(n)}(0)}{n!} + o_{\varepsilon \rightarrow 0}(\varepsilon^{k})\,.
\end{equation}
Now, since by assumption we have
\begin{equation}
\label{eq:condition}
c_{1}(\alphab,\betab) =  \cdots = c_{k-1}(\alphab,\betab) = 0\,,
\end{equation}
we get
\begin{equation}
\|\pi_{\varepsilon} - \pi'_{\varepsilon}\|_{\kappa}^{2} = \rho^{2} c_{k}(\alphab,\betab) \varepsilon^{k} \frac{g^{(k)}(0)}{k!} + o_{\varepsilon \rightarrow 0}(\varepsilon^{k}) = O_{\varepsilon \rightarrow 0}(\varepsilon^{k})
\end{equation}
hence $\|\pi_{\varepsilon} - \pi'_{\varepsilon}\|_{\kappa} = O_{\varepsilon \rightarrow 0}(\varepsilon^{k/2})$. Moreover, defining for $i \in T_{+}$ $a_i = \beta_{i}/\rho$ and for $j \in T_{-}$ $b_j=-\beta_{j}/\rho$ we have
\begin{equation}
\W^{p}_{p}(\pi_{\varepsilon},\pi'_{\varepsilon}) = \min_{\gamma \in \Pi(\a,\b)} \sum_{i \in T_{+}, j \in T_{-}} \|\varepsilon \alpha_i \ubf -\varepsilon \alpha_j \ubf\|^{p} \gamma_{ij} = \varepsilon^{p} \|\ubf\|^{p} \min_{\gamma \in \Pi(\a,\b)} \sum_{i \in T_{+}, j \in T_{-}} |\alpha_i-\alpha_j|^{p}  \gamma_{ij}\,.
\end{equation}
Therefore
\begin{equation}
\W^{p}_{p}(\pi_{\varepsilon},\pi'_{\varepsilon})  \geq \left(\varepsilon \|\ubf\| \min_{i \in T_{+}, j \in T_{-}} |\alpha_i-\alpha_j|\right)^{p}\,,
\end{equation}
hence $\W_{p}(\pi_{\varepsilon},\pi'_{\varepsilon}) \geq \varepsilon \|\ubf\| \min_{i \in T_{+}, j \in T_{-}} |\alpha_i-\alpha_j|$. When $i \neq j$ we have $\alpha_i \neq \alpha_j$ by assumption. Since $T_{+} \cap T_{-} = \emptyset$ we have $\min_{i \in T_{+}, j \in T_{-}} |\alpha_i-\alpha_j| >0$. This discussion proves that, as soon as the condition \eqref{eq:condition} holds and $\delta > \frac{2}{k}$, we have
\begin{equation}
\sup_{(\pi,\pi') \in \Sfrak} \frac{\W_{p}(\pi,\pi')}{\|\pi-\pi'\|^{\delta}_{\kappa}} \geq \sup_{\varepsilon >0 } \frac{\W_{p}(\pi_{\varepsilon},\pi'_{\varepsilon})}{\|\pi_{\varepsilon}-\pi_{\varepsilon}'\|^{\delta}_{\kappa}} \gtrsim \sup_{\varepsilon >0 } \frac{\varepsilon}{\varepsilon^{\delta k /2}} = \sup_{\varepsilon >0 } \varepsilon^{1-\delta k /2} = +\infty\,.
\end{equation} 
Consequently, $(\Sfrak,W_p)$ is not $(\kappa,\delta)$-embeddable when $\delta > \frac{2}{k}$ which concludes the proof by contraposition. 
\end{proof}

The idea now is to find a couple $(\alphab,\betab)$ that satisfy the conditions $\sum_{i=1}^{M} \beta_i=0$ and $c_1(\alphab,\betab) = c_2(\alphab,\betab) =  \cdots = c_{k-1}(\alphab,\betab) = 0$. The following lemma show that it is possible to construct such vectors provided that $M=k+1$.

\begin{restatable}{lemma}{lemma_explicit}
\label{lemma:lemma_explicit}
Consider a TI, PSD kernel $\kappa(\xbf,\ybf)=\kappa_{0}(\xbf-\ybf)$ on $\R^{d}$ such that $\kappa_{0}$ is $k$ times differentiable at $0$ with $k \in \mathbb{N}^{*}$. With the same notations $c_{s}(\alphab,\betab)$ and $r(\betab)$ as in Lemma \ref{lemma:superlemmmma}, there exists $\alphab \in \R^{k+1} \setminus\{0\}$ with $\alpha_i \neq \alpha_j$ for $i \neq j$ and $\betab \in \R^{k+1}\setminus\{0\}$ with $\sum_{i=1}^{k+1}\beta_i=0$ such that
\begin{equation}
c_1(\alphab,\betab) = c_2(\alphab,\betab) =  \cdots = c_{k-1}(\alphab,\betab) = 0\,.
\end{equation}
Also if $k$ is odd then $\#T_{+}(\betab)=\#T_{-}(\betab) = \frac{k+1}{2}$ and if $k$ is even $\#T_{+}(\betab)=\frac{k}{2}+1$ and $\#T_{-}(\betab) = \frac{k}{2}$. Overall for any $k \in \mathbb{N}^{*}$ we have $r(\betab) \leq \lfloor \frac{k}{2} \rfloor+1$.
\end{restatable}

\begin{proof}
The condition $c_1(\alphab,\betab)= 0$ writes $\sum_{i,j=1}^{k+1} \beta_i \beta_j (\alpha_i-\alpha_j)= 0$ which is true for any $\alphab \in \R^{k+1}$ when $\betab \in \R^{k+1}$ satisfies $\sum_{i=1}^{k+1}\beta_i=0$. Indeed $\sum_{i,j=1}^{k+1} \beta_i \beta_j (\alpha_i-\alpha_j)= (\sum_{j=1}^{k+1} \beta_j) \sum_{i}^{k+1} \beta_i \alpha_i - (\sum_{i=1}^{k+1} \beta_i) \sum_{j}^{k+1} \beta_j \alpha_j = 0$. The condition $c_2(\alphab,\betab)= 0$ writes $\sum_{i,j=1}^{k+1} \beta_i \beta_j (\alpha_i-\alpha_j)^{2}= 0$. However $\sum_{i,j=1}^{k+1} \beta_i \beta_j (\alpha_i-\alpha_j)^{2}= \sum_{i,j=1}^{k+1} \beta_i \beta_j (\alpha_i^{2}+\alpha_j^{2}-2 \alpha_i \alpha_j)$. The term $\sum_{i,j=1}^{k+1} \beta_i \beta_j \alpha_i \alpha_j$ vanishes as soon as $\sum_{i=1}^{k+1} \beta_i \alpha_i =0$. The other terms $\sum_{i,j=1}^{k+1} \beta_i \beta_j \alpha_i^{2}$ and $\sum_{i,j=1}^{k+1} \beta_j \beta_i \alpha_j^{2}$ as soon as $\sum_{i=1}^{k+1} \beta_i =0$. With an immediate recurrence by using the Binomial formula we see that  $c_1(\alphab,\betab) = c_2(\alphab,\betab) =  \cdots = c_{k-1}(\alphab,\betab) = 0$ as soon as
\begin{equation}
\label{eq:eqconditions2}
\sum_{i=1}^{k+1} \beta_i = \sum_{i=1}^{k+1} \beta_i \alpha_i = \sum_{i=1}^{k+1} \beta_i \alpha_i^{2} = \cdots = \sum_{i=1}^{k+1} \beta_i \alpha_i^{k-1}=0\,.
\end{equation}

Define $\betab \in \R^{k+1}$ by for all $1 \leq i \leq k+1, \beta_i =  (-1)^{i-1}\binom{k}{i-1}$ and $\alphab \in \R^{k+1}$ by $\alpha_i =i$. Then the $\alpha_i$'s are pairwise distinct and
\begin{equation}
0 = \sum_{i=0}^{k} (-1)^{i}\binom{k}{i} = \sum_{i=1}^{k+1} (-1)^{i-1}\binom{k}{i-1}= \sum_{i=1}^{k+1} \beta_i\,.
\end{equation}
Then for any $1 \leq s \leq k-1$ we have
\begin{equation}
\sum_{i=1}^{k+1} \beta_i \alpha_i^{s}= \sum_{i=1}^{k+1} (-1)^{i-1}\binom{k}{i-1} i^{s}= \sum_{i=0}^{k} (-1)^{i}\binom{k}{i} (i+1)^{s}= \sum_{i=0}^{k} (-1)^{i}\binom{k}{i} \left(\sum_{l=0}^{s} \binom{s}{l} i^{l}\right) \,.
\end{equation}
Consquently
\begin{equation}
\sum_{i=1}^{k+1} \beta_i \alpha_i^{s} = \sum_{l=0}^{s} \binom{s}{l} \left(\sum_{i=0}^{k} (-1)^{i}\binom{k}{i} i^{l} \right)\,.
\end{equation}
But for $0 \leq l \leq s$ we have
\begin{equation}
\sum_{i=0}^{k} (-1)^{i}\binom{k}{i} i^{l} = \sum_{i=0}^{k} (-1)^{k-i}\binom{k}{k-i} (k-i)^{l} = \sum_{i=0}^{k} (-1)^{k-i}\binom{k}{i} (k-i)^{l} = (-1)^{k} \sum_{i=0}^{k} (-1)^{i}\binom{k}{i} (k-i)^{l}\,,
\end{equation} 
so $\sum_{i=0}^{k} (-1)^{i}\binom{k}{i} i^{l} =(-1)^{k} k! S_{2}(l,k)$ where $S_{2}(l,k)$ is the Stirling number of the second kind which is zero as soon as $l<k$. Since $l \leq s \leq k-1 <k$ by hypothesis we have that $\sum_{i=0}^{k} (-1)^{i}\binom{k}{i} i^{l} =0$ and thus  $\sum_{i=1}^{k+1} \beta_i \alpha_i^{s}=0$ for all $1 \leq s \leq k-1$ and $\sum_{i=1}^{k+1} \beta_i=0$. So this implies that $c_1(\alphab,\betab) = c_2(\alphab,\betab) =  \cdots = c_{k-1}(\alphab,\betab) = 0$. For such $\betab$ we have that $\#T_{+}(\betab)=\#T_{-}(\betab) = \frac{k+1}{2}$ for $k$ odd. If $k$ is even then $\#T_{+}(\betab)=\frac{k}{2}+1$ and $\#T_{-}(\betab) = \frac{k}{2}$.
\end{proof}
With this results we can now prove Theorem \ref{theo:regularitykerneltheo}.

\begin{proof}[Proof of Theorem \ref{theo:regularitykerneltheo}]
Define $(\alphab,\betab)$ as in Lemma \ref{lemma:lemma_explicit}. Then we have $c_1(\alphab,\betab) = c_2(\alphab,\betab) =  \cdots = c_{k-1}(\alphab,\betab) = 0$ and $r(\betab) \leq \lfloor \frac{k}{2} \rfloor+1$ which proves the theorem by using Lemma \ref{lemma:superlemmmma} with $M=k+1$.

\end{proof}

\subsection{Proof of Proposition \ref{lemma:bounding_wass} \label{proof:lemma:bounding_wass}}

Proposition \ref{lemma:bounding_wass} is an immediate corollary of the following variation of its statement:
\begin{proposition}
\label{lemma:bounding_wass_detailed}
Consider any $\pi,\pi' \in \P(\R^{d})$ having densities $f,g$ with respect to the Lebesgue measure, \ie\ $\pi = f \dr \xbf, \pi' = g \dr \xbf$. Denote $V_{d}=\pi^{d / 2}/ \Gamma(d / 2+1)$ the volume of the unit $d$ -dimensional unit sphere. 
\begin{enumerate}[label=(\roman*)]
\item Consider $1\leq p< \expo$. If $\moment_{\expo}[\pi], \moment_{\expo}[\pi']$ are finite then
\begin{equation*}
\W_p(\pi,\pi') \leq c_{d,p,\expo} (\moment^{\expo}_{\expo}[\pi]+\moment^{\expo}_{\expo}[\pi'])^{\frac{d+2p}{p(d+2\expo)}} \left(\int_{\R^{d}}|f(\xbf)-g(\xbf)|^{2} \dr \xbf \right)^{\frac{\expo-p}{(d+2\expo)p}}\,,
\end{equation*}
where $0<c_{d,p,\expo} \leq 2
 (\max\{V_d,1\})^{\frac{1}{2p}}$.
\item Consider  $1\leq p< \expo$. If $\max\{\moment_{\expo}[\pi], \moment_{\expo}[\pi']\} \leq \cte$ where $\cte >0$ then
\begin{equation*}
\W_p(\pi,\pi') \leq 
2c_{d,p,\expo}
 \cte^{\frac{r(d+2p)}{p(d+2\expo)}} \left(\int_{\R^{d}}|f(\xbf)-g(\xbf)|^{2} \dr \xbf \right)^{\frac{\expo-p}{(d+2\expo)p}}\,,
\end{equation*}
\item If $\pi,\pi'$ are supported in some Euclidean ball centered at $0$ of radius $\cte > 0$ then, for any $p \in [1,+\infty)$, 
\begin{equation}
\W_p(\pi,\pi') \leq 2^{\frac{p-1}{p}}
V_{d}^{\frac{1}{2p}}\cte^{\frac{2p+d}{2p}} \left(\int_{\R^{d}}|f(\xbf)-g(\xbf)|^{2} \dr \xbf \right)^{\frac{1}{2p}}
\end{equation}
\end{enumerate}
\end{proposition}

\begin{proof}
As a preliminary observe that by \citet[Theorem 6.15]{Villani} the Wasserstein distance is bounded by a weighted Total Variation distance:
\begin{equation*}
\W_p^{p}(\pi,\pi') \leq 2^{p-1}\int_{\R^{d}} \|\xbf\|_2^{p}\  \dr|\pi-\pi'|(\xbf)=2^{p-1}\int_{\R^{d}} \|\xbf\|_{2}^{p}\ |f(\xbf)-g(\xbf)|  \ \dr\xbf\,.
\end{equation*}
Given any $R > 0$, write $\int_{\R^{d}} \|\xbf\|_{2}^{p}\ |f(\xbf)-g(\xbf)|\  \dr\xbf = \int_{\|\xbf\|_2 \leq R} \|\xbf\|_{2}^{p}\ |f(\xbf)-g(\xbf)|\  \dr\xbf + \int_{\|\xbf\|_2 > R} \|\xbf\|_{2}^{p}\ |f(\xbf)-g(\xbf)|\  \dr\xbf$. By Cauchy-Schwarz inequality the first term of this decomposition is bounded as
\begin{align*}
\int_{\|\xbf\|_2 \leq R} \|\xbf\|_{2}^{p}\ |f(\xbf)-g(\xbf)|\  \dr\xbf 
&\leq 
\sqrt{\int_{\|\xbf\|_2 \leq R} \|\xbf\|_{2}^{2p}\ \dr \xbf} \sqrt{\int_{\|\xbf\|_2 \leq R}\ |f(\xbf)-g(\xbf)|^{2}\  \dr\xbf}
\leq 
C  \|f-g\|_{L_2(\R^{d})}
\end{align*}
where 
\[
C \coloneqq \sqrt{\int_{\|\xbf\|_2 \leq R} \|\xbf\|_{2}^{2p}\ \dr \xbf}
=
\sqrt{\int_{\|\ubf\|_2 \leq 1} \|R\ubf\|_{2}^{2p}\ R^{d}\dr \ubf}
=
\sqrt{R^{2p+d}\int_{\|\ubf\|_2 \leq 1} \|\ubf\|_{2}^{2p}\ \dr \ubf}
\leq R^{\frac{2p+d}{2}} \sqrt{V_{d}}\,.
\]

The second term is bounded as
\begin{align*}
\int_{\|\xbf\|_2 > R} \|\xbf\|_{2}^{p}\ |f(\xbf)-g(\xbf)|\  \dr\xbf 
&= \int_{\|\xbf\|_2 > R} \|\xbf\|_{2}^{p-\expo} \|\xbf\|_{2}^{\expo}\ |f(\xbf)-g(\xbf)|\  \dr\xbf \\
&\stackrel{r>p}{\leq} R^{p-\expo} \int_{\|\xbf\|_2 > R} \|\xbf\|_{2}^{\expo}\ |f(\xbf)-g(\xbf)|\  \dr\xbf 
\leq R^{p-\expo} (\moment^{\expo}_{\expo}[\pi]+\moment^{\expo}_{\expo}[\pi'])
\end{align*}
hence
\begin{equation}
\label{eq:withoutB}
\forall p \in [1, \expo), \ \W_p^{p}(\pi,\pi') \leq 
2^{p-1}
\Big(V_{d}^{1/2}
 \|f-g\|_{L_2(\R^{d})} R^{\frac{2p+d}{2}} + (\moment^{\expo}_{\expo}[\pi]+\moment^{\expo}_{\expo}[\pi']) R^{p-\expo}\Big)\,.
\end{equation}
We now have the ingredients to prove the three points.

For the first point, with $R := \left(\frac{\moment^{\expo}_{\expo}[\pi]+\moment^{\expo}_{\expo}[\pi']}{V_{d}^{1/2}
}\right)^{\frac{2}{d+2\expo}} \|f-g\|_{L_2(\R^{d})}^{-\frac{2}{d+2\expo}}$ we have $V_{d}^{1/2} \|f-g\|_{L_2(\R^{d})} R^{\frac{2p+d}{2}} = (\moment^{\expo}_{\expo}[\pi]+\moment^{\expo}_{\expo}[\pi']) R^{p-\expo}$ hence by \eqref{eq:withoutB} we have for each $p \in [1, \expo)$
\begin{align*}
\W_p^{p}(\pi,\pi') &\leq 
2^{p}
 (\moment^{\expo}_{\expo}[\pi]+\moment^{\expo}_{\expo}[\pi']) R^{p-\expo} = 
 2^{p}
   (\moment^{\expo}_{\expo}[\pi]+\moment^{\expo}_{\expo}[\pi']) \left(\frac{\moment_{\expo}[\pi]+\moment_{\expo}[\pi']}{V_{d}^{1/2}
}
\right)^{\frac{2(p-\expo)}{d+2\expo}} \|f-g\|_{L_2(\R^{d})}^{\frac{2(\expo-p)}{d+2\expo}}\,.
\end{align*}
Taking the $p$-th root yields the first claim once we check that $2^{p} (\moment^{\expo}_{\expo}[\pi]+\moment^{\expo}_{\expo}[\pi']) \left(\frac{\moment_{\expo}[\pi]+\moment_{\expo}[\pi']}{
V_{d}^{1/2}
}\right)^{\frac{2(p-\expo)}{d+2\expo}} \leq c_{d,p,r}^{p} (\moment^{\expo}_{\expo}[\pi]+\moment^{\expo}_{\expo}[\pi'])^{(d+2p)/(d+2r)}$ where $0<c_{d,p,\expo} \leq 
2
(\max\{V_d,1\})^{\frac{1}{2p}}$. Since
\[
2^{p} (\moment^{\expo}_{\expo}[\pi]+\moment^{\expo}_{\expo}[\pi']) \left(\frac{\moment_{\expo}[\pi]+\moment_{\expo}[\pi']}{
V_{d}^{1/2}
}\right)^{\frac{2(p-\expo)}{d+2\expo}}
= 2^{p} 
V_{d}^{\frac{\expo-p}{d+2\expo}} 
(\moment^{\expo}_{\expo}[\pi]+\moment^{\expo}_{\expo}[\pi'])^{\frac{d+2p}{d+2\expo}} 
\]
it is enough to bound
$c_{d,p,r} \coloneqq 
2 V_{d}^{\frac{\expo-p}{p(d+2\expo)}}\,.$

Indeed, since the function $\expo \mapsto \frac{\expo-p}{d+2\expo}=\frac{1}{2}-\frac{d+2p}{2(d+2r)}$ is monotonically increasing and $p<r < \infty$, we have $0 < \frac{\expo-p}{d+2r} < \lim_{\expo' \rightarrow +\infty} \frac{\expo'-p}{d+2\expo'}=\frac{1}{2}$, hence we have as claimed
\[
c_{d,p,r} = 2 V_d^{\frac{\expo-p}{p(d+2\expo)}} \leq 2(\max\{V_d,1\})^{\frac{\expo-p}{p(d+2\expo)}} \leq 2(\max\{V_d,1\})^{\frac{1}{2p}}.
\]

The second point is an immediate consequence of the first one. Since $1 \leq p < r$ we have $\frac{d+2p}{p(d+2r)} \leq \frac{1}{p} \leq 1$, hence using that $\max\{\moment_{\expo}[\pi], \moment_{\expo}[\pi']\} \leq \cte$ we get
\[
\left(\moment_{\expo}^{\expo}[\pi]+ \moment_{\expo}^{\expo}[\pi']\right)^{\frac{d+2p}{p(d+2\expo)}} \leq 2^{\frac{d+2p}{p(d+2\expo)}} M^{\frac{r(d+2p)}{p(d+2\expo)}} \leq 2M^{\frac{r(d+2p)}{p(d+2\expo)}}\,.
\]

For the last point we have $\forall \expo >1, \max\{\moment_{\expo}[\pi], \moment_{\expo}[\pi']\} \leq \cte$ and thus \eqref{eq:withoutB} gives for any choice of $R>0$:
\begin{equation}
\label{eq:eqforallr}
\forall \expo >1, \forall p \in [1, \expo), \ \W_p^{p}(\pi,\pi') \leq 2^{p-1} 
\Big(V_{d}^{1/2}
 \|f-g\|_{L_2(\R^{d})} R^{\frac{2p+d}{2}} + 2 \Big(\frac{\cte}{R}\Big)^{\expo} R^{p}\Big)\,.
\end{equation}
Consider any $R  > \cte$. We can take the limit as $\expo \rightarrow +\infty$ in \eqref{eq:eqforallr} which gives
\begin{equation*}
\forall R  > \cte, \forall p \in [1, +\infty), \ \W_p^{p}(\pi,\pi') \leq 2^{p-1}
\Big(V_{d}^{1/2}\|f-g\|_{L_2(\R^{d})} R^{\frac{2p+d}{2}}\Big)\,.
\end{equation*}
since $\lim\limits_{\expo\rightarrow +\infty} (\frac{\cte}{R})^{\expo} = 0$. Since this is true for any $R > \cte$ we can conclude that
\begin{equation*}
\forall p \in [1, +\infty), \ \W^{p}_p(\pi,\pi') \leq 2^{p-1}
V_{d}^{1/2}
\cte^{\frac{2p+d}{2}}\left(\int_{\R^{d}}|f(\xbf)-g(\xbf)|^2 \dr \xbf\right)^{\frac{1}{2}}\,.
\end{equation*} 
Taking the $p$-th root yields the conclusion.

\end{proof}

\subsection{Proof of Theorem \ref{theo:regularcase_rkhs} and \ref{theo:regularcase} \label{proof:theo:regularcase}}

We first prove the following result:
\regularcaserkhs*
\begin{proof}
Take any $\pi,\pi' \in \Sfrak_{B, \cte, \expo, \kappa}$ and recall that this implies notably that $\moment_{\expo}[\pi] \leq \cte$ (and similarly for $\pi'$). By Proposition
~\ref{lemma:bounding_wass_detailed} we have, with $C_{1} := 2c_{d,p,\expo} M^{\frac{r(d+2p)}{p(d+2r)}}$:

\begin{equation}
\W_p(\pi,\pi') \leq C_1 \left(\int |f(\xbf)-g(\xbf)|^2 \dr \xbf\right)^{\frac{\expo-p}{p(d+2\expo)}}\,.
\end{equation}
Since $\kappa_0 \in L_1(\R^{d})$ it has a Fourier transform $\widehat{\kappa_0}$, which is non-negative by Bochner's theorem. Consequently:
\begin{equation}
\begin{split}
\W_p(\pi,\pi') &\stackrel{\star}{\leq} C_1\left((2\pi)^{-d} \int |\hat{f}(\w)-\hat{g}(\w)|^2 \dr \w\right)^{\frac{\expo-p}{p(d+2\expo)}} = (2\pi)^{\frac{-d(\expo-p)}{p(d+2\expo)}}C_1\left(\int |\hat{f}(\w)-\hat{g}(\w)|^2 \dr \w\right)^{\frac{\expo-p}{p(d+2\expo)}}  \\
&= (2\pi)^{\frac{-d(\expo-p)}{p(d+2\expo)}}C_1\left(\int \frac{|\hat{f}(\w)-\hat{g}(\w)|}{\sqrt{\widehat{\kappa_0}(\w)}}\sqrt{\widehat{\kappa_0}(\w)}|\hat{f}(\w)-\hat{g}(\w)| \dr \w\right)^{\frac{\expo-p}{p(d+2\expo)}} \\
&\stackrel{\star\star}{\leq} (2\pi)^{\frac{-d(\expo-p)}{p(d+2\expo)}}C_1 \left(\int \frac{|\hat{f}(\w)-\hat{g}(\w)|^2}{\widehat{\kappa_0}(\w)} \dr \w \right)^{\frac{\expo-p}{2p(d+2\expo)}}\left(\int \widehat{\kappa_0}(\w) |\hat{f}(\w)-\hat{g}(\w)|^2 \dr \w \right)^{\frac{\expo-p}{2p(d+2\expo)}} \\
&\stackrel{\star\star\star}{\leq} (2\pi)^{\frac{-d(\expo-p)}{p(d+2\expo)}}C_1\left(\int \frac{|\hat{f}(\w)-\hat{g}(\w)|^2}{\widehat{\kappa_0}(\w)} \dr \w \right)^{\frac{\expo-p}{2p(d+2\expo)}} (2\pi)^{\frac{d(\expo-p)}{2p(d+2\expo)}}\|\pi-\pi'\|_{\kappa}^{\frac{\expo-p}{p(d+2\expo)}} \\
&=(2\pi)^{\frac{-d(\expo-p)}{2p(d+2\expo)}}C_1\left(\int \frac{|\hat{f}(\w)-\hat{g}(\w)|^2}{\widehat{\kappa_0}(\w)} \dr \w \right)^{\frac{\expo-p}{2p(d+2\expo)}} \|\pi-\pi'\|_{\kappa}^{\frac{\expo-p}{p(d+2\expo)}}\\
&= C_1\|f-g\|_{\Hcal_{\kappa}}^{\frac{\expo-p}{p(d+2\expo)}} \|\pi-\pi'\|_{\kappa}^{\frac{\expo-p}{p(d+2\expo)}}\,,
\end{split}
\end{equation}
where in $(\star)$ we used the Plancherel formula, in $(\star\star)$ we used the Cauchy–Schwarz inequality and in $(\star\star\star)$ we relied on Lemma \ref{lemma:mmdform} whose proof is postponed below. In the last step we used $(\int \frac{|\hat{f}(\w)-\hat{g}(\w)|^2}{\widehat{\kappa_0}(\w)} \dr \w)^{1/2} = (2\pi)^{d/2}\|f-g\|_{\Hcal_{\kappa}}$. We used Theorem 10.12 in \citealt{Wendland} where we adapted the conventions on the Fourier transform. We can apply this theorem since $\kappa_0$ is continuous (by hypothesis), and its Fourier transform $\widehat{\kappa_0} > 0$ thus $\kappa_0$ is positive definite \citep[Corollary 6.9]{Wendland}. Finally $\max\{\|f\|_{\Hcal_{\kappa}}, \|g\|_{\Hcal_{\kappa}}\} \leq B$ by hypothesis. Thus $\|f-g\|^{\frac{\expo-p}{(d+2\expo)p}}_{\Hcal_{\kappa}} \leq (B+B)^{\frac{\expo-p}{(d+2\expo)p}} \leq 2B^{\frac{\expo-p}{(d+2\expo)p}}$ since $\frac{\expo-p}{(d+2\expo)p} \leq 1$. This concludes the proof with $C := 2B^{\frac{\expo-p}{(d+2\expo)p}} C_1= 4c_{d,p,r} B^{\frac{\expo-p}{(d+2\expo)p}}M^{\frac{r(d+2p)}{p(d+2r)}}.$

\end{proof}

As a consequence we have the theorem:
\regularcase*
\begin{proof}
This is a direct consequence of Theorem~\ref{theo:regularcase_rkhs} once we establish that, under the assumptions on $\kappa$, we have $\Sfrak_{ B, \cte, \expo,s} \subseteq \Sfrak_{C B, \cte, \expo, \kappa}$ where $C = C(d,s, \kappa)$ is the constant from Lemma \ref{lemma:general_bound} below. Indeed, consider
 $\pi = f \dr \xbf \in \Sfrak_{ B, \cte, \expo,s}$. By hypothesis we have $\moment_\expo[\pi] \leq \cte$ and $\|f\|_{H^{s}(\R^d)}\leq B$. With the hypothesis on the kernel $\kappa$ we can use Lemma \ref{lemma:general_bound} below to prove that there is a constant $C = C(d,s, \kappa)$ such that $\|f\|_{\Hcal_{\kappa}} \leq C \|f\|_{H^{s}(\R^{d})} \leq C B$, which shows that $\pi \in \Sfrak_{C B, \cte, \expo, \kappa}$ as claimed. Thus, by Theorem \ref{theo:regularcase_rkhs}, with $c_{d,p,\expo}$ the constant defined in Proposition \ref{lemma:bounding_wass} we have:
\begin{equation}
\forall \pi,\pi' \in \Sfrak_{B, \cte, \expo,s}, \W_p(\pi,\pi') \leq 
4c_{d,p,\expo}
2(C(d,s,\kappa)B)^{\frac{\expo-p}{(d+2\expo)p}} M^{\frac{(d+2p)\expo}{(d+2\expo)p}}
\|\pi-\pi'\|_{\kappa}^{\frac{\expo-p}{p(d+2\expo)}}\,,
\end{equation}
which concludes the proof.
\end{proof}

\begin{lemma}
\label{lemma:general_bound}
Let $\kappa(\xbf,\ybf)=\kappa_0(\xbf-\ybf)$ be a TI, PSD kernel on $\R^{d}$ with $\kappa_{0} \in L_1(\R^{d})$ such that 
$\widehat{\kappa_{0}}(\w) > 0$ for every $\w$ and $\frac{1}{\widehat{\kappa}_0(\w)}=O(\|\w\|_2^{s_{\kappa}})$ as $\|\w\|_2 \rightarrow +\infty$ for some $s_{\kappa} \in \R_{+}$. For any $s\geq s_{\kappa}/2$, there exists a constant $C = C(d, s, \kappa) > 0$ such that for every $f\in H^{s}(\R^{d})$ we have
\begin{equation}
\|f\|_{\Hcal_{\kappa}} \leq C \|f\|_{H^{s}(\R^{d})}\,.
\end{equation}

\end{lemma}
\begin{proof}
Given any $R>0$ we write $\int_{\R^d} \frac{|\hat{f}(\w)|^{2}}{\widehat{\kappa}_0(\w)} \dr \w=\int_{\|\w\|_2
\leq R} \frac{|\hat{f}(\w)|^{2}}{\widehat{\kappa}_0(\w)} \dr \w+\int_{\|\w\|_2 > R} \frac{|\hat{f}(\w)|^{2}}{\widehat{\kappa}_0(\w)} \dr \w$ and use the shorthand $I_{\|\w\|_2 \leq R}$ and $I_{\|\w\|_2 > R}$ for the two terms. Since $\kappa_{0} \in L_1(\R^{d})$ the Fourier transform $\widehat{\kappa}_0$ is continuous. It is also positive and thus and the term $I_{\|\w\|_2< R}$ can be bounded as 
\begin{equation}
I_{\|\w\|_2\leq R} \leq \left(\sup_{\|\w\|_2 \leq R} \widehat{\kappa_{0}}(\w)^{-1}\right) \int_{\|\w\|_2 \leq R} |\hat{f}(\w)|^{2} \dr \w \leq \left(\sup_{\|\w\|_2 \leq R} \widehat{\kappa}_{0}(\w)^{-1}\right) \|f\|^{2}_{H^{s}(\R^{d})}\,.
\end{equation}

Now consider $I_{\|\w\|_2> R}$ and take $s \geq \frac{s_{\kappa}}{2}$. We have:
\begin{equation}
\begin{split}
\int_{\|\w\|_2 > R}|\hat{f}(\w)|^{2}\frac{1}{\widehat{\kappa_0}(\w)} \dr \w&=\int_{\|\w\|_2> R}(1+\|\w\|_2^{2})^{s}|\hat{f}(\w)|^{2}(1+\|\w\|_2^{2})^{-s} \frac{1}{\widehat{\kappa_0}(\w)} \dr \w\\
&\leq \sup_{\|\w\|_2> R}\left( \frac{(1+\|\w\|_2^{2})^{-s}}{\widehat{\kappa_0}(\w)}\right) \int\limits_{\|\w\|_2> R}(1+\|\w\|_2^{2})^{s}|\hat{f}(\w)|^{2} \dr \w \\
&\leq \sup_{\|\w\|_2> R}\left( \frac{(1+\|\w\|_2^{2})^{-s}}{\widehat{\kappa_0}(\w)}\right)\|f\|^2_{H^{s}(\R^{d})} 
\end{split}
\end{equation} 
By hypothesis $\frac{\|\w\|_2^{-2s}}{\widehat{\kappa_0}(\w)}=O_{\|\w\|_2 \rightarrow +\infty}(\frac{1}{\|\w\|_2^{2s-s_{\kappa}}})$. Since $s \geq \frac{s_{\kappa}}{2}$ we have $2s-s_{\kappa} \geq 0$ thus the quantity $\underset{\|\w\|_2> R}{\sup}\left( \frac{(1+\|\w\|_2^{2})^{-s}}{\widehat{\kappa_0}(\w)}\right)$ is finite. The previous reasoning gives, for any $R > 0$, 
\begin{equation}
\|f\|_{\Hcal_\kappa}^2 = (2\pi)^{-d}\int_{\R^d} \frac{|\hat{f}(\w)|^{2}}{\widehat{\kappa_0}(\w)} \dr \w \leq (2\pi)^{-d}\left(\sup_{\|\w\|_2 \leq R} \frac{1}{\widehat{\kappa_{0}}(\w)}+\underset{\|\w\|_2> R}{\sup} \frac{(1+\|\w\|_2^{2})^{-s}}{\widehat{\kappa_0}(\w)}\right)\|f\|^2_{H^{s}(\R^{d})}\,.
\end{equation}

The infimum over $R > 0$ yields a constant $C(d,s,\kappa)$ such that $\|f\|_{\Hcal_\kappa} \leq C(d,s,\kappa) \|f\|_{H^{s}(\R^{d})}$. 

\end{proof}

\begin{restatable}{lemma}{}
\label{lemma:mmdform}
Let $\kappa(\xbf,\ybf)=\kappa_0(\xbf-\ybf)$ be a TI, PSD kernel on $\R^{d} \times \R^{d}$ where $\kappa_0 \in L_1(\R^{d})$.  Then for $\pi,\pi' \in \P(\R^{d})$ we have the formula
\begin{equation}
\|\pi-\pi'\|^{2}_{\kappa}= (2\pi)^{-d}\int \widehat{\kappa_0}(\w)|\widehat{\pi}(\w)-\widehat{\pi'}(\w)|^{2} \dr \w\,.
\end{equation}
In particular when $\pi,\pi'$ have densities $f,g$ with respect to the Lebesgue measure we have
\begin{equation}
\|\pi-\pi'\|^{2}_{\kappa}= (2\pi)^{-d}\int \widehat{\kappa_0}(\w)|\hat{f}(\w)-\hat{g}(\w)|^{2} \dr \w\,.
\end{equation}
\end{restatable}
\begin{proof}
This result can be found in \citet{Sriperumbudur} but we rewrite the proof for completeness. Since $\kappa_{0}$ is a continuous PSD function and $\kappa_0 \in L_1(\R^{d})$ then by Bochner's theorem $\widehat{\kappa_0} \geq 0$. So $\kappa_{0}$ is even ($\kappa$ is symmetric), integrable, continuous (in particular at $0$) and has nonnegative Fourier transform so $\widehat{\kappa_0} \in L_1(\R^{d})$ \citep{SteinWeiss2016}. Then by Fourier inversion theorem
\begin{equation}
\label{eq:inversion}
\forall \xbf \in \R^{d}, \ \kappa_0(\xbf)=(2\pi)^{-d}\int e^{i\w^{\top}\xbf} \widehat{\kappa_{0}}(\w) \dr \w\,.
\end{equation}
In the following we define the measure $\Lambda$ by $\dr \Lambda(\w):=(2\pi)^{-d} \widehat{\kappa_{0}}(\w)\dr \w$ (which is a non-negative finite measure thanks to Bochner's theorem).
We have:
\begin{equation}
\begin{split}
\|\pi-\pi'\|^{2}_{\kappa}&=\int \int \kappa_0(\xbf-\ybf) \dr (\pi-\pi')(\xbf)\dr (\pi-\pi')(\ybf)  \\
&\stackrel{\star}{=} \int \int \int e^{i\w^{\top}(\xbf-\ybf)}\dr \Lambda(\w)  \dr (\pi-\pi')(\xbf) \dr (\pi-\pi')(\ybf)\\
&=\int \left(\int  e^{i\w^{\top}\xbf} \dr (\pi-\pi')(\xbf)\right)\left(\int  e^{-i\w^{\top}\ybf} \dr (\pi-\pi')(\ybf)\right) \dr \Lambda(\w) \\
&= \int (\widehat{\pi}(\w)-\widehat{\pi'}(\w))\overline{(\widehat{\pi}(\w)-\widehat{\pi'}(\w))}\dr \Lambda(\w) =\int |\widehat{\pi}(\w)-\widehat{\pi'}(\w)|^{2}\dr \Lambda(\w)\\
&=(2\pi)^{-d}\int \widehat{\kappa_0}(\w)|\widehat{\pi}(\w)-\widehat{\pi'}(\w)|^{2} \dr \w\,,
\end{split}
\end{equation}
where in $(\star)$ we used \eqref{eq:inversion} and Fubini theorem.
\end{proof}

\subsection{The Compactly Supported Case \label{proof:compactly_supported_case}}
We will prove the following result:
\begin{lemma}
\label{lemma:compactly_supported_case}
Let $\kappa(\xbf,\ybf)=\kappa_0(\xbf-\ybf)$ be a TI, PSD kernel on $\R^{d}$ with $\kappa_{0} \in L_1(\R^{d})$ such that 
$\widehat{\kappa}_{0}(\w) > 0$ for every $\w$ and $\frac{1}{\widehat{\kappa}_0(\w)}=O(\|\w\|_2^{s_{\kappa}})$ as $\|\w\|_2 \rightarrow +\infty$ for some $s_{\kappa} \in \R_{+}$. Consider $0<\cte, B<+\infty$, $s \geq s_{\kappa}/2$ and the following model set
\begin{equation}
\Sfrak_{B, \cte,s} :=  \left\{\pi \in \P(\R^{d}): \ \pi= f\dr \xbf,\ \|f\|_{H^s(\R^d)}\leq B \text{ and } \ \supp(\pi) \subseteq B(0,\cte)\right\}\,,
\end{equation}
where $B(0,\cte)$ is the Euclidean ball centered at zero with radius $\cte$. For any $p \in [1,+\infty)$, there exists a constant $C = C(d,p,\cte,B, \kappa, s) > 0$ such that
\begin{equation}
\forall \pi,\pi' \in \Sfrak_{B, \cte,s}, \W_p(\pi,\pi') \leq C \|\pi-\pi'\|_{\kappa}^{\frac{1}{2p}}\,.
\end{equation}
\end{lemma}
\begin{proof}
By the third  point of Proposition \ref{lemma:bounding_wass_detailed} there is a constant $C = C(d,p,\cte) > 0$ such that
\begin{equation}
\W_p(\pi,\pi') \leq C \left(\int_{\R^{d}}|f(\xbf)-g(\xbf)|^{2} \dr \xbf \right)^{\frac{1}{2p}}
\end{equation}
for every $\pi,\pi' \in  \Sfrak_{B, \cte,s}$. Then, with the same strategy as in the proof of Theorem \ref{theo:regularcase} we have
\begin{equation}
\W_p(\pi,\pi') \leq C_1\left(\int \frac{|\widehat{f}(\w)-\widehat{g}(\w)|^2}{\widehat{\kappa_0}(\w)} \dr \w\right)^{\frac{1}{4p}} \|\pi-\pi'\|_{\kappa}^{\frac{1}{2p}}\,.
\end{equation}
for some constant $C_1 > 0$ which depends on $d,p,\cte$. By Lemma \ref{lemma:general_bound} there exists a constant $C_2 = C_2(\kappa,s,B,d)$ such that $\int \frac{|\widehat{f}(\w)-\widehat{g}(\w)|^2}{\widehat{\kappa_0}(\w)} \dr \w \leq C_2$. This concludes the proof.
\end{proof}

\subsection{Proof of Lemma \ref{lemma:characterization_mmd}, Proposition \ref{prop:regu_wass} and Theorem \ref{theo:maintheo_noncompact} \label{proof:prop:regu_wass}}

\characterizationmmd*
\begin{proof}
We first prove that the kernel in this proposition defines a TI, PSD kernel. It is clearly translation invariant by definition and symmetric since the convolution of even functions is even thus $\kappa_{0}$ is even. Also $\kappa_{0}$ is continuous and bounded since $\alpha$ is continuous and bounded. Since $\alpha$ is even its Fourier transform is real-valued  hence $\widehat{\kappa}_{0} = \hat{\alpha}^{2} = |\hat{\alpha}|^{2} \geq 0$ so the Fourier transform of $\kappa_0$ is non negative. Finally $\kappa_{0} \in L_1(\R^{d})$ as the convolution of two integrable functions. Using Bochner's theorem (see Theorem \ref{theo:bochner}) shows that the kernel $\kappa$ is a TI, PSD kernel. Moreover:
\begin{equation}
\|\alpha*\pi-\alpha*\pi\|^{2}_{L_2(\R^{d})} = \int |\alpha*\pi(\xbf)-\alpha*\pi'(\xbf)|^{2} \dr \xbf \stackrel{\star}{=} (2\pi)^{-d} \int  |\widehat{\alpha*\pi}(\w)-\widehat{\alpha*\pi'}(\w)|^{2} \dr \w\,,
\end{equation}
where in $(\star)$ we used Plancherel formula which is possible since $\alpha * \pi \in L_2(\R^{d})$ because $\alpha \in L_2(\R^{d})$ (same for $\alpha * \pi'$). So using that $\widehat{\alpha * \pi}=\widehat{\alpha} \times \widehat{\pi}$ ($\alpha$ is a probability density function and $\pi$ a probability distribution):
\begin{equation}
\|\alpha*\pi-\alpha*\pi\|^{2}_{L_2(\R^{d})} = (2\pi)^{-d} \int |\widehat{\alpha}(\w) \widehat{\pi}(\w) - \widehat{\alpha}(\w) \widehat{\pi'}(\w)|^{2} \dr \w = (2\pi)^{-d} \int |\widehat{\alpha}(\w)|^{2} |\widehat{\pi}(\w)-\widehat{\pi'}(\w)|^{2} \dr \w.
\end{equation}
Finally, since $\widehat{\kappa_{0}}=|\widehat{\alpha}|^{2}$ we get
\begin{equation}
\|\alpha*\pi-\alpha*\pi\|^{2}_{L_2(\R^{d})} = (2\pi)^{-d} \int \widehat{\kappa_0}(\w) |\widehat{\pi}(\w)-\widehat{\pi'}(\w)|^{2} \dr \w \stackrel{\star \star}{=} \|\pi-\pi'\|_{\kappa}^{2}\,,
\end{equation}
where in $(\star \star)$ we used Lemma \ref{lemma:mmdform}. This concludes the proof.
\end{proof}

\reguwass*
\begin{proof}
In order to prove the proposition we will apply the first point of Proposition \ref{lemma:bounding_wass_detailed} with $\pi_{\alpha}$ and $\pi_{\alpha}'$ that admit the densities $f=\alpha*\pi$ and $g=\alpha*\pi'$ and thus the term $\|f-g\|_{L_{2}(\R^{d})}$ in Proposition \ref{lemma:bounding_wass_detailed} becomes $\|f-g\|_{L_{2}} = \|\alpha*\pi-\alpha*\pi'\|_{L_2(\R^{d})}$. To apply Proposition \ref{lemma:bounding_wass_detailed} we need to show that $\pi_\alpha,\pi'_\alpha$ have $\expo$-finite moments which will be true by using that $\pi,\pi'$ and $\alpha$ have $\expo$-finite moments. Indeed
\begin{equation}
\begin{split}
\mathbb{E}_{\xbf \sim \pi_\alpha}\|\xbf\|_2^{\expo}&=\int \|\xbf\|_2^{\expo} (\alpha*\pi)(\xbf) \dr \xbf= \int \|\xbf\|_2^{\expo} \left(\int \alpha(\xbf-\ybf) \dr \pi(\ybf)\right) \dr \xbf \\
&\stackrel{\star}{=} \int \int \|\xbf\|_2^{\expo} \alpha(\xbf-\ybf) \dr \xbf \dr \pi(\ybf) = \int \left( \int \|\xbf\|_2^{\expo} \alpha(\xbf-\ybf) \dr \xbf \right) \dr \pi(\ybf)\,,
\end{split}
\end{equation}
where in ($\star$) we used the Fubini theorem ($\alpha$ is non-negative). Moreover, for any $\ybf \in \R^{d}$,
\begin{equation}
\int \|\xbf\|_2^{\expo} \alpha(\xbf-\ybf) \dr \xbf = \int \|\ybf+\zbf\|_2^{\expo} \alpha(\zbf) \dr \zbf \leq 2^{\expo-1} \left(  \|\ybf\|_2^{\expo} \int \alpha(\zbf) \dr \zbf+ \int \|\zbf\|_2^{\expo} \alpha(\zbf) \dr \zbf\right)\,,
\end{equation}
where in the last inequality we used $\|\zbf+\ybf\|_2^{\expo} \leq 2^{\expo-1}(\|\zbf\|_2^{\expo} +\|\ybf\|_2^{\expo})$. Moreover since $\int \alpha(\zbf) \dr \zbf=1$ we have:
\begin{equation}
\mathbb{E}_{\xbf \sim \pi_\alpha}\|\xbf\|_2^{\expo} \leq 2^{\expo-1}\left(\int \|\ybf\|_2^{\expo} \dr \pi(\ybf)+ \int \|\zbf\|_2^{\expo} \alpha(\zbf)\dr \zbf \right)<+\infty
\end{equation}
So by using the first point of Proposition \ref{lemma:bounding_wass_detailed} we have
\begin{equation}
\W_p(\pi_{\alpha},\pi_{\alpha}') \leq C_{d,p, \expo}\left(\mathbb{E}_{\xbf \sim \pi_\alpha}\|\xbf\|_2^{\expo}+\mathbb{E}_{\ybf \sim \pi'_\alpha}\|\ybf\|_2^{\expo}\right)^\frac{2p+d}{(d+2\expo)p}\|\alpha*\pi-\alpha*\pi'\|_{L_2(\R^{d})}^{\frac{2(\expo-p)}{(d+2\expo)p}}\,,
\end{equation}
for some constant $C_{d,p,\expo} >0$. Finally, to relate the term $\|\alpha*\pi-\alpha*\pi'\|_{L_2(\R^{d})}$ with the MMD we use the Lemma \ref{lemma:characterization_mmd}.

\end{proof}

Finally we can prove the following theorem:
\maintheononcompact*

\begin{proof}
With the notations of the theorem we have, by Proposition \ref{prop:regu_wass},
\begin{equation}
\W_p(\pi_\alpha,\pi'_\alpha) \leq C_{d,\expo,p} \left(\E_{\xbf \sim \pi_\alpha}[\|\xbf\|_2^{\expo}]+\E_{\ybf \sim \pi_\alpha'}[\|\ybf\|_2^{\expo}]\right)^\frac{2p+d}{(d+2\expo)p}\|\pi-\pi'\|_{\kappa}^{\frac{2(\expo-p)}{(d+2\expo)p}}\,,
\end{equation}
where $C_{d,\expo,p}$ is defined in Proposition \ref{prop:regu_wass}. We can control both terms $\E_{\xbf \sim \pi_\alpha}[\|\xbf\|_2^{\expo}],\E_{\ybf \sim \pi_\alpha'}[\|\ybf\|_2^{\expo}]$ as in the proof of Proposition \ref{prop:regu_wass} so that
\begin{equation}
\E_{\xbf \sim \pi_\alpha}[\|\xbf\|_2^{\expo}] \leq 2^{\expo}\left(\int \|\ybf\|_2^{\expo} \dr \pi(\ybf)+ \int \|\zbf\|_2^{\expo} \alpha(\zbf)\dr \zbf \right) \leq 2^{\expo}(\cte^{\expo}+\int \|\zbf\|_2^{\expo} \alpha(\zbf)\dr \zbf)\,,
\end{equation}
since $\pi \in \Sfrak$ (and in the same way for $\E_{\ybf \sim \pi'_\alpha}[\|\ybf\|_2^{\expo}]$). Consequently:
\begin{equation}
\W_p(\pi_\alpha,\pi'_\alpha) \leq C_{d,\expo,p}2^{(\expo+1)(\frac{2p+d}{(d+2\expo)p})}(\cte^{\expo}+\int \|\zbf\|_2^{\expo} \alpha(\zbf)\dr \zbf)^{\frac{2p+d}{(d+2s)p}}\|\pi-\pi'\|_{\kappa}^{\frac{2(\expo-p)}{(d+2\expo)p}}\,.
\end{equation}
By defining $C'_{d,\expo,p}=2^{(\expo+1)\frac{2p+d}{(d+2\expo)p}}C_{d,\expo,p}$ and using Lemma \ref{lemma:regwass_compared} we have
\begin{equation}
\W_p(\pi_{\alpha},\pi'_{\alpha}) \leq C'_{d,\expo,p}(\cte^{\expo}+\int \|\zbf\|_2^{\expo} \alpha(\zbf)\dr \zbf)^{\frac{2p+d}{(d+2\expo)p}} + 2 \left(\int \|\zbf\|_2^{p} \alpha(\zbf) \dr \zbf\right)^{1/p}\,,
\end{equation}
which concludes the proof.
\end{proof}

\section{Proofs of Section \ref{sec:wasserstein_learnability}}
\label{section:wasslearn}

\subsection{Proof of Lemma \ref{lemma:rosaco2} \label{proof:lemma:rosaco2}}

\rosaco*
\begin{proof}
The proof in mainly taken from \citet{NIPS2012_c54e7837} but we rewrite it in our context. Considering the admissible coupling $\gamma=(id \times P_S)\#\pi \in \Pi(\pi,P_S\#\pi)$, then \begin{equation}
\W^{p}_p(\pi,P_S\#\pi) \leq \int D^{p}(\xbf,\ybf) \dr \gamma(\xbf,\ybf)=\int D^{p}(\xbf,P_S(\xbf)) \dr \pi(\xbf)=\E_{\xbf \sim \pi}[D(\xbf,P_S(\xbf))^{p}]\,.
\end{equation}
Conversely, if $\gamma^{*}$ is an optimal coupling for $\W_p(\pi,P_S\#\pi)$ then for all $(\xbf,\ybf) \in \supp(\gamma^{*})$ we have that $\ybf \in \supp(P_S\#\pi)$ by definition of a coupling which means that $\ybf \in S$  and so by hypothesis $D^{p}(\xbf,\ybf) \geq D^{p}(\xbf,P_S(\xbf))$. Therefore,
\begin{equation}
\W^{p}_p(\pi,P_S\#\pi) =\int D^{p}(\xbf,\ybf) \dr \gamma^{*}(\xbf,\ybf) \geq \int D^{p}(\xbf,P_S(\xbf)) \dr \gamma^{*}(\xbf,\ybf) =\int D^{p}(\xbf,P_S(\xbf)) \dr \pi(\xbf) \,.
\end{equation}
Hence $\W^{p}_p(\pi,P_S\#\pi) \geq \E_{\xbf \sim \pi}[D(\xbf,P_S(\xbf))^{p}]$. The last inequality can be proved in the same way by considering an optimal coupling $\gamma^{*}$ between $\pi$ and $\nu$:
\begin{equation}
\begin{split}
\W^{p}_p(\pi,\nu) &=\int D^{p}(\xbf,\ybf) \dr \gamma^{*}(\xbf,\ybf) \stackrel{\supp(\nu) \subseteq S}{\geq} \int D^{p}(\xbf,P_S(\xbf)) \dr \gamma^{*}(\xbf,\ybf) \\
&=\int D^{p}(\xbf,P_S(\xbf)) \dr \pi(\xbf)=\E_{\xbf \sim \pi}[D(\xbf,P_S(\xbf))^{p}] = \W^{p}_p(\pi,P_S\#\pi)\,.
\end{split}
\end{equation}
\end{proof}

\section{Proofs of Section \ref{sec:compress_section}}
\label{section:appendixcompress}

\subsection{Proof of Proposition \ref{prop:holderlripandiopareequivalent} \label{app:proofseccompress_section}}
We recall the result here:

\holderlripandiopareequivalent*

\begin{proof}
For the proof we will need that if $(a,b)\in \R_{+}$ and $\delta \in [0,1]$ then $(a+b)^{\delta}\leq a^{\delta}+b^{\delta}$.

\textbf{IOP $\implies$ LRIP} Suppose that $\D$ satisfies \eqref{eq:iop}. Let $\pi,\pi' \in \Sfrak$. Then by the triangle inequality:
\begin{equation}
\begin{split}
\|\pi-\pi'\|_{\mathcal{L}(\Hcal),p}&\leq \|\pi-\D[\Acal(\pi)]\|_{\mathcal{L}(\Hcal),p}+\|\pi'-\D[\Acal(\pi)]\|_{\mathcal{L}(\Hcal),p} \,.
\end{split}
\end{equation}
For the first term $\|\pi-\D[\Acal(\pi)]\|_{\mathcal{L}(\Hcal),p}$ we can apply the Hölder IOP with $\ebf=0$ which gives $\|\pi-\D[\Acal(\pi)]\|_{\mathcal{L}(\Hcal),p}\leq \eta$ since $\pi \in \Sfrak$ so $\operatorname{Bias}(\pi,\Sfrak)=0$. For the second term see that $\Acal(\pi)=\Acal(\pi')+(\Acal(\pi)-\Acal(\pi'))$ so we can apply the IOP with $\ebf=\Acal(\pi)-\Acal(\pi')$ which gives $\|\pi'-\D[\Acal(\pi)]\|_{\mathcal{L}(\Hcal),p}=\|\pi'-\D[\Acal(\pi')+\ebf]\|_{\mathcal{L}(\Hcal),p}\leq 0+C\|\Acal(\pi)-\Acal(\pi')\|^{\delta}_2+\eta$ and finally we have \eqref{eq:lrip} with constant $C$ and error $2\eta$.

\textbf{LRIP $\implies$ IOP} Suppose that $\Acal$ satisfies \eqref{eq:lrip}. Consider the decoder
\begin{equation}
\D[\sbf]\in \underset{\pi \in \Sfrak}{\arg\min} \|\Acal(\pi)-\sbf\|_2\,,
\end{equation}
which means that $\|\Acal(\D[\sbf])-\sbf\|_2\leq \|\Acal(\tau)-\sbf\|_2$ for any $\tau \in \Sfrak$. We define $$\operatorname{Bias}(\pi,\Sfrak):=\inf_{\tau \in \Sfrak}\left(\|\pi-\tau\|_{\mathcal{L}(\Hcal),p}+2C\|\Acal(\tau)-\Acal(\pi)\|^{\delta}_2\right)\,.$$ We show that this decoder satisfies \eqref{eq:iop} with this $\operatorname{Bias}$ term. Let $\pi \in \P(\Xcal)$ and $\ebf \in \mathbb{C}^{m}$. Consider any $\tau \in \Sfrak$. Then
\begin{equation}
\begin{split}
\|\pi-\D[\Acal(\pi)+\ebf]\|_{\mathcal{L}(\Hcal),p} & \leq \|\pi-\tau\|_{\mathcal{L}(\Hcal),p}+\|\tau-\D[\Acal(\pi)+\ebf]\|_{\mathcal{L}(\Hcal),p} \\
&\stackrel{*}{\leq} \|\pi-\tau\|_{\mathcal{L}(\Hcal),p}+ C\|\Acal(\tau)-\Acal(\D[\Acal(\pi)+\ebf])\|^{\delta}_2 +\eta\\
&\stackrel{**}{\leq} \|\pi-\tau\|_{\mathcal{L}(\Hcal),p}+C\|\Acal(\tau)-(\Acal(\pi)+\ebf)\|^{\delta}_2 \\
&+C\|(\Acal(\pi)+\ebf)-\Acal(\D[\Acal(\pi)+\ebf])\|^{\delta}_2 +\eta\,,
\end{split}
\end{equation}
where in (*) we use the LRIP since $\tau$ and $\D[\Acal(\pi)+\ebf]$ are in $\Sfrak$. In (**) we use the triangle inequality and the property $(a+b)^{\delta}\leq a^{\delta}+b^{\delta}$. By the properties of the decoder we have $\|(\Acal(\pi)+\ebf)-\Acal(\D[\Acal(\pi)+\ebf])\|_2 \leq \|(\Acal(\pi)+\ebf)-\Acal(\tau)\|_2$. Consequently:
\begin{equation}
\begin{split}
\|\pi-\D[\Acal(\pi)+\ebf]\|_{\mathcal{L}(\Hcal),p} & \leq \|\pi-\tau\|_{\mathcal{L}(\Hcal),p}+2C\|\Acal(\tau)-(\Acal(\pi)+\ebf)\|^{\delta}_2+\eta  \\
&\leq \|\pi-\tau\|_{\mathcal{L}(\Hcal),p}+2C\|\Acal(\tau)-\Acal(\pi)\|^{\delta}_2+2C\|\ebf\|^{\delta}_2+\eta\,. \\
 \|\pi-\D[\Acal(\pi)+\ebf]\|_{\mathcal{L}(\Hcal),p} &\stackrel{*}{\leq} \operatorname{Bias}(\pi,\Sfrak) + 2C\|\ebf\|^{\delta}_2+\eta\,,
\end{split}
\end{equation}
where in (*) we used the definition of $\operatorname{Bias}(\pi,\Sfrak)$ since the previous was true for any $\tau \in \Sfrak$.
\end{proof}

\subsection{Proof of Proposition \ref{prop:wass_learn_is_ness} \label{proof:prop:wass_learn_is_ness}}

\wasslearnnecessary*
\begin{proof}
Under the hypothesis of the proposition we have $\Phi \in \Lip(\R^{d}, \R^{m})$ and
\begin{equation}
\label{eq:hypothesis_lrip}
\forall \pi,\pi' \in \Sfrak, \|\pi-\pi'\|_{\mathcal{L}(\Hcal),p} \leq C \|\Acal(\pi)-\Acal(\pi')\|_2\,,
\end{equation}
for some $C>0$. As shown in \citet[Appendix D, Proof of Lemma 3.2 and Lemma 3.4]{gribonval2020compressive} the duality property of the Wasserstein distance implies $\|\Acal(\pi)-\Acal(\pi')\|_2 \leq L\W_1(\pi,\pi')$. The argument is the following: for $\pi,\pi' \in \Sfrak$,
\begin{equation}
\begin{split}
\|\Acal(\pi)-\Acal(\pi')\|_2&= \sup_{\ubf \in \R^{m} : \|\ubf\|_2 \leq 1} \ |\langle \ubf, \Acal(\pi)-\Acal(\pi')\rangle| \\
&= \sup_{\ubf \in \R^{m} : \|\ubf\|_2 \leq 1} |\int \langle \ubf, \Phi(\xbf) \rangle \dr \pi(\xbf) - \int \langle \ubf, \Phi(\ybf) \rangle \dr \pi'(\ybf)| \\
&= \sup_{\ubf \in \R^{m}: \|\ubf\|_2 \leq 1}|\int \Phi_{\ubf}(\xbf) \dr \pi(\xbf) - \int  \Phi_{\ubf}(\ybf) \dr \pi'(\ybf)|\,,
\end{split}
\end{equation}
where we define $\Phi_{\ubf}(\cdot) = \langle \ubf, \Phi(\cdot)\rangle$. Moreover, for any $\ubf \in \R^{m}$ with $\|\ubf\|_2 \leq 1$ we have $\Phi_{\ubf} \in \Lip_{L}(\R^{d}, \R)$ since $\Phi \in \Lip(\R^{d}, \R^{m})$. Consequently, using the duality property of the Wasserstein distance:
\begin{equation}
\|\Acal(\pi)-\Acal(\pi')\|_2 \leq \sup_{f \in \Lip_L(\R^d, \R)}|\int f(\xbf) \dr \pi(\xbf) - \int  f(\ybf) \dr \pi'(\ybf)| =  L \W_1(\pi,\pi')\,.
\end{equation}
Combining with \eqref{eq:hypothesis_lrip} we have  
\begin{equation}
\forall \pi,\pi' \in \Sfrak, \|\pi-\pi'\|_{\mathcal{L}(\Hcal),p} \leq C L \W_1(\pi,\pi')\,.
\end{equation}
Finally to conclude we use $\W_1(\pi,\pi') \leq \W_p(\pi,\pi')$ since $p \in [1,+\infty)$ \citep[Section 5.1]{San15a}.

\end{proof}

\end{document}